\documentclass[review]{article}

\usepackage{amsmath}
\usepackage{amsfonts}
\usepackage{amssymb}
\usepackage{mathtools}
\usepackage{amsthm}
\usepackage{bm}
\usepackage{indentfirst}
\usepackage{hyperref}
\usepackage{graphicx}
\usepackage{subfigure}
\usepackage{array}
\usepackage{authblk}
\usepackage{fullpage}
\usepackage{color}
\usepackage{algorithm}
\usepackage{algpseudocode}
\usepackage{cleveref}
\usepackage{url}

\ifpdf
  \DeclareGraphicsExtensions{.eps,.pdf,.png,.jpg}
\else
  \DeclareGraphicsExtensions{.eps}
\fi

\newcommand{\TheTitle}{Physics-Information-Aided Kriging: Constructing Covariance Functions using Stochastic Simulation Models}

\newcommand{\tensor}[1]{\mathbf{#1}}
\newcommand{\mexp}[1]{\mathbb{E}\left\{#1\right\}}
\newcommand{\trans}{\mathsf{T}}
\newcommand{\cov}{\text{Cov}}

\newcommand{\cor}{\text{Cor}}

\newcommand{\half}{\frac{1}{2}}
\newcommand{\ey}{y(\bm x)}
\newcommand{\hy}{\hat{y}(\bm x)}
\newcommand{\Hy}{\{Y^m(\bm x)\}_{m=1}^M}
\newcommand{\Hyl}{\{Y_L^m(\bm x)\}_{m=1}^{M_L}}
\newcommand{\Hyh}{\{Y_H^m(\bm x)\}_{m=1}^{M_H}}
\newcommand{\hs}{\hat{s}(\bm x)}
\newcommand{\smc}{\sigma_{_{MC}}(\bm x)}

\newtheorem{thm}{Theorem}[section]

\newtheorem{corollary}[thm]{Corollary}

\title{\TheTitle}

\author{Xiu Yang\footnote{Email: xiy518@lehigh.edu}}
\affil{Department of Industrial and Systems Engineering, Lehigh University, Bethlehem, PA 18015}
\author{Guzel D. Tartakovsky}
\affil{Intera}
\author{Alexandre M. Tartakovsky}
\affil{Civil annd Environmental Engineering, University of Illinois,
Urbana-Champion, IL 61801 and Pacific Northwest National Laboratory, Richland, WA
99354}
  
\begin{document}
\maketitle

\begin{abstract}
In this work, we propose a new Gaussian process regression (GPR) method:
physics information aided Kriging (PhIK). In the standard data-driven Kriging, the unknown
function of interest is usually treated as a Gaussian process with assumed stationary
covariance with hyperparameters estimated from data. In PhIK, we compute the mean
and covariance function from realizations of available stochastic models, e.g.,
from realizations of governing stochastic partial differential equations
solutions. Such constructed Gaussian process generally is
non-stationary, and does not assume a specific form of the covariance function.
Our approach avoids the optimization step in data-driven GPR methods to
identify the hyperparameters. More importantly, we prove that the physical
constraints in the form of a deterministic linear operator are guaranteed in the
resulting prediction. We also provide an error estimate in preserving 
the physical constraints when errors are included in the stochastic model 
realizations. To reduce the computational cost of obtaining stochastic model 
realizations, we propose a multilevel Monte Carlo estimate of the mean and 
covariance functions. Further, we present an active learning algorithm that 
guides the selection of additional observation locations. The efficiency and 
accuracy of PhIK are demonstrated for reconstructing a partially known modified 
Branin function, studying a three-dimensional heat transfer problem and learning a 
conservative tracer distribution from sparse
concentration measurements.

\end{abstract}


\section{Introduction}

Gaussian process regression (GPR), also known as \emph{Kriging} in geostatistics, is a 
widely used method in applied mathematics, statistics and machine learning for
constructing surrogate models, interpolation, classification, supervised
learning, and active learning \cite{forrester2008engineering,sacks1989design,stein2012interpolation}. 
GPR constructs a statistical model of a partially observed function (of time and/or 
space) assuming this function is a realization
of a Gaussian process (GP). GP is uniquely described by its mean and covariance
function. In the standard (here referred to as \emph{data-driven}) GP, 
prescribed forms of mean and covariance
functions are assumed, and the hyperparameters (e.g., variance and correlation
length) are computed from data via negative log-marginal likelihood function 
minimization. There are several variants of GPR, including simple, ordinary, and
universal Kriging \cite{kitanidis1997introduction}. GPR is also closely related
to kernel machines in machine learning, but it includes more information as it 
provides the uncertainty estimate \cite{williams2006gaussian}. 

In the ordinary Kriging, a constant mean and a prescribed form of the stationary
covariance function (also known as \emph{kernel}) are used. 
This stationarity assumption reduces the number of hyperparameters and model
complexity. However, in many practical problems, state variable are not
stationary. Hence, using a stationary GP may not yield accurate approximation.
Furthermore, even if the state variable is stationary, there are many choices
of the covariance functions with different smoothness
properties~\cite{durrande2012additive,dietrich1997fast}.

In addition to (measurement) data, physical models that reflect
(partial) understanding of the underlying system are available in many areas.
Especially, when obtaining measurements or conducting experiments
is very costly, which is quite common in practical problems,
these physical models and simulation tools are critical in research and
applications.
For example, in hydrology study, partial physical knowledge is available in the
form of partial differential differential equations (e.g., the Darcy law 
governing the hydraulic head) with unknown boundary, initial condition, and/or 
space-dependent coefficient (e.g., the hydraulic conductivity). Subsequently, 
the standard treatment 
is to model the unknown parameters as random variables/fields to formulate 
stochastic partial differential equations (SPDEs). In general, the covariance of
the state variable of SPDEs defined on a bounded domain with 
non-periodic boundary conditions is non-stationary~\cite{Tart2003PoF,Jarman2013SIAMUQ}. This non-stationarity can help
reflecting inhomogeneity of the state variables in the underlying physical 
system. 
Some success has been shown to incorporate 
physical knowledge in kernels, 
e.g.,~\cite{schober2014probabilistic,raissi2018numerical,owhadi2015bayesian} computed 
kernels for linear and weakly nonlinear (allowing accurate linearization)
ordinary and partial differential equations by substituting a GPR approximation
of the state variables in governing equations and obtaining 
a likelihood to be solved by optimization methods. 
However, for more complex systems, such as those involving strongly
nonlinear terms and random variables/fields,
it can be difficult to implement GPR in this manner in that the GP representing
the system is no longer Gaussian after a nonlinear mapping. 

In this work, we propose to incorporate physical knowledge in Kriging by 
computing mean and covariance function from a physics model. Therefore, we call
this method the \emph{physics-information-aided Kriging}, or \emph{PhIK}. 
Specifically, we compute the mean and covariance function used in Kriging 
from simulation ensembles.
%
Because computational tools, including commercial and open source packages,
have achieved a significant degree of maturity for many science applications
(e.g., climate modeling, hydrology, aerospace engineering, electrical 
engineering, etc.), we can run parallel simulations in the ``Monte Carlo 
(MC) mode'' efficiently. 
This could be achieved by treating unknown parameters as
random parameters or random fields.
This is a common step in uncertainty quantification (UQ) and sensitivity analysis~\cite{murphy2004quantification,
witteveen2011uncertainty, hou2012sensitivity, yang2013uncertainty,
zhang2014enabling, dai2017geostatistics}, and existing data from these research
can be reused in the proposed method.
We will show in this work that using empirical mean and covariance from these MC
(or other sampling techniques) simulations to construct a GP (i.e., to estimate
the mean and covariance function) helps integrating physical knowledge in the 
resulting model. In addition to making GPR prediction more accurate in terms of
preserving some physical constraints, this approach removes the need for 
assuming a specific form of the kernel and solving an optimization problem
for its hyperparameters. A similar idea is adopted in the ensemble Kalman filter
(EnKF) \cite{evensen2003ensemble} for data assimilation in time-dependent 
problems, where the ensemble mean and variance are used.

The cost of estimating mean and
covariance depends on the size and complexity of the physical model. We propose
to reduce this cost by using multilevel Monte Carlo (MLMC)
\cite{giles2008multilevel}. Traditionally, MLMC has been used to approximate the
mean and one-point moments  by combining a relatively few 
high-resolution simulations with a (larger) number of coarse resolution
simulations to compute moments with the desired accuracy. We extend MLMC for
approximating covariance function, a two-point second moment. Then, we provide 
error estimates for PhIK and MLMC-based PhIK describing how well physics 
constraints are preserved. 
Finally, we apply PhIK for active learning  (i.e., choosing additional 
measurement locations) using the posterior variance in the PhIK
method.

This work is organized as follows: Section~\ref{sec:method} introduces
PhIK, the MLMC method for estimating statistics, the error estimates for PhIK,
and active learning. Section~\ref{sec:numeric} provides two numerical examples 
to demonstrate the efficiency of the proposed method. Conclusions are presented 
in Section~\ref{sec:concl}.


\section{Methodology}
\label{sec:method}

This section begins by reviewing the general GPR 
framework \cite{forrester2008engineering} and the Kriging method based on the
assumption of stationary GP \cite{abrahamsen1997review}. Next, we introduce the
PhIK and MLMC-based PhIK. 
Finally, we present an active learning algorithm based on PhIK.

\subsection{GPR framework}
\label{subsec:gpr}

We denote the observation locations as 
$\bm X = \{\bm x^{(i)}\}_{i=1}^N$ ($\bm x^{(i)}$
are $d$-dimensional vectors in $D\subseteq\mathbb{R}^d$) and the observed state
values at these locations as
$\bm y=(y^{(1)}, y^{(2)},\cdots, y^{(N)})^\trans$ ($y^{(i)}\in\mathbb{R}$). For
simplicity, we assume that $y^{(i)}$ are scalars. We aim to predict $y$ at any
new location $\bm x^* \in D$. The GPR method 
treats the observed data as a realization of random field
$Y:D\times\Omega\rightarrow\mathbb{R}$:
\begin{equation}
  \label{eq:gp0}
Y(\bm x;\omega) \sim \mathcal{GP}\left(\mu(\bm x), k(\bm x, \bm x')\right),
\end{equation}
where and $k:D\times D\rightarrow\mathbb{R}$ are the mean and covariance
(kernel) functions:
  \begin{align}
    \mu(\bm x) & = \mexp{Y(\bm x)}\\ 
    k(\bm x,\bm x') & = \cov\left\{Y(\bm x), Y(\bm x')\right\}
                      = \mexp{(Y(\bm x)-\mu(\bm x))(Y(\bm x')-\mu(\bm x'))},
  \end{align}
and $Y(\bm x)$ is the concise notation of $Y(\bm x; \omega)$.
The variance of $Y(\bm x)$ is $k(\bm x, \bm x)$, and its standard deviation is
$\sigma(\bm x)=\sqrt{k(\bm x,\bm x)}$.  

In practice, one can parametrize forms of $\mu(\bm x)$ and 
$k(\bm x, \bm x)$. 
For example, in the widely used ordinary Kriging method, a stationary GP is assumed. In this
case, $\mu$ is set as a constant $\mu(\bm x)\equiv\mu$. 
More importantly, it is assumed that $k(\bm x, \bm x')=k(\bm\tau)$, where 
$\bm\tau=\bm x-\bm x'$, and $\sigma^2(\bm x)=k(\bm x,\bm x)=k(\bm 0)=\sigma^2$
is a constant. 
Popular forms of kernels include polynomial, exponential, Gaussian, and 
Mat\'{e}rn functions. For example, the Gaussian kernel can be written as
$k(\bm\tau)=\sigma^2\exp\left(-\frac{1}{2}\Vert \bm x-\bm x'\Vert^2_w\right)$,
where the weighted norm is defined as 
$\displaystyle\Vert \bm x-\bm x'\Vert^2_w=\sum_{i=1}^d
\left(\frac{x_i-x'_i}{l_i}\right)^2$ and $l_i$'s are correlation lengths.
Taking a classical frequentist stance, we have the 
following best linear unbiased predictor (BLUP)~\cite{sacks1989design, jones1998efficient}
\begin{equation}
\label{eq:krig_pred0}
  \hat y(\bm x^*) = \hat\mu + \bm c^\trans\tensor{C}^{-1}(\bm y-\hat\mu\bm 1), 
\end{equation}
and its mean squared error 
\begin{equation}
\label{eq:mse}
  \hat s^2(\bm x^*) = \hat\sigma^2-\bm c^\trans \tensor{C}^{-1}\bm c +
  \dfrac{(1-\bm 1^\trans\tensor {C}^{-1}\bm c)^2}{\bm 1^\trans \tensor C^{-1}\bm 1},
\end{equation}
where $\tensor C$ is the following covariance matrix:
\begin{equation}
\label{eq:cov_matrix0}
\tensor C = 
\begin{pmatrix}
k(\bm x^{(1)}, \bm x^{(1)}) & \cdots & k(\bm x^{(1)}, \bm x^{(N)}) \\
    \vdots & \ddots & \vdots  \\
k(\bm x^{(N)}, \bm x^{(1)}) & \cdots & k(\bm x^{(N)}, \bm x^{(N)})
\end{pmatrix},
\end{equation}
and 
$  \bm c=(k(\bm x^{(1)},\bm x^*), 
             \cdots,k(\bm x^{(N)},\bm x^*))^\trans, \bm 1=(1, 1, \dotsc,
             1)^\trans, \hat\mu = \frac{\bm 1^\trans\tensor C^{-1}\bm c}{\bm 1^\trans \tensor
  C^{-1} \bm 1}$. Hyperparameters $\hat\sigma$ and $\hat l_i$ can be obtained by maximum
  likelihood estimate. Of note, in most literature,
  Eqs.~\eqref{eq:krig_pred0} and ~\eqref{eq:mse} are based on correlation 
  function and matrix. Here, we use covariance function and matrix in the
  expressions (which is more general) to keep consistency with PhIK description,
  which does not assume $Y(\bm x;\omega)$ is stationary.

Alternatively, given complete information of the prior GPR 
  $Y(\bm x)$ (i.e., $\mu(\bm x)$ and $k(\bm x, \bm x)$ are given), a Bayesian
approach uses the posterior distribution
$Y(\bm x^*)| \bm X, \bm y \sim\mathcal{N}(\hat y(\bm x^*), \hat s^2(\bm x^*))$
for the prediction~\cite{currin1988bayesian, rasmussen2004gaussian}, 
where 
  \begin{align}
    &  \hat y(\bm x^*) = \mu(\bm x^*) + \bm c^\trans\tensor{C}^{-1}(\bm y-\bm\mu), \\
    & \hat s^2(\bm x^*) = \sigma^2(\bm x^*)-\bm c^\trans \tensor{C}^{-1}\bm c,
  \end{align}
  and $\bm \mu=(\mu(\bm x^{(1)}),\cdots,\mu(\bm x^{(N)}))^\trans$.

Moreover, to account for  
the observation noise, one can assume that the noise is independent and
identically distributed (i.i.d.) Gaussian random variables with zero mean and variance
$\delta^2$, and replace $\tensor C$ with $\tensor C+\delta^2\tensor I$, 
which, from the Bayesian perspective, results in a posterior
distribution of the normal random $Y(\bm x^*)$.
Moreover, in the noiseless case, if $\tensor C$ is not invertible, following the
common GPR approach, one can add a small regularization term $\alpha\tensor I$
($\alpha$ is a small positive real number) to $\tensor C$ such that it becomes
full rank. Adding the regularization term is equivalent to assuming there is a
measurement noise. 

A more general approach is to assume a non-stationary covariance function, which
potentially increases the number of hyperparameters
\cite{paciorek2004nonstationary, plagemann2008nonstationary, brahim2004gaussian}.
However, these methods still need to assume a specific form of the correlation
functions according to experience. The key computational challenge in the
data-driven GPR is the optimization step of maximizing the (marginal) likelihood.
In many practical cases, this is a non-convex optimization problem, and the 
condition number of $\tensor C$ can be quite large. A more 
fundamental challenge in the data-driven GPR is that it does not explicitly 
account for physical constraints and requires a large amount of data to
accurately model the physics. The PhIK introduced in the next section aims
to address both of these challenges.  


\subsection{PhIK}
\label{subsec:enkrig}
We design PhIK based on the Bayesian interpretation of the GPR
method. Namely, we specify a prior distribution of $Y(\bm x)$, which
is achieved by taking advantage of the existing domain knowledge in the form of 
realizations of a stochastic model of the underlying system. As such, there is no
need to assume a specific form of the correlation functions and solve an 
optimization problem for the hyperparameters. This idea is motivated by many 
scientific and engineering problems, where numerical or analytical 
physics-based models are available. These models typically include random 
parameters or random processes/fields to reflect the lack of understanding
(e.g., physical laws) or information (e.g., coefficients, parameters, etc.) of 
the real system. Then, MC type of simulations are conducted to 
generate an ensemble of state variables, from which the statistics of these 
state variables, e.g., mean and standard deviation, are estimated. 
We can compute the empirical mean and covariance of the ensemble, and use them
to construct a GP. Finally, the results of PhIK is this GP conditioned on
the observed state variables.

Specifically, assume that we have $M$ realizations, 
denoted as $\{Y^m(\bm x)\}_{m=1}^M (\bm x\in D)$, 
of a stochastic model $u:D\times\Omega\rightarrow \mathbb{R}$.
The empirical mean of these realizations is 
\begin{equation}
  \label{eq:mc_mean}
  \mu_{_{MC}}(\bm x)=\dfrac{1}{M}\sum_{m=1}^M Y^m(\bm x).
  \mu_{_{MC}}(\bm x)=\dfrac{1}{M}\sum_{m=1}^M Y^m(\bm x),
\end{equation}
and the empirical covariance is
\begin{equation}
  \label{eq:mc_kernel}
  k_{_{MC}}(\bm x, \bm x')=\dfrac{1}{M-1} \sum_{m=1}^M
  \left(Y^m(\bm x)-\mu_{_{MC}}(\bm x)\right) \left(Y^m(\bm x')-\mu_{_{MC}}(\bm x')\right).
\end{equation}
It is clear that $\mu_{_{MC}(\bm x)}$ and $k_{_{MC}}(\bm x, \bm x')$
are unbiased estimates of $\mexp{u(\bm x)}$ and $\cov\left\{u(\bm x),
u(\bm x')\right\}$, respectively.
Also, the covariance matrix of $\tensor C$ in Eq.~\eqref{eq:cov_matrix0}
has the following empirical counterpart:
\begin{equation}
  \label{eq:mc_cov}
  \tensor C_{_{MC}}=\dfrac{1}{M-1} \sum_{m=1}^M
  \left(\bm Y^m-\bm\mu_{_{MC}} \right)
  \left(\bm Y^m-\bm\mu_{_{MC}} \right)^\trans,
\end{equation}
where $\bm Y^m=(Y^m(\bm x^{(1)}), \cdots, Y^m(\bm x^{(N)}))^\trans$, 
$\bm\mu_{_{MC}}=(\mu_{_{MC}}(\bm x^{(1)}), \cdots, \mu_{_{MC}}(\bm x^{(N)}))^\trans$.
Similar to the standard GPR, PhIK results in a posterior distribution
at location $\bm x^*$ as $Y(\bm x^*)|\bm X, \bm y, \alpha \sim \mathcal{N}(\hat
y(\bm x^*), \hat s^2(\bm x^*))$, where
\begin{equation}
\label{eq:krig_pred2}
  \hat y(\bm x^*) = \mu_{_{MC}}(\bm x^*) + 
  \bm c_{_{MC}}^\trans(\tensor{C}_{_{MC}}+\alpha \tensor I)^{-1}(\bm y-\bm\mu_{_{MC}}), 
\end{equation}
\begin{equation}
\label{eq:mse_mc}
  \hat s^2(\bm x^*) = \hat\sigma_{_{MC}}^2(\bm x^*)-\bm c_{_{MC}}^\trans 
      (\tensor{C}_{_{MC}}+\alpha \tensor I)^{-1}\bm c_{_{MC}},
\end{equation}
where $\hat\sigma^2_{_{MC}}(\bm x^*)=k_{_{MC}}(\bm x^*,\bm x^*)$ is the
variance of data set $\{Y^m(\bm x^*)\}_{m=1}^M$, 
$\bm c_{_{MC}}=(k_{_{MC}}(\bm x^{(1)}, \bm x^*), \cdots,
k_{_{MC}}(\bm x^{(N)}, \bm x^*))$, and $\alpha \tensor I$ is the regularization 
term because $\tensor C_{_{MC}}$ is not full rank.
By construction, PhIK can be considered as a Kriging method using empirical
statistics. More importantly, PhIK has several advantages: 
\begin{itemize}
  \item It does not need to assume stationarity of the GP.
  \item It does not need to assume a specific form of the covariance relation.
        Thus, the form of the resulting GP is more flexible. 
  \item It does not need to solve the optimization problem to identify 
        hyperparameters, which can be a challenging problem often suffering from
        the ill-conditioned covariance matrix.
  \item It incorporates physical constraints via the mean and covariance function.
\end{itemize}

Next, we present a theorem that details how well PhIK prediction preserves linear
physical constraints. In this analysis, we treat the stochastic
model $u(\bm x;\omega)$ as a random field.
\begin{thm}
  \label{thm:err_bound}
  Consider random fields $u(\bm x;\omega):\mathbb{R}^d\times\Omega\rightarrow
  \mathbb{R}$ and
  $g(\bm x;\omega):\mathbb{R}^{d'}\times\Omega\rightarrow\mathbb{R}$, where $d$ and $d'$ are positive
  integers. Each realization of $u(\bm x;\omega)$ is in a normed space $(U,
  \Vert\cdot\Vert_U)$ and each realization of $g(\bm x;\omega)$ is in a normed space
  $(V,\Vert\cdot\Vert_V)$. Assume that $\mathcal{L}:
  (U,\Vert\cdot\Vert_U)\rightarrow (V, \Vert\cdot\Vert_V)$ is a linear operator
  such that $\Vert\mathcal{L}u(\bm x;\omega)-g(\bm x;\omega)\Vert\leq\epsilon$ 
  for any $\omega\in\Omega$.
  Let $\{Y^m(\bm x)\}_{m=1}^M$ be a finite number of realizations of $u(\bm x;\omega)$, i.e., 
  $Y^m(\bm x)=u(\bm x;\omega^m)$. Then, the prediction
  $\hat y(\bm x)$ from PhIK satisfies
  \begin{equation}
  \label{eq:err_bound}
  \begin{aligned}\Vert\mathcal{L}\hat y(\bm x)-\overline{g(\bm x)}\Vert_V
    & \leq \epsilon+\left[ 2\epsilon\sqrt{\dfrac{M}{M-1}}+\sigma\left(g(\bm x;\omega^m)\right) \right] \cdot \\
    &\qquad \left\Vert(\tensor C+\alpha\tensor I)^{-1}_{_{MC}}(\bm y-\bm\mu_{_{MC}})\right
  \Vert_{\infty}\sum_{i=1}^N  \sigma(Y^m(\bm x^{(i)})),\end{aligned}
\end{equation}
where $\sigma(Y^m(\bm x^{(i)}))$ is the standard deviation of the data set
$\{Y^m(\bm x^{(i)})\}_{m=1}^M$ for each fixed $\bm x^{(i)}$, 
$\overline{g(\bm x)}=\frac{1}{M}\sum_{m=1}^M g(\bm x;\omega^m)$, and
$\sigma(g(\bm x;\omega^m))=\left(\frac{1}{M-1}\sum_{m=1}^M \Vert g(\bm x;\omega^m)-\overline{g(\bm
x)}\Vert^2_V \right)^{\half}$.
\end{thm}
We present the proof of this theorem in Appendix~\ref{sec:app_proof1}. 
In the following discussion, we rewrite $\Vert\cdot\Vert_V$ as
$\Vert\cdot\Vert$ to simplify the notation.


This theorem holds for various norms, e.g., $L_2$ norm, $L_{\infty}$
norm, and $H^1$ norm. In practice, the realizations $Y^m(\bm x)$ are obtained by 
numerical simulations and are subject to numerical errors, model errors, etc. 
Thus, the theorem includes $\epsilon$ in the upper bound. It also indicates that
the standard deviation of ensemble member $Y^m$ at all observation locations 
$\bm x^{(i)}$ affects the upper bound of 
$\Vert\mathcal{L}(\hat y(\bm x))-\overline{g(\bm x)}\Vert$. If the variance 
of $Y^m(\bm x^{(i)})$ is small at every $\bm x^{(i)}$, e.g., when the
physical model is less uncertain, the resulting prediction $\hat y(\bm x)$ will
not violate the linear constraint much, i.e., 
$\Vert\mathcal{L}\hat y(\bm x)-\overline{g(\bm x)}\Vert$ is small. 
Moreover, if $g(\bm x;\omega)$ is a deterministic function $g(\bm x)$, then
$\sigma(g(\bm x;\omega^m))=0$ in the upper bound (see
Eq.~\eqref{eq:err_bound}).
Another important factor for the error bound is $\max_i |\tilde a_i|$, i.e., 
$\Vert(\tensor C+\alpha\tensor I)^{-1}_{_{MC}}(\bm y-\bm\mu_{_{MC}})\Vert_{\infty}$.
Because
$\Vert (\tensor C_{_{MC}}+\alpha \tensor I)^{-1}(\bm y-\bm \mu_{_{MC}}) \Vert_{\infty} \leq 
    \Vert (\tensor C_{_{MC}}+\alpha \tensor I)^{-1}(\bm y-\bm \mu_{_{MC}}) \Vert_2 \leq 
    \Vert (\tensor C_{_{MC}}+\alpha \tensor I)^{-1}\Vert_2\Vert\bm y-\bm \mu_{_{MC}}\Vert_2$,
the upper bound can be affected by the
difference between the physical model output and the observation, i.e.,
$\Vert\bm y-\bm \mu_{_{MC}}\Vert_2$, which is affected by the physical model's
accuracy, and the reciprocal of the smallest 
eigenvalue of $\tensor C_{_{MC}}+\alpha\tensor I$, which is bounded by
$\alpha^{-1}$.
In addition, the following corollary describes a special case.
\begin{corollary}
 In Theorem~\ref{thm:err_bound}, if $g$ is a deterministic function, i.e., 
 $g(\bm x;\omega)=g(\bm x)$, and $\mathcal{L}u(\bm x;\omega)=g(\bm x)$ for any
 $\omega\in\Omega$, then $\mathcal{L}\hat y(\bm x)=g(\bm x)$.
\end{corollary}
\begin{proof}
  Because $\mathcal{L}Y^m(\bm x)=g(\bm x)$ and 
  $\mathcal{L}\mu_{_{MC}}(\bm x)=\overline{g(\bm x)}=g(\bm x)$, we have
  \begin{equation*}
    \begin{aligned}
      \mathcal{L}k_{_{MC}}(\bm x, \bm x^{(i)}) & =\mathcal{L}\Big[\dfrac{1}{M-1} \sum_{m=1}^M
    \Big(Y^m(\bm x)-\mu_{_{MC}}(\bm x)\Big) \left(Y^m(\bm x^{(i)})-\mu_{_{MC}}(\bm x^{(i)})\right)\Big] \\
   & =\dfrac{1}{M-1} \sum_{m=1}^M\mathcal{L}\Big(Y^m(\bm x)-\mu_{_{MC}}(\bm
      x)\Big)\left(Y^m(\bm x^{(i)})-\mu_{_{MC}}(\bm x^{(i)})\right)=0. 
  \end{aligned}
  \end{equation*}
 Therefore, 
  \begin{equation*} 
    \mathcal{L}\hat y(\bm x)=\mathcal{L}\Big(\mu_{_{MC}}(\bm x) + \sum_{i=1}^N
    \tilde a_i k_{_{MC}}(\bm x, \bm x^{(i)})\Big)=\mathcal{L}\mu_{_{MC}}(\bm x)
    =g(\bm x).
  \end{equation*}
\end{proof}

For example, if $u(\bm x;\omega)$ satisfies the Dirichlet boundary condition 
$u(\bm x;\omega)=g(\bm x)$, $\bm x\in\partial D_D$ for any
$\omega\in\Omega$, then $\hat y(\bm x)=g(\bm x), \bm x\in\partial D_D$.
Similarly, if $u(\bm x;\omega)$ satisfies the Neumann boundary condition, 
$\partial u(\bm x;\omega)/\partial {\bm n}=0, \bm x\in\partial D_N$, then
$\partial \hat{y}(\bm x;\omega)/\partial {\bm n}=0, \bm x\in\partial D_N$. 
Another example is 
if $u$ satisfies 
$\nabla\cdot u(\bm x;\omega)=0$ for any $\omega\in\Omega$, $\hat y(\bm x)$
is also a divergence-free field. In general cases, i.e., $g$ is random
and $\epsilon\neq 0$, the upper bound in Theorem~\ref{thm:err_bound}
describes how well a physical constraint is preserved.

In this work, we choose to use the MC method to compute the mean and covariance 
because of its robustness. Here, the criteria for choosing $M$ are similar to the
standard MC method. In our numerical examples, the size of $M$ is
around $100$ for demonstration purpose. 
When the system's number of degrees of freedom (DOF) is large
(e.g., a high-resolution 3D model), the cost of estimating the mean and
covariance using MC method can be very high. Several methods can be used to reduce the cost
of estimating empirical statistics, including quasi-Monte
Carlo~\cite{niederreiter1992random},
probabilistic collocation~\cite{XiuH05}, Analysis Of Variance 
(ANOVA)~\cite{YangCLK12}, and compressive sensing~\cite{YangK13}, as well as 
mode reduction methods, e.g., the moment equation method~\cite{Tart2017WRR}. 
Depending on the applications, these methods could be significantly more 
efficient than MC
, and some of them have been implemented in studying uncertainty and sensitivity
of complex systems such as 
climate~\cite{qian2016uncertainty}, hydrology~\cite{hou2012sensitivity}, and urban 
flow~\cite{margheri2016hybrid}.
It is not difficult to show that conclusions similar to
Theorem~\ref{thm:err_bound} hold if 
$\mu(\bm x)$ and $k(\bm x,\bm x')$ are approximated using a \emph{linear 
combination} of $\{Y^m(\bm x)\}_{m=1}^M$, where these realizations are based on
a different sampling strategy. 


\subsection{Estimating statistics using MLMC}
\label{subsec:mlmc}

The MC method requires a sufficiently large ensemble of $Y$ to
accurately estimate the mean and covariance matrix, which, in some
applications, can be unpractical. 
To address this issue, instead of using aforementioned different sampling
strategies, we estimate empirical statistics using MLMC.
For simplicity, we demonstrate the idea via two-level MLMC simulations. We use 
$Y_L^m(\bm x)$ ($m=1,...,M_L$) and $Y_H^m(\bm x)$ ($m=1,...,M_H$) to denote
$M_L$ low-accuracy and $M_H$ high-accuracy realizations of
the stochastic model for the system. We assume that $Y_L^m$ and $Y_H^m$
are realizations of stochastic models $u_L:D_L\times\Omega\rightarrow\mathbb{R}$
and $u_H:D_H\times\Omega\rightarrow\mathbb{R}$, respectively.
We also denote $\overline u(\bm x)=u_H(\bm x)-u_L(\bm x)$.
For example, $D_L\subset\mathbb{R}^d$ and $D_H=D\subseteq\mathbb{R}^d$ can be 
coarse and fine grids in numerical simulations, respectively. Thus,
$u_L$ and $u_H$ are low- and high-resolution random processes. In this case, when
computing $\overline u$, we interpolate $u_L$ from $D_L$ to $D_H$. To simplify 
notations, we use $u_L$ to denote both the low-resolution random process on $D_L$
and the interpolated random process from $D_L$ to $D_H$ in the MLMC formula. 
The mean of $u_H(\bm x)$ is estimated as
\begin{equation}
  \label{eq:mlmc_mean}
  \mexp{u_H(\bm x)}\approx\mu_{_{MLMC}}(\bm x)=\dfrac{1}{M_L}\sum_{m=1}^{M_L} Y_L^m(\bm x) +
      \dfrac{1}{M_H}\sum_{m=1}^{M_H}\overline Y^m(\bm x),
\end{equation}
which is the standard MLMC estimate of the mean \cite{giles2008multilevel}. In 
the past, MLMC was used only to estimate single point statistics, e.g.,
\cite{barth2011multi, bierig2015convergence, bierig2016estimation}. 
Here, we propose an MLMC estimate of the covariance function of $u_H(\bm x)$
based on the following relationship:
\begin{equation}
  \begin{aligned}
    \cov\left\{u_H(\bm x), u_H(\bm x')\right\} & = \cov\left\{u_L(\bm x) + \overline u(\bm x),
  u_L(\bm x') + \overline u(\bm x')\right\} \\
  & = \cov\left\{u_L(\bm x), u_L(\bm x')\right\} + \cov\left\{u_L(\bm x), \overline
u(\bm x') \right\} \\
&  +  \cov\left\{\overline u(\bm x), u_L(\bm x')\right\} + \cov\left\{\overline
u(\bm x), \overline u(\bm x') \right\}.
  \end{aligned}
\end{equation}
Because $u_L$ and $\overline u$ are sampled independently in MLMC, we have
\[\cov\left\{u_L(\bm x), \overline u(\bm x')\right\}=\cov\left\{\overline u(\bm x),
u_L(\bm x')\right\}=0.\]
Thus,
\begin{equation}
  \label{eq:mlmc_cov0}
  \begin{aligned}
    \cov\left\{u_H(\bm x), u_H(\bm x')\right\} = 
    \cov\left\{u_L(\bm x), u_L(\bm x')\right\} + \cov\left\{\overline u(\bm x),
    \overline u(\bm x') \right\},
  \end{aligned}
\end{equation}
and the unbiased MLMC approximation of the covariance is
\begin{equation}
\label{eq:mlmc_cov}
  \begin{aligned}
    & \cov\left\{u_H(\bm x), u_H(\bm x')\right\} 
    \approx k_{_{MLMC}}(\bm x, \bm x') \\
    = &\dfrac{1}{M_L-1} \sum_{m=1}^{M_L} \bigg(Y_L^m(\bm
  x)-\dfrac{1}{M_L}\sum_{m=1}^{M_L}Y_L^m(\bm x) \bigg) 
    \bigg(Y_L^m(\bm x')-\dfrac{1}{M_L}\sum_{m=1}^{M_L}Y_L^m(\bm x')\bigg)  \\
    & + \dfrac{1}{M_H-1} \sum_{m=1}^{M_H} \bigg(\overline Y^m(\bm
    x)-\dfrac{1}{M_H}\sum_{m=1}^{M_H}\overline Y^m(\bm x)\bigg)
  \bigg(\overline Y^m(\bm x')-\dfrac{1}{M_H}\sum_{m=1}^{M_H}\overline Y^m(\bm x')\bigg).
 \end{aligned}
\end{equation}
Finally, the MLMC-based PhIK model yields similar results as MC-based PhIK by
replacing all ``MC'' terms with ``MLMC'' terms:
\begin{equation}
\label{eq:mlmc_pred}
  \hat y(\bm x^*) = \mu_{_{MLMC}}(\bm x^*) + 
  \bm c_{_{MLMC}}^\trans (\tensor{C}_{_{MLMC}}+\alpha\tensor I)^{-1}(\bm y-\bm\mu_{_{MLMC}}), 
\end{equation}
\begin{equation}
\label{eq:mse_mlmc}
  \hat s^2(\bm x^*) = \sigma^2_{_{MLMC}}(\bm x^*)-\bm c_{_{MLMC}}^\trans 
  (\tensor{C}_{_{MLMC}}+\alpha \tensor I)^{-1}\bm c_{_{MLMC}},
\end{equation}
where $(\bm\mu_{_{MLMC}})_i=\mu_{_{MLMC}}(\bm x^{(i)})$, 
 $(\bm c_{_{MLMC}})_i=k_{_{MLMC}}(\bm x^{(i)}, \bm x^*)$, and
 $(\tensor C_{_{MLMC}})_{ij}=k_{_{MLMC}}(\bm x^{(i)}, \bm x^{(j)})$.

The following corollary is a straightforward extension of 
Theorem~\ref{thm:err_bound} for PhIK with the
mean and covariance obtained from MLMC.
\begin{corollary}
  \label{cor:err_bound2}
  Assume that $u_{_H}(\bm x;\omega)$ and $u_{_L}(\bm x;\omega)$ are random fields
  defined on $\mathbb{R}^d\times\Omega$ such that each realization of them is in
  a normed space $(U,\Vert\cdot\Vert_U)$. Let $\{Y_H^m(\bm x)\}_{m=1}^{M_H}$ 
  and $\{Y_L^m(\bm x)\}_{m=1}^{M_L}$ be finite ensembles of realizations of
  $u_{_H}(\bm x;\omega)$ and $u_{_L}(\bm x;\omega)$, respectively.
  $\Vert\mathcal{L}u_{_H}(\bm x;\omega)-g(\bm x;\omega)\Vert<\epsilon_{_H}$ and
  $\Vert\mathcal{L}u_{_L}(\bm x;\omega)-g(\bm x;\omega)\Vert<\epsilon_{_L}$ for
  any $\omega\in\Omega$, where $\mathcal{L}$, $\Vert\cdot\Vert$ and
  $g(\bm x;\omega)$  are defined in Theorem~\ref{thm:err_bound}.
The MLMC-based PhIK prediction $\hat y(\bm x)$ satisfies 
\begin{equation}\Vert\mathcal{L}\hat y(\bm x)-\overline{g(\bm x)}\Vert\leq C_{_H}\epsilon_{_H}
+ C_{_L}\epsilon_{_L} + \sigma(g(\bm x;\omega^m))\sum_{i=1}^N\tilde
a_i\sigma(Y_L^m(\bm x^{(i)})),\end{equation}
 where $\overline{g(\bm x)}$ is defined in Theorem~\ref{thm:err_bound},
 \[\begin{aligned}
     & C_{_H}=1+2\sum_{i=1}^N\tilde a_i\sqrt{\dfrac{M_H}{M_H-1}}\sigma(\overline
   Y^m(\bm x^{(i)})),  \\
   & C_{_L}= 2+2\sum_{i=1}^N\tilde a_i\left(\sqrt{\dfrac{M_L}{M_L-1}}\sigma(Y_L^m(\bm
  x^{(i)}))+\sqrt{\dfrac{M_H}{M_H-1}}\sigma(\overline Y^m(\bm x^{(i)}))\right),
\end{aligned}
 \]
  and $\tilde a_i$ is the $i$-th entry of $(\tensor C_{_{MLMC}}+\alpha\tensor I)^{-1}(\bm
  y-\bm\mu_{_{MLMC}})$.
  Here, $\sigma(g(\bm x;\omega^m))$ is defined in Theorem~\ref{thm:err_bound},
  $\sigma(Y_L^m(\bm x^{(i)}))$ and $\sigma(\overline Y^m(\bm x^{(i)}))$ are
  standard deviation of data sets $\{Y_L^m(\bm x^{(i)})\}_{m=1}^{M_L}$ and
$\{\overline Y^m(\bm x^{(i)})\}_{m=1}^{M_H}$, respectively.
\end{corollary}
We present the proof in Appendix~\ref{sec:app_proof2}.


It is not complicated to extend the two-level MC to a general $L$-level
MLMC. We present the following theorem for the $L$-level ($L>2$) MLMC-based PhIK
error bounds. The proof of this theorem immediately follows from 
Theorem~\ref{thm:err_bound}.
\begin{thm}
  \label{thm:general_mlmc}
  Assume that $\{Y_l^m(\bm x)\}_{m=1}^{M_l}, l=1,\cdots,L$ 
  are finite ensembles of realizations of random fields $u_l(\bm x;\omega)$,
  where each realization of any $u_l$ belongs to a normed space
  $(U,\Vert\cdot\Vert_U)$. Denote $\overline Y_l=Y_{l}-Y_{l-1}$ for $l=2,\cdots, L$ and
  $\overline Y_1=Y_1$. The MLMC-based PhIK prediction $\hat y(\bm x)$ can 
  be given as 
  \begin{equation}
  \label{eq:gen_mlmc_pred}
  \hat y(\bm x) = \mu_{_{MLMC}}(\bm x) + 
    \sum_{i=1}^N\tilde a_i k_{_{MLMC}}(\bm x, \bm x^{(i)}), 
\end{equation}
where
  \[\mu_{_{MLMC}}(\bm x)=\sum_{l=0}^L\dfrac{1}{M_l}\sum_{m=1}^{M_L}\overline Y_l(\bm x); \]
  \begin{multline*}k_{_{MLMC}}(\bm x,\bm x')\\=\sum_{l=0}^L \dfrac{1}{M_l-1} \sum_{m=1}^{M_L}
  \bigg(\overline Y_l^m(\bm x)-\dfrac{1}{M_l}\sum_{m=1}^{M_L}\overline Y_l^m(\bm
  x)\bigg)
  \bigg(\overline Y_l^m(\bm x')-\dfrac{1}{M_l}\sum_{m=1}^{M_L}\overline
  Y_l^m(\bm x')\bigg);
\end{multline*}
  and $\tilde a_i=((\tensor C_{_{MLMC}}+\alpha\tensor I)^{-1}(\bm y-\bm\mu_{_{MLMC}}))_i$,
$\bm\mu_{_{MLMC}}=(\mu_{_{MLMC}}(\bm x^{(1)}),\cdots,\mu_{_{MLMC}}(\bm x^{(N)}))^\trans$, 
  $(\tensor C_{_{MLMC}})_{ij}=k_{_{MLMC}}(\bm x^{(i)}, \bm x^{(j)})$.
  Let $\mathcal{L}$, $g(\bm x;\omega)$, $\overline{g(\bm x)}$, and
  $\Vert\cdot\Vert$ be given as in Theorem~\ref{thm:err_bound}.

  1) If $g(\bm x;\omega)$ is a deterministic function, i.e., $g(\bm x;\omega)=g(\bm x)$,
  and $u_l$ satisfies $\mathcal{L}u_l(\bm x;\omega)=g(\bm x)$ for any 
  $\omega\in\Omega$ and for $l=1,\cdots, L$, then $\mathcal{L}\hat y(\bm x)=g(\bm x)$.

  2) If $Y_l$ satisfies
  $\Vert\mathcal{L}Y_l(\bm x;\omega)-g(\bm x;\omega)\Vert\leq\epsilon_l$ for
  $l=1,\cdots, L$, then
  \begin{equation}\Vert\mathcal{L}\hat y(\bm x)\Vert\leq \sum_{l=1}^L C_l\epsilon_{l} +
  \sigma(g(\bm x;\omega^m))\sum_{i=1}^N\tilde a_i\sigma(Y_L^m(\bm x^{(i)})),\end{equation}
  where
 \[ C_l=\begin{dcases}
   1+2\sum_{i=1}^N\tilde a_i\sqrt{\dfrac{M_l}{M_l-1}}\sigma(\overline
     Y_l^m(\bm x^{(i)})),  & l=L; \\
     2+2\sum_{i=1}^N\tilde
    a_i\bigg(\sqrt{\dfrac{M_l}{M_l-1}}\sigma(\overline Y_l^m(\bm
   x^{(i)}))+\sqrt{\dfrac{M_{l+1}}{M_{l+1}-1}}\sigma(\overline Y_{l+1}^m(\bm
   x^{(i)}))\bigg), & 1\leq l<L.
 \end{dcases}\]
\end{thm}

Of note, MLMC estimates for variance and higher-order single point (i.e., a fixed 
$\bm x\in\mathbb{R}^d$) statistical moments was proposed 
in \cite{barth2011multi, bierig2015convergence, bierig2016estimation}. 
The covariance estimate Eq.~\eqref{eq:mlmc_cov} proposed herein
also can be used to estimate variance by setting $\bm x'=\bm x$. The systematic
convergence analysis of the MLMC can be found 
in \cite{giles2008multilevel,barth2011multi, cliffe2011multilevel, bierig2015convergence, bierig2016estimation}.
Other multifidelity methods, e.g., \cite{giles2008improved, zhu2017multi},
also can be used as long as they compute the mean and covariance efficiently. 


\subsection{Active learning}
\label{sec:act}
In this context, \emph{active learning} 
(e.g., \cite{cohn1996active,jones1998efficient,tong2001support,collet2015optimism})
is a process of identifying locations for additional observations that minimize 
the prediction error and reduce MSE or uncertainty. In the GPR framework, a 
natural way is to add observations at the locations corresponding to local 
maxima in $s^2(\bm x)$, e.g., \cite{forrester2008engineering, raissi2017machine}. 
Then, we can make a new prediction $\hat y(\bm x)$ for $\bm x\in D$ and compute
a new $\hat s^2(\bm x)$ to select the next location for additional observation (see
Algorithm~\ref{algo:act}). Such treatment differs from other sensor placement 
methods based on deterministic approximation of unknown fields 
(e.g., \cite{zhang2008efficient, yang2010eof}). This selection 
criterion is based on the statistical interpretation of the interpolation. 
\begin{algorithm}[!h]
  \caption{Active learning based on GPR}
  \label{algo:act}
  \begin{algorithmic}[1]
  \State Specify the locations $\bm X$, corresponding observation $\bm y$, and 
         the maximum number of observation $N_{\max}$ affordable. The number of 
         available observations is denoted as $N$.
    \While {$N_{\max}>N$}
  \State Compute the MSE $\hat s^2(\bm x)$ of MLE prediction $\hat y(\bm x)$ for 
         $\bm x\in D$.
  \State Locate the location $\bm x_m$ for the maximum of $\hat s^2(\bm x)$ for
         $\bm x\in D$. 
  \State Obtain observation $y_m$ at $\bm x_m$ and set 
         $\bm X = \{\bm X, \bm x_m\}, \bm y = (\bm y^\trans, y_m)^\trans, N=N+1$.
  \EndWhile
  \State Construct the MLE prediction of $\hat y(\bm x)$ on $D$ using $\bm X$ 
         and $\bm y$.
  \end{algorithmic}
\end{algorithm}

Notably, Algorithm \ref{algo:act} is a greedy algorithm to identify additional
observation locations when some observations are affordable. It cannot guarantee
to identify the optimal new observation locations. More sophisticated algorithms
can be found in literature, e.g., \cite{jones1998efficient, krause2008near,xiao2018new},
and PhIK is complementary to these methods because it provides the GP.
Also, it is not necessary that the new observations are added one by one. Roughly 
speaking, if there are several maxima of $\hat s^2(\bm x)$ and they are not clustered (to avoid 
potential ill-conditioning of $\tensor C$ in some cases), the observations at
these locations can be added simultaneously. In this work, we add
new observations one by one in the numerical examples for demonstration purposes.
The efficiency of the active learning algorithm depends on the correlation
$\cor\left\{Y(\bm x), Y(\bm x')\right\}=\text{Cov}\{Y(\bm x), Y(\bm
x')\}/(\sigma(Y(\bm x))\sigma(Y(\bm x')))$. Intuitively, if the correlation is
large, then adding a new observation will provide information in a large
neighborhood of this location, reducing the MSE in a large region. An
extreme example is that when $\cor\left\{Y(\bm x), Y(\bm x')\right\}\equiv 1$
(e.g., the correlation length of the GP is infinite), only one observation is
needed to reconstruct the field. On the other hand, if the correlation is small 
(e.g., the correlation length of the GP is small), an observation can only influence a small
neighborhood, then more observations are required to reduce the uncertainty in
the prediction of the entire domain. An extreme example of this scenario is 
$\cor\{Y(\bm x), Y(\bm x')\}\equiv 0$. Unless we have observations everywhere,
the MSE in the prediction at the locations with no observations is
unchanged no matter how many observations we have because at these locations
$\bm c=\bm 0$ in Eq.~\eqref{eq:mse}.

There is an important difference between performance of Kriging and
PhIK in the context of active learning. PhIK prediction on entire $D$ can be
written as
$\hat{\bm y}=\bm\mu_{_{MC}} + \sum_{i=1}^N \tilde a_i \bm k_{_{MC}}^{(i)}$ 
(see Eq.~\eqref{eq:krig_form1}), where $\hat{\bm y}=\hat y(\bm x)$, and
$\bm\mu_{_{MC}}=\mu_{_{MC}}(\bm x)$
$\bm k_{_{MC}}^{(i)}=k(\bm x,\bm x^{(i)})$. Because $\bm\mu_{_{MC}}$ and 
$\bm k_{_{MC}}^{(i)}$ are linear combination of $\{\hat Y^m(\bm x)\}_{m=1}^M$,
PhIK prediction on the entire field $D$, i.e., $\hat{\bm y}$, is a linear 
combination of ensemble $\{\hat Y^m(\bm x)\}_{m=1}^M$. Therefore, once this
ensemble is given, the approximation space for the exact field is fixed, and 
adding more observation will not enlarge this space. It is expected that a
stochastic model that can better describe the underlying system will result in
more accurate PhIK prediction via providng a more appropriate approximation space
$\text{span}\{\hat Y^m(\bm x)\}_{m=1}^M$. In contrast, in Kriging method,
adding a new observation $y^{(N+1)}=y(\bm x^{(N+1)})$ enlarges the 
approximation space from $\text{span}\{k(\bm x, \bm x^{(i)})\}_{i=1}^N$
to $\text{span}\{k(\bm x, \bm x^{(i)})\}_{i=1}^{N+1}$ (see
Eq.~\eqref{eq:krig_form0}), which has the potential to increse the accuracy of
the prediction if the kernel $k(\cdot, \cdot)$ is chosen appropriately. We refer
interested readers to literatures on reproducing kernel 
Hilbert space for theoretical results on approximation, 
e.g.,~\cite{yuan2010reproducing, berlinet2011reproducing}. 

In both methods, the error $\Vert\hat{\bm y} - \bm y\Vert$ converges to zeros
asymptotically because both Kriging and PhIK predictions coincide with the exact
field wherever noiseless observation is available. However, in general, there is
no guarantee the error decreases monotonically for each newly added observation.

There is a large body of literature in statistics and machine learning on the
\emph{learning curve} that describes the average MSE over $D$ as a function of
$N$, the number of available observations, e.g., 
\cite{ylvisaker1975designs, micchelli1979design, williams2000upper, ritter2007average}. 
Both noisy and noiseless scenarios have been studied, 
and we refer interested readers to the aforementioned literatures.

\subsection{Some discussions}
\subsubsection{Potential drawbacks}
The PhIK methods requires using multiple simulation results of a stochastic
model, which can be costly when the model is complex.
There are different types of methods, e.g., multifidelity methods, that help to
reduce the cost of estimating statistics by
incorporating information from costly high-fidelity simulations and cheaper 
low-fidelity simulations. The MLMC used in this work is one of such methods.
We note that running simulations with different sets of random 
parameters is a typical step in studies like model calibration, uncertainty
quantification, sensitivity analysis as pointed out in 
Section~\ref{subsec:enkrig}. Therefore, we can either use the simulation 
results in these studies directly, or adjust
approaches in these areas that reduce the computational cost to PhIK. For 
example, a bifidelity method is used to dramatically reduce the cost of 
running simulations~\cite{yang2020bifidelity}. Again, the motivation of PhIK is
to exploit existing physical models to maximally extracted information very limited
measurement data. 

Another drawback is that PhIK results in a GP conditioned on observation data
and this GP depends on the stochastic model we use. If the model can not
appropriately reflect the known information or model unknowns using random
variables/fields, then PhIK may not work well. Further, when implementing the
active learning using PhIK, the locations of new observations is mainly decided 
by the pattern of the standard deviation of simulation reuslts. Therefore, it is
possible that some important locations are not selected for acquring new
observation if this information is not reflected in the model.

\subsubsection{Computational cost}
After obtaining simulations results, PhIK
compute the posterior mean and variance at a new location $\bm x^*$, i.e.,
\[\hat y(\bm x^*) = \mu(\bm x^*) + \bm c_{N\times 1}'\tensor C^{-1}_{N\times N}(\bm y-\bm \mu)_{N\times 1}, \]
\[\hat s^2(\bm x^*) = \sigma^2(\bm x^*) - \bm c_{N\times 1}'\tensor
C^{-1}_{N\times N}\bm c_{N\times 1}, \]
where we omit the notation ``MC'' or ``MLMC'' and the subscripts describe the size 
of each vector or matrix. The cost of obtaining $\tensor C$ is $O(M\times N^2)$.
We use the Cholesky decomposition for $\tensor
C$, the cost of which is $O(N^3)$. Vector $\bm c$ is obtained by a matrix-vector 
multiplication, the cost of which is $O(M\times N)$. Therefore, to compute the
$\hat y(\bm x^*)$ and $\hat s(\bm x^*)$ on all
grid points is $O(DOF\times M + M\times N^2+N^3 + DOF\times(M\times N + N^2))=O(DOF\times
M\times N + DOF\times N^2 + M\times N^2+ N^3)$. 
Since we consider complex system whose measurement is
very expensive, so $N$ is small (e.g., $N$ is no more than $30$
in our numerical examples) and
$M\ll DOF$ ($M$ is usually a few hundred).
Therefore, the complexity is approximately $O(DOF)$. 
This is the same as Kriging in this
scenario (small $N$). In our numerical examples to be shown in
the next section, the largest 
DOF is $32,768$ and it also has largest $M=500$.
Implementing PhIK method costs less than one second on a laptop with
2.9GHz Dual-Core Intel Core i5 using MATLAB without parallelization. 
For larger systems, one can use parallel computing package as well as
hardware like GPU to accelerate the computing. 

\subsubsection{Connection with other methods}

In general, a GPR model use the following regression
model~\cite{sacks1989design}:
\begin{equation}
  \label{eq:gpr_general}
  Y(\bm x) = \sum_{j=1}^K \beta_j f_j(\bm x) + Z(\bm x), 
\end{equation}
where $\beta_j$'s are constant scalar, $f_j$'s are deterministic functions, and
$Z$ is a GP with zeros mean. Apparently, $Y$ is a GP, and the observation data
are considered to be collected from a realization of $Y$. But usually only a few
data on this realization is collected. In ordinary Kriging,
$K=1, f_1(\bm x)\equiv 1$ and $\beta$ as well as hyperparamters of $Z$ are
estimated from data in the BLUP setting, i.e., $Y$ is identified from
partial data of one of its realization. In PhIK, $K=1, \beta=1$, $f_1$ is the
empirical mean function and $Z$'s covariance is the empirical covariance. 
In other words, we have several complete realizations of a random 
field (not necessarily Gaussian), and we use their empirical mean and 
covariance function to construct $Y$ directly. In a nutshell, Kriging constructs
GP using data (which is typically sparse) on a domain, while PhIK uses 
these data plus some realizations on the entire domain. In both scenarios, 
we can implement multifidelity methods to improve the computational 
efficiency.

The multifidelity may refer to low and high-resolution models based on the same equations 
(e.g., in the MLMC approach in the third numerical example) or  
models using different mathematical description~\cite{peherstorfer2018survey}, e.g.,
some low-fidelity model may neglect less important forces and average out fast processes. 
Multi-fidelity methods were shown to be efficient for computing one-point 
statistics (e.g., variance) for quantifying uncertainty \cite{ng2012multifidelity,geraci2017multifidelity,peherstorfer2018survey}.
In our work, we use MLMC to reduce the cost of estimating mean and two-point statistics (i.e., covariance). 

Another type of multifidelity framework, 
e.g.,~\cite{kennedy2000predicting,perdikaris2015multi,perdikaris2016multifidelity,parussini2017multi, owhadi2017multigrid},
uses linear or nonlinear combination of \emph{multiple} GPs to hierarchically construct
a high-fidelity model. For example, the multifidelity GPR 
in~\cite{kennedy2000predicting, perdikaris2015multi} assumes 
$Y(\bm x) = \rho(\bm x)Y_{\text{low}}(\bm x) + Z(\bm x)$ (we use two-fidelity case 
for illustration), where $Y_{\text{low}}$ is a GP regressing low-fidelity data
and $\rho(x)$ is a deterministic scalar function. It can be intepretted as 
replacing $\sum_{j=1}^K \beta_j f_j(\bm x)$ in Eq.~\eqref{eq:gpr_general} with
a GP $Y_{\text{low}}$. This procedure can be repeated recursively as in~\cite{
perdikaris2015multi,perdikaris2016multifidelity,parussini2017multi}. This idea
can be implemented in PhIK, too. For example, one can construct $Y_{\text{low}}$
using low-fidelity simulations first, then use high-fidelity simulations as well
as data to identify $\rho$ and $Z$.
Alternatively, one can construct $Y_{\text{low}}$ using PhIK with multifidelity
simulations as we show in this work, then identify $Z$ using a predecided form
of kernel function as in the aforementioned hierarchical method, an example of
which is shown in~\cite{yang2019physics}.

The main difference between PhIK and the aforementioned methods in
the context of multifidelity is that the latter considers multifidelity data
$\bm y=(y^{(1)}, \dotsc, y^{(N)})$, while the former considers single fidelity data. In the purely 
data-driven setting, the aforementioned hierarchical methods, each data set is 
a small subset of a \emph{single} realization of a random field at different
fidelity level. These data are
integrated to construct a random field on the entire domain. On the other hand, 
in PhIK, roughly speaking, we also have \emph{multiple} realizations of different
random fields on the entire domain, which are used to construct a random field
on the entire domain associated with the data. In this construction procedure, 
data is not used.
If we only have one realization from each stochastic model, then we have the
same setting as the purely data-driven setting, and PhIK does not apply
because we only have degenerate empirical statistics. Regarding the computational
cost, Kriging is known to be expensive when the data set is large. The
complexity in constructing a GP is $O(N^3)$, which can be further reduced, e.g.,
to $O(N)$, if special structure exists~\cite{quinonero2005unifying,
perdikaris2016multifidelity,parussini2017multi}. Of note, here $O(N)$ is the
cost to identify the kernel function, and then the cost to compute the covariance
matrix is $O(N^2)$. In the PhIK case, the cost of directly implementing the
algorithm to construct the covariance matrix is $O(MN^2)$. As we usually set 
$M$ to be a few hundred, when $N$ is large, the complexity is $O(N^2)$. The cost
of computing the empirical mean vector at observation locations is $MN$
multiplied by a constant, so the complexity is $O(N)$. Therefore, the complexity
of ``training'' PhIK is $O(N^2)$, which is the same as the purely data-driven methods.

We also note that merging model and data can be categorized as a multi-resolution 
and multifidelity modeling task, where a real system is represented as a set of 
models of different resolutions at different abstraction levels from the viewpoint of 
simulation objectives~\cite{davis1998experiments,hong2013specification,Rabelo2015MRM}.
In our cases, the data is considered as a high-fidelity ``model'' and the
stochastic model is treated as a low-fidelity model. 

Finally, the well-known work~\cite{lindgren2011explicit} built connection 
between a specific family of
fractional SPDE and the Matern random field in that the solution of the SPDE
is a Gaussian random field with a Matern kernel. Hence, one can solve the SPDE
to multiple times and use the empirical mean and covariance to construct the
corresponding GP. Our method can be considered as an extension of
this approach as it uses any stochastic model's empirical mean and
covariance to construct a GP. 


\section{Numerical examples}
\label{sec:numeric}

We present two numerical examples to demonstrate the performance of PhIK. Both
numerical examples are two-dimensional in physical space. In the first two
examples, we use the MC-based PhIK introduced in Section~\ref{subsec:enkrig}, 
and in the third one, we employ the MLMC-based PhIK presented in
Section~\ref{subsec:mlmc}. We compare PhIK with the ordinary Kriging (in the
following, we refer to it as Kriging). In the Kriging method,
we tested the Gaussian kernel and  Mat\'{e}rn kernel. We do not observe 
significant difference in the results and only report solutions obtained with 
the Gaussian kernel. All Kriging results are obtained by the MATLAB package
GPML~\cite{rasmussen2010gaussian}. In each example, we use $\ey$ to denote the
reference solution, $\hy$ to denote the posterior mean, which is typically
used as the prediction or reconstruction result, and $\hs$ to denote the posterior
standard deviation, which can be interpretted as the uncertainty in the prediction.
The simulations of stochastic model are denoted as
$\Hy$ for the MC-based PhIK. Similarly, in the MLMC-based PhIK, $\Hyl$ and $\Hyh$
denote simulations of low- and high-resolution simulations, respectively.

\subsection{Branin function}
We consider the following modified Branin function \cite{forrester2008engineering}:
\begin{equation}
\label{eq:branin_fun}
f(x, y) = a(\bar y-b\bar x^2+c\bar x-r)^2 + g(1-p)\cos(\bar x)+g + qx,
\end{equation}
where
\[\bar x=15x-5,~\bar y=15y,~(x,y)\in D=[0,1]\times [0,1],\]
and
\[a=1,~b=5.1/(4\pi^2),~c=5/\pi,~r=6,~g=10,~p=1/(8\pi),~q=5.\]
The contour of $f$ and eight randomly chosen observation locations are
presented in Figure~\ref{fig:branin_krig}(a). The function $f$ is evaluated on 
a $41\times 41$ uniform grid, and we denote the resulting discrete field (a $41\times 41$ matrix) as
$\ey$. We will compare $\hy$, $\hs$ and $\hy-\ey$ by different
methods. 

\subsubsection{Field reconstruction}

We first use  Kriging with the eight
observation data sets. Figure~\ref{fig:branin_krig}~(b) presents
the resulting $\hy$, and Figure~\ref{fig:branin_krig}~(c) depicts corresponding
$\hs$.
The difference $\hy-\ey$ is shown in Figure~\ref{fig:branin_krig}~(d), 
which quantifies the $\hy$ deviation from the ground truth.
Apparently, this reconstruction deviates considerably from $\ey$ in 
Figure~\ref{fig:branin_krig}~(a), especially in the region $[0, 0.5]\times [0.5, 1]$.
This is consistent with Figure~\ref{fig:branin_krig}(c) as $\hs$ is large in
this region. This is because there is no observation in this region. 
\begin{figure}[thbp]
\centering
\subfigure[$\ey$]{
\includegraphics[width=0.4\textwidth]{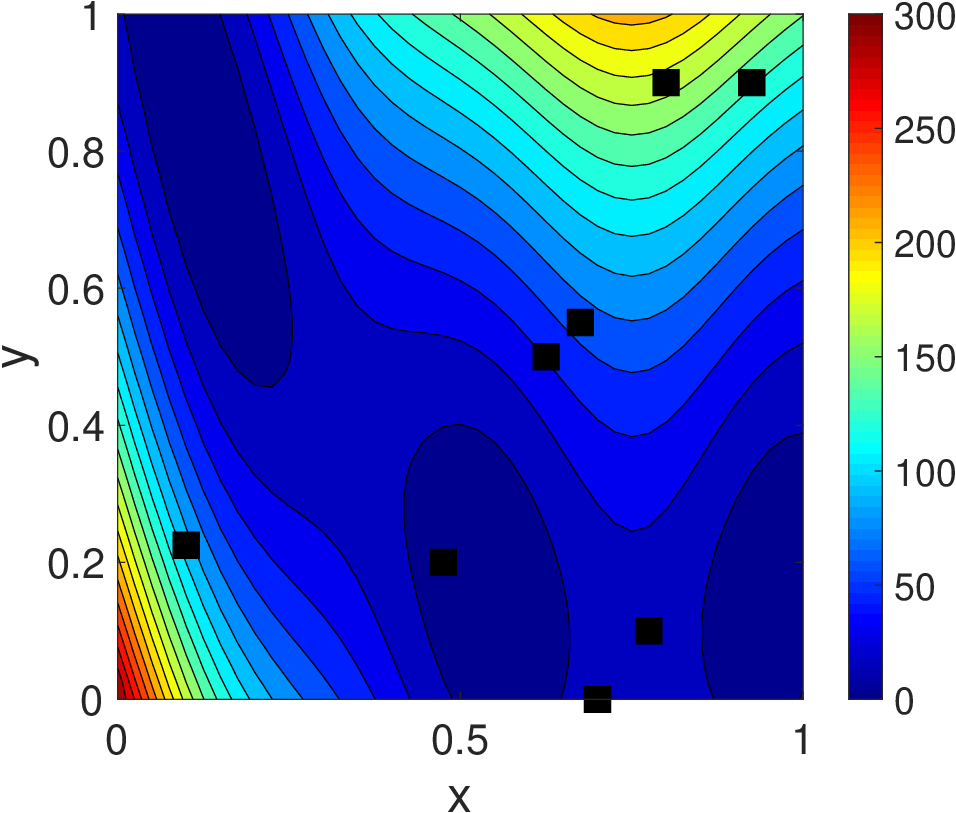}}\quad
\subfigure[$\hy$]{
\includegraphics[width=0.4\textwidth]{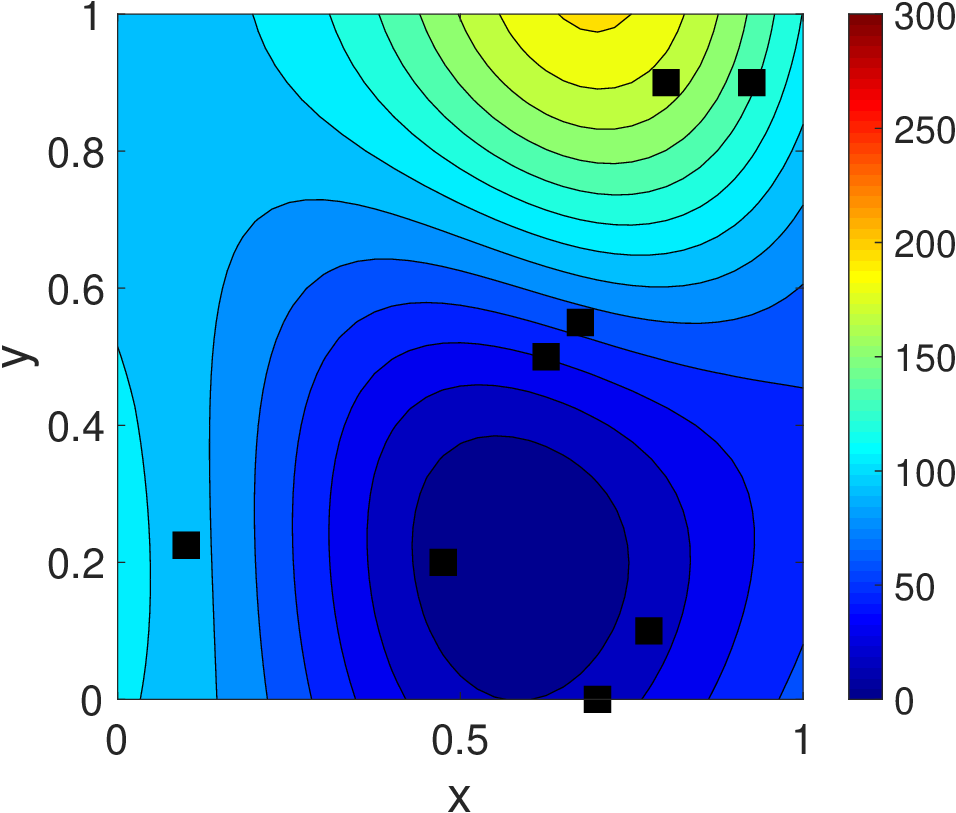}}\\
\subfigure[$\hs$]{
\includegraphics[width=0.4\textwidth]{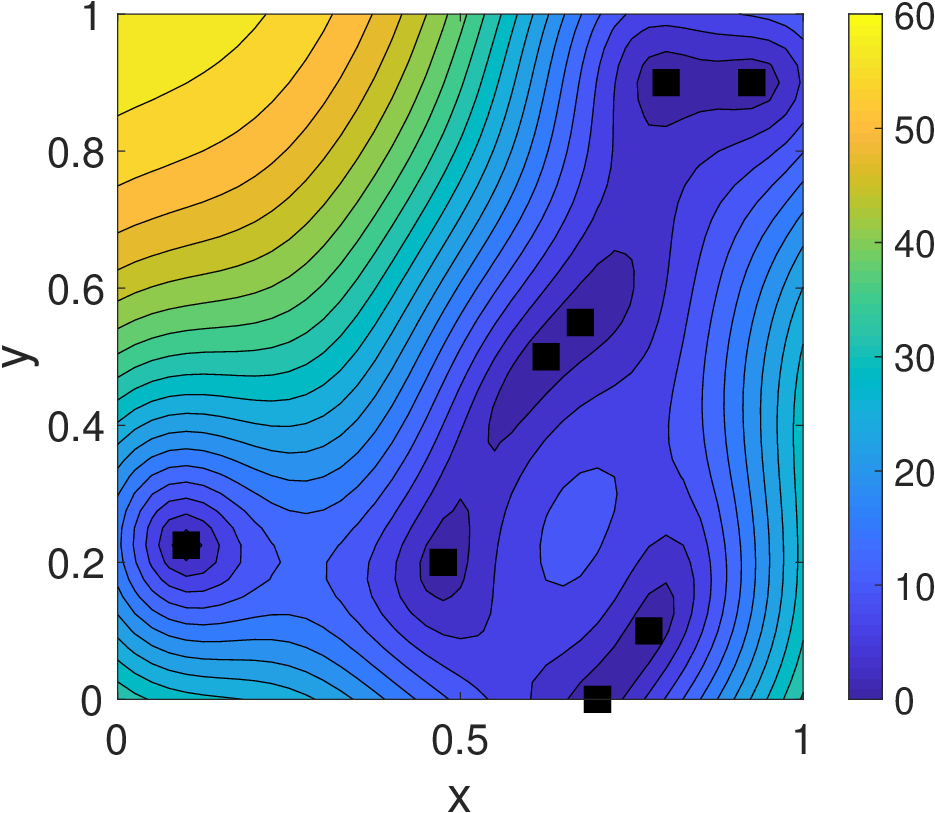}}\quad
\subfigure[$\hy-\ey$]{
\includegraphics[width=0.415\textwidth]{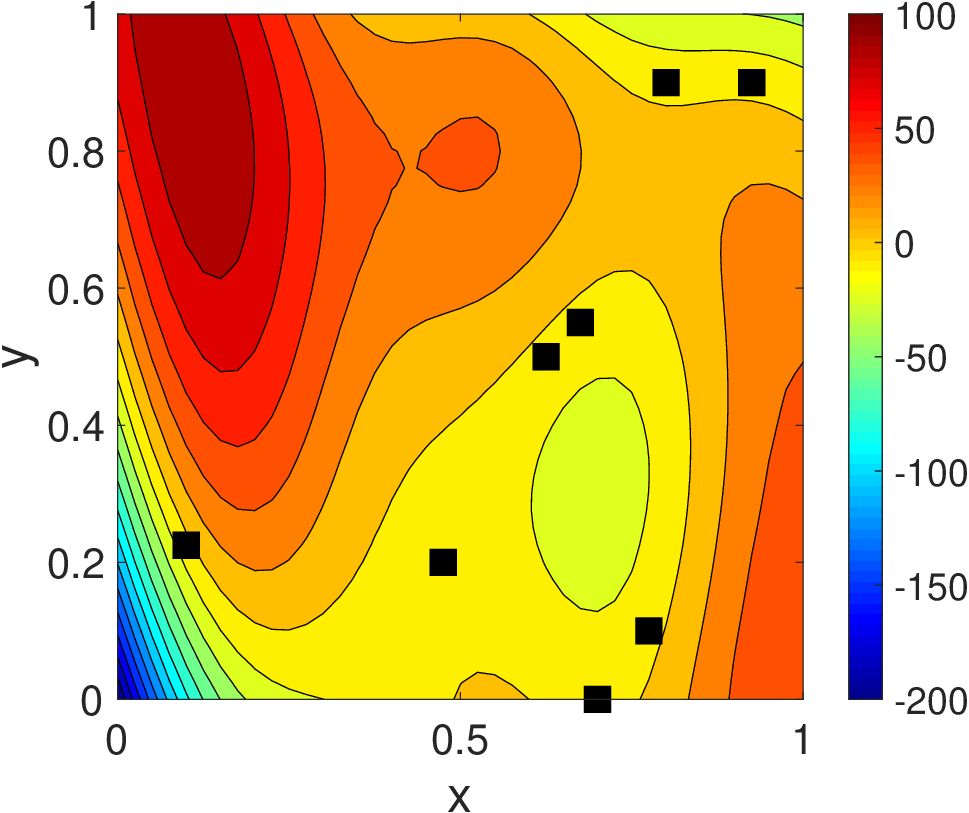}}
\caption{Branin function: reconstruction of the modified Branin function 
         by Kriging: (a) exact field $\ey$; (b) $\hy$; 
         (c) $\hs$; (d) difference $\hy-\ey$.}
\label{fig:branin_krig}
\end{figure}

Next, we assume that based on ``domain knowledge",  $f(x, y)$ is
partially known, e.g., its form is known, but the coefficients $b$ and $q$ are 
unknown. Then, we treat these coefficients as random fields $\hat b$ and $\hat q$,
which indicates that the field $f$ is described by a random function 
$\hat f: D\times\Omega\rightarrow\mathbb{R}$:
\begin{equation}
\label{eq:branin_rf}
  \hat f(x, y;\omega) = a(\bar y-\hat b(x, y;\omega)\bar x^2+c\bar x-r)^2 
      + g(1-p)\cos(\bar x)+ \hat g + \hat q(x,y;\omega)x,
\end{equation}
\begin{multline*}
  \hat b(x,y;\omega) 
    = b\Big\{0.9+\dfrac{0.2}{\pi}\sum_{i=1}^3\Big[ \dfrac{1}{4i-1}\sin((2i-0.5)\pi x)\xi_{2i-1}(\omega)  \\
     + \dfrac{1}{4i+1}\sin((2i+0.5)\pi y)\xi_{2i}(\omega)\Big]\Big\}, 
\end{multline*}
\begin{multline*}
  \hat q(x,y;\omega) 
   = q\Big\{1.0+\dfrac{0.6}{\pi}\sum_{i=1}^3\Big[ \dfrac{1}{4i-3}\cos((2i-1.5)\pi x)\xi_{2i+5}(\omega)  \\
     + \dfrac{1}{4i-1}\cos((2i-0.5)\pi y)\xi_{2i+6}(\omega)\Big]\Big\},
\end{multline*}
and $\{\xi_i(\omega)\}_{i=1}^{12}$ are i.i.d. Gaussian random variables with 
zero mean and unit variance. Further, we allow for the model error by setting 
$\hat g=20$, which is different from $g=10$ in Eq.~\eqref{eq:branin_rf}. We use
this partial knowledge to compute the mean and covariance function of $\hat{f}$
by generating $M=100$ samples of $\xi_i(\omega)$ and evaluating $\hat f$ on the
$41\times 41$ uniform grid for each sample of  $\xi_i(\omega)$. 
Figure~\ref{fig:branin_enkrig}~(a)-(c) present
$\hy$, $\hs$, and $\hy-\ey$. These results are much better than those obtained
by Kriging as both $\hs$ and $|\hy-\ey |$ are much smaller. 
More significantly, the $\hs$ in PhIK is much 
smaller than Kiriging in the $[0, 0.5]\times [0.5, 1]$ subdomain with no 
observations. This is because in PhIK, the covariance matrix is computed
based on standard deviation of the physics-based model. 
Figure~\ref{fig:branin_enkrig}~(d) shows the standard deviation of $\Hy$
denoted as $\smc$.
Note that $\sigma_{_{MC}}$ is a
measure of uncertainty in the physical model $\hat{f}$. 
This plot demonstrates that $\hs$ (which is a measure of uncertainty in PhIK) 
has a similar pattern as $\smc$,
but with smaller magnitude. It indicates that PhIK reduces uncertainty by
conditioning the prediction model $\hat f$ on observations.  
\begin{figure}[!h]
\centering
\subfigure[$\hy$]{
\includegraphics[width=0.4\textwidth]{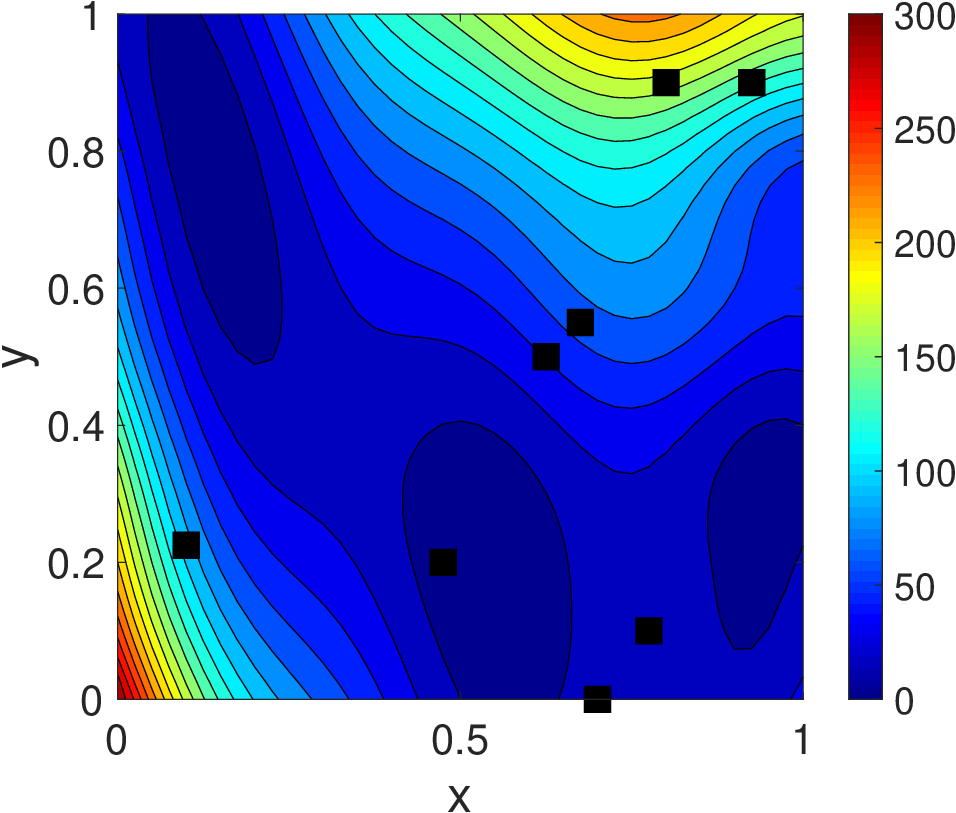}}\quad
\subfigure[$\hs$]{
\includegraphics[width=0.4\textwidth]{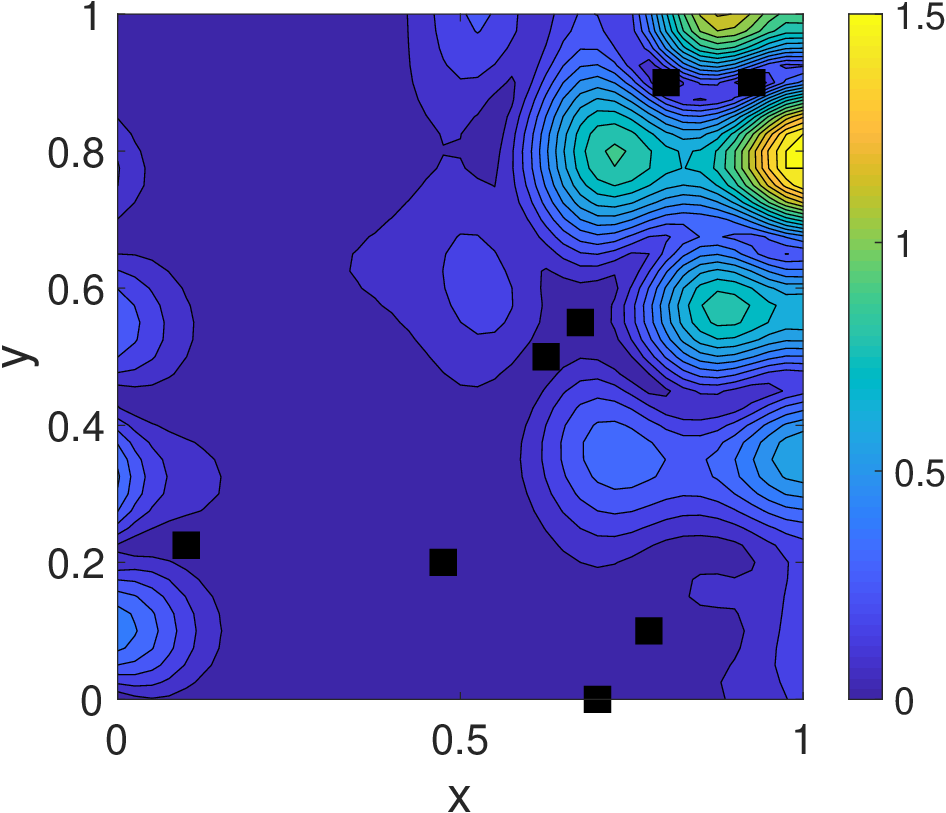}}\\
\subfigure[$\hy-\ey$]{
\includegraphics[width=0.4\textwidth]{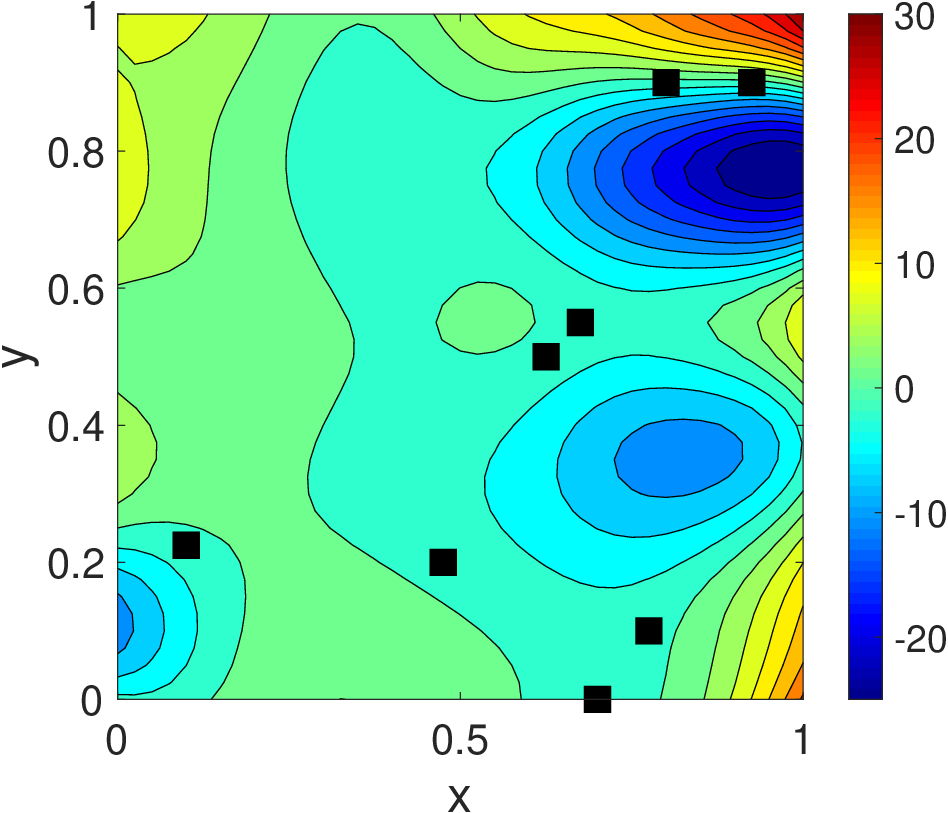}}\quad
\subfigure[empirical std.]{
\includegraphics[width=0.39\textwidth]{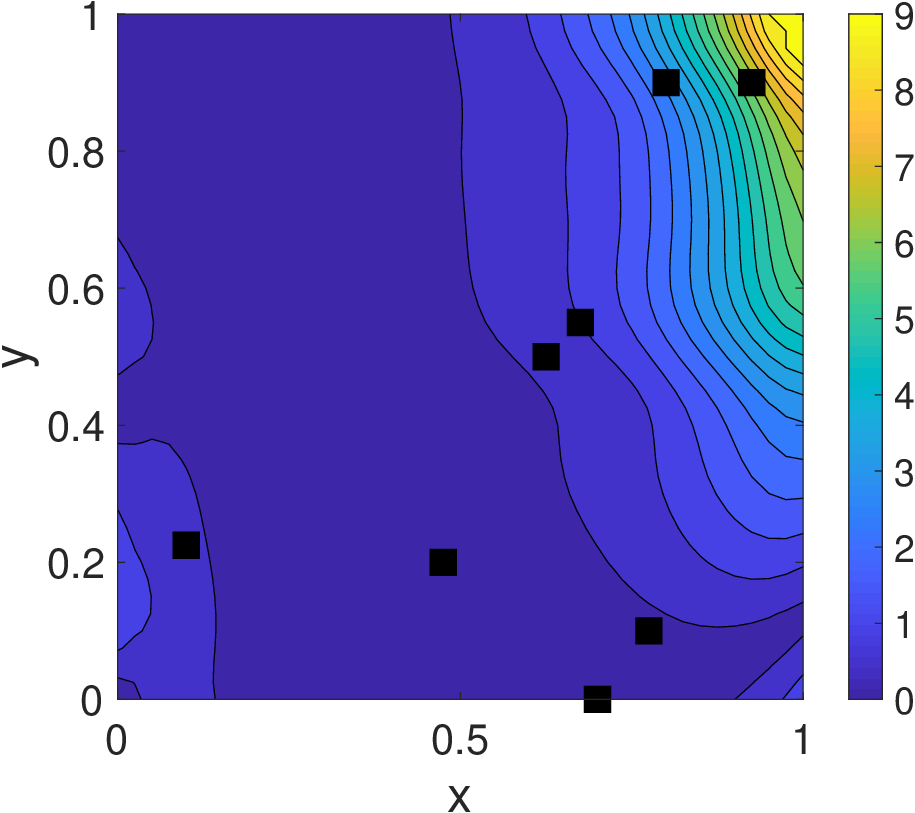}}
\caption{Branin function: reconstruction of the field by PhIK: (a) $\hy$; 
         (b) $\hs$; (c) $\hy-\ey$; 
         (d) standard deviation of $\Hy$.}
\label{fig:branin_enkrig}
\end{figure}

\subsubsection{Active learning}

After obtaining $\hs$, we use Algorithm~\ref{algo:act}
to perform active learning by adding one by one new observations of $f$ at
$(x,y)$ where $\hs$ has maximum. Figure~\ref{fig:branin_act} displays 
locations of additional observations (marked as stars) and results by Kriging
and PhIK when eight new observations are added. In this figure, the first row compares $\hy$ by these two methods, the
second row compares $\hs$, and the last row compares $\hy-\ey$.
As expected, the reconstruction accuracy increases as more observations are 
added. The PhIK still outperforms Kriging, but the difference is less significant
compared with results in Figures~\ref{fig:branin_krig}
and~\ref{fig:branin_enkrig}. Notably, the active learning with Kriging
places new observation points on all four boundaries.
This illustrates that the GPR is more accurate for interpolation than extrapolation.
Because most original observations are within the domain, the results are
extrapolated toward the boundary. In contrast, the active learning 
algorithm with PhIK places most new observation points near the right boundary,
where the $\smc$ is large. This further confirms the influence of the physical model on the PhIK algorithm.

\begin{figure}[!h]
\centering
\subfigure[Krig $\hy$]{
\includegraphics[width=0.4\textwidth]{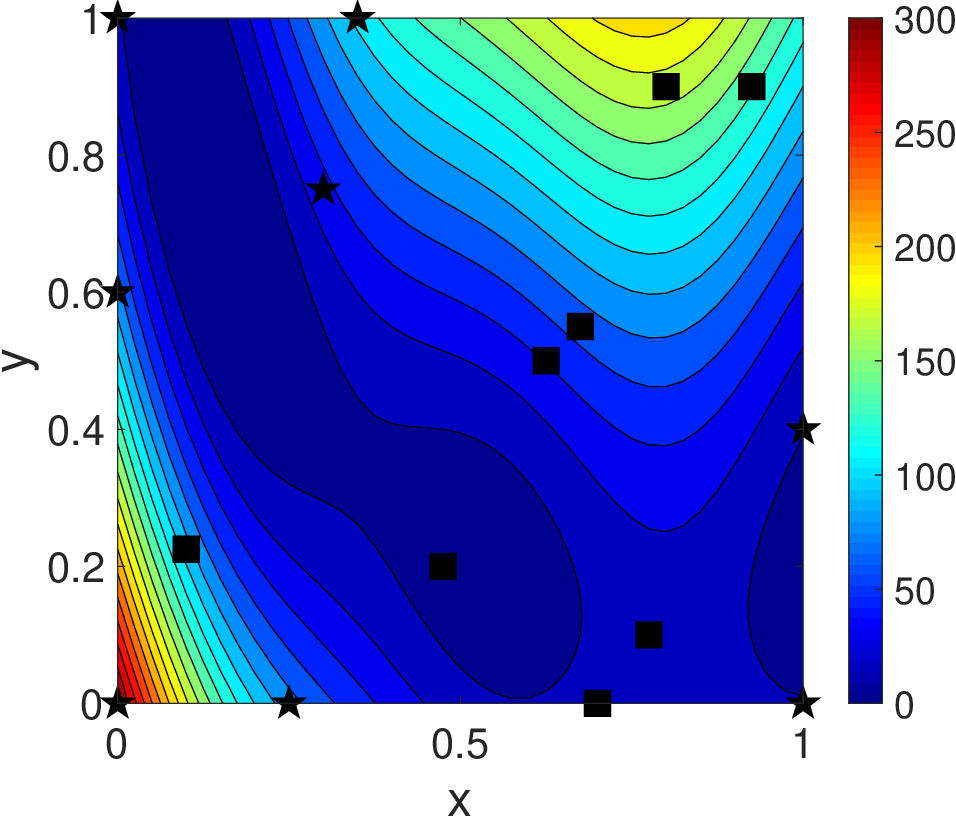}}\quad
\subfigure[PhIK $\hy$]{
\includegraphics[width=0.4\textwidth]{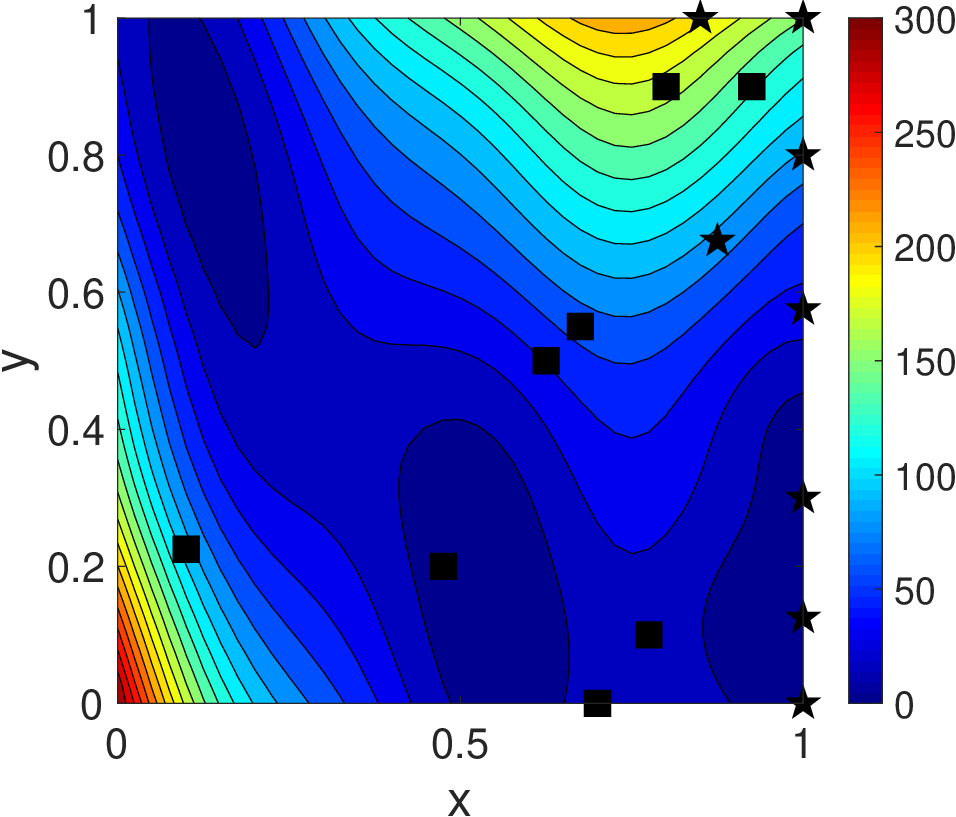}}  \\
\subfigure[Krig $\hs$]{
\includegraphics[width=0.4\textwidth]{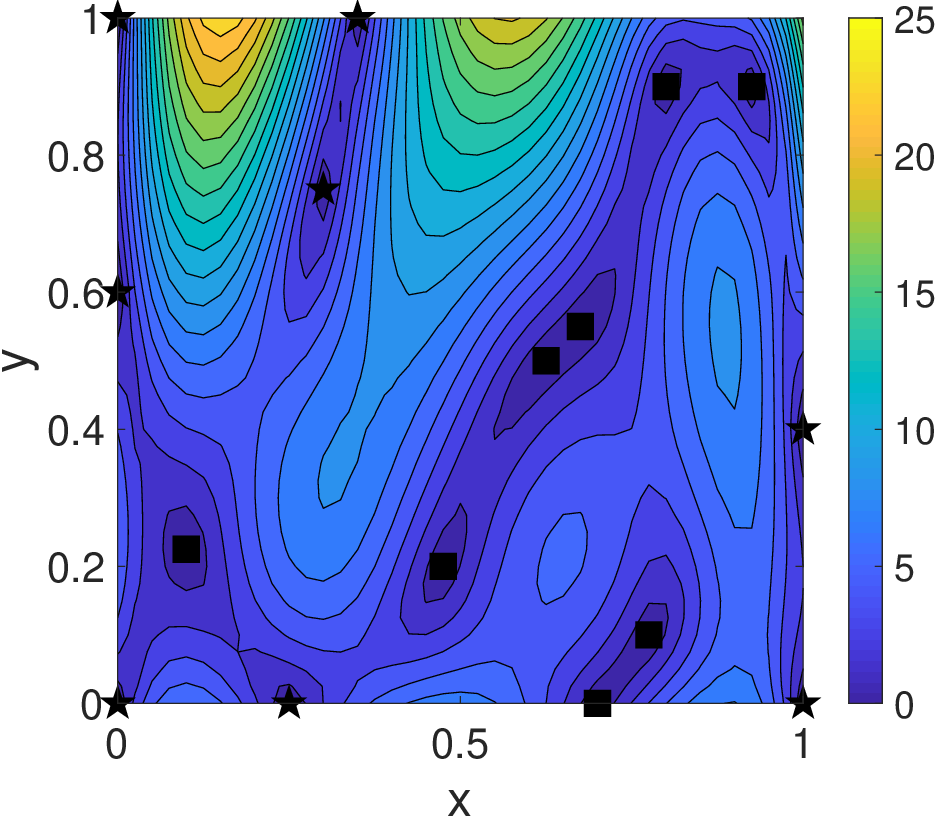}}\quad
\subfigure[PhIK $\hs$]{
\includegraphics[width=0.425\textwidth]{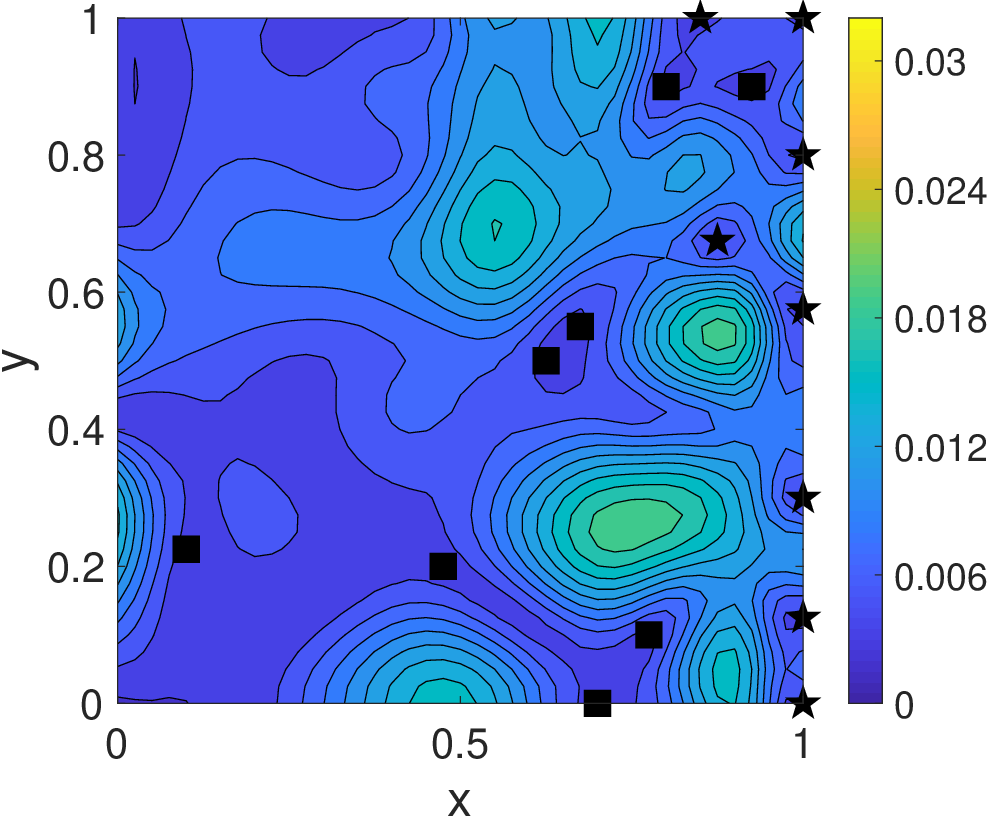}}\\
\subfigure[Krig $\hy-\ey$]{
\includegraphics[width=0.4\textwidth]{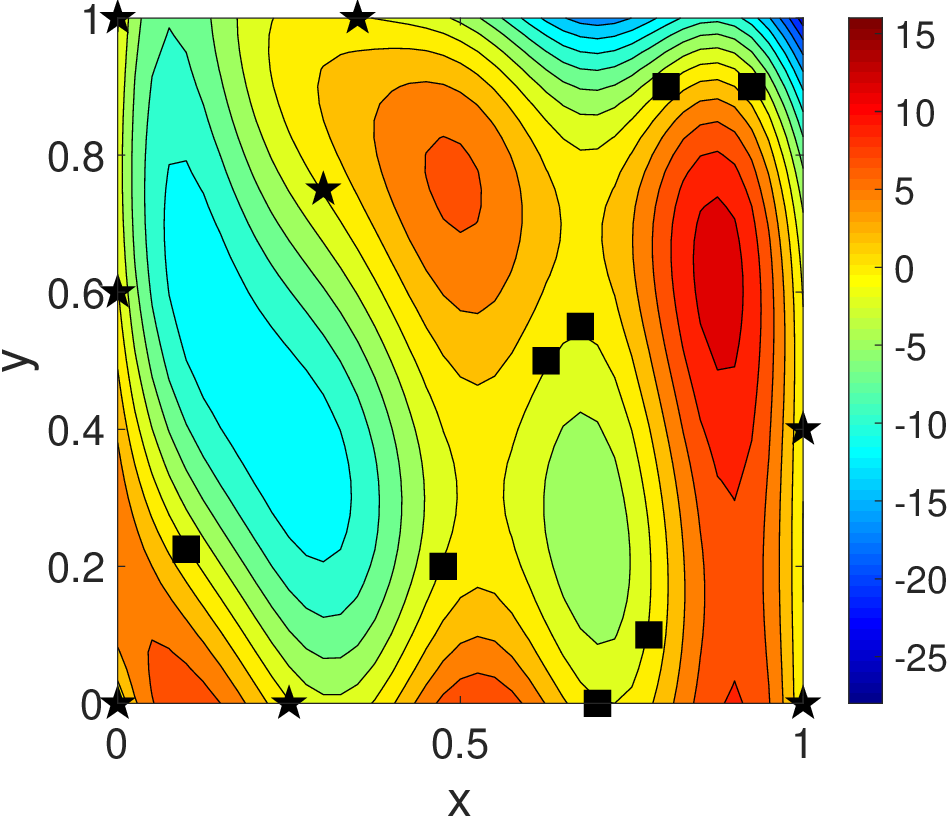}}\quad
\subfigure[PhIK $\hy-\ey$]{
\includegraphics[width=0.39\textwidth]{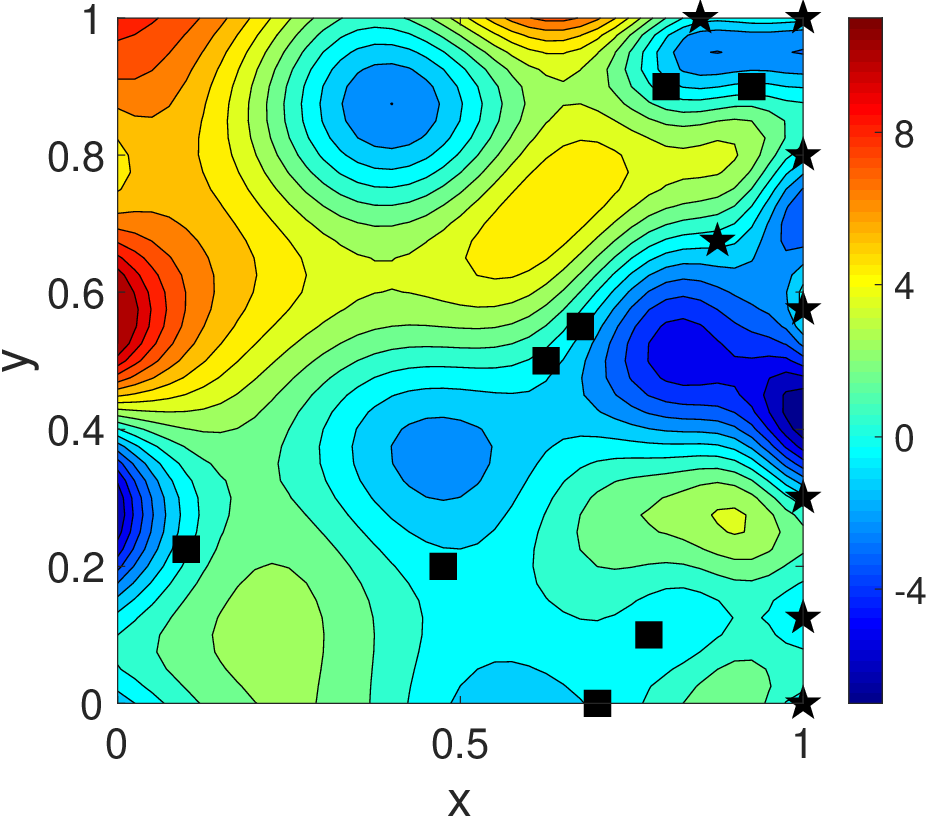}}
\caption{Branin function: reconstruction of the field by PhIK via active
  learning. Black squares mark the locations of the original eight observations, 
  and stars are newly added observations. First row: $\hy$; second row: $\hs$; 
  third row: $\hy-\ey$.}
\label{fig:branin_act}
\end{figure}

Figure~\ref{fig:branin_comp_l2} shows the relative error
$\Vert\hy-\ey\Vert_2/\Vert \ey \Vert_2$
in Kriging and PhIK as a function of the observation numbers,
where the first eight are the ``original" observations and the rest are added
according to the active learning algorithm. With the original eight observations,
the PhIK result (about $8\%$ error) is much better than the Kriging (more than 
$50\%$ error). As more observations are added by the active learning algorithm, 
the error of Kriging decreases 
quickly to nearly $1\%$ (24 observations).
The error of PhIK reduces from $8\%$ to $4\%$ (10 observations), 
then it slowly reduces to $3\%$ (24 observations).
According to the discussion in Section~\ref{sec:act}, this is because
Kriging introduces a new basis function $k(\bm x, \bm x^{(N+1)})$ when a new
observation at location $\bm x^{(N+1)}$ is available 
and the Gaussian kernel is suitable for approximating smooth functions like $f$. 
On the other hand, using $12$ observations is sufficient to identify an accurate
approximation of the exact field in the space spanned by $\Hy$, and adding few
observations doesn't result in significant
change. This is why the convergence is very slow with more than $12$ observations. 
We also note that if we use empirical mean only to approximate
$\ey$, the relative error is $18\%$. This indicates that observation data are
incorporated effectively to improve the prediction accuracy.

\begin{figure}[!h] \centering
  \includegraphics[width=0.4\textwidth]{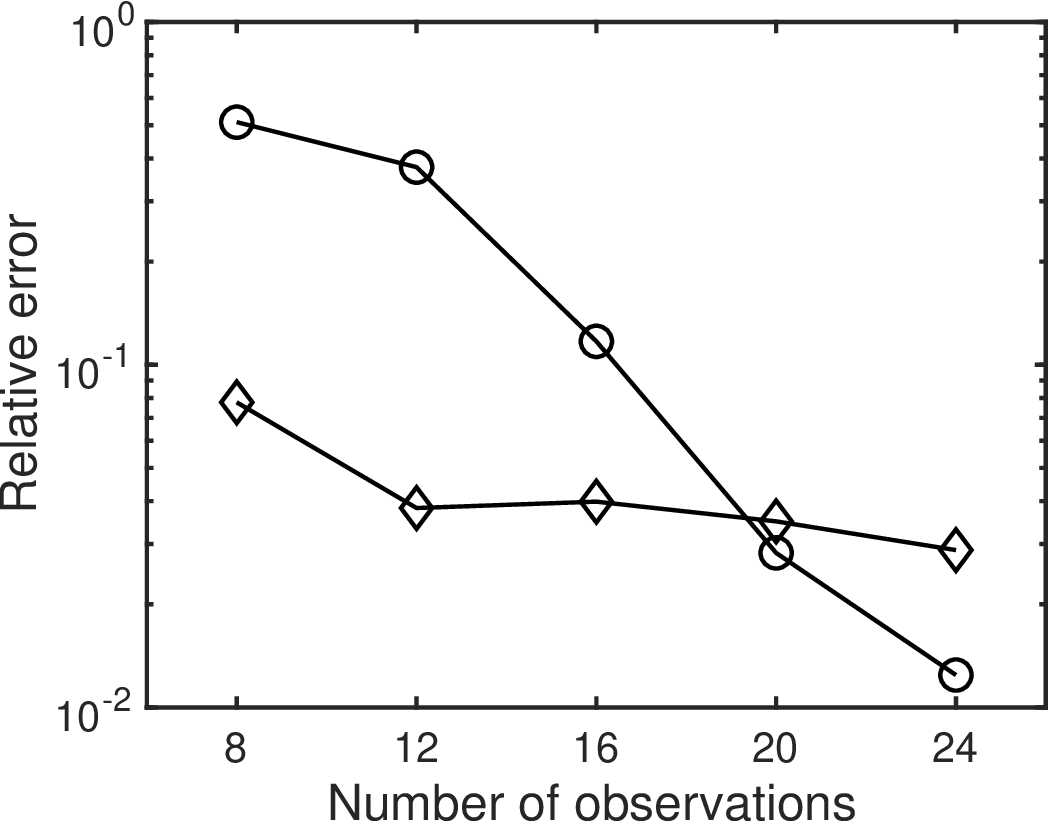}
  \caption{Branin function: relative error of the reconstructed field $\Vert\hy
  -\ey\Vert_2/\Vert\ey\Vert_2$ using Kriging (``$\circ$") and PhIK
  (``$\diamond$") with different numbers of total observations via active
learning.} 
\label{fig:branin_comp_l2} 
\end{figure}

\subsection{Heat transfer in a hallow sphere}
In the second example, we consider a heat transfer problem in a hallow sphere
$D = B_4(0) - B_2(0)$, where $B_r(0)$ is a ball with radius $r$ centered at $0$.
The governing equation is
\begin{equation}
  \label{eq:heat}
  \left\{
  \begin{aligned}
    & \dfrac{\partial u(\bm x,t)}{\partial t} - \nabla\cdot(\kappa \nabla u(\bm x,t)) = 0, \quad \bm x\in D, \\
    & \kappa\dfrac{\partial u(\bm x, t)}{\partial \bm n} =
      \theta^2(\pi-\theta)^2\phi^2(\pi-\phi)^2 , \quad  \text{if}~\Vert \bm
    x\Vert_2=4~\text{and}~\phi\geq 0, \\
    & \kappa\dfrac{\partial u(\bm x, t)}{\partial \bm n} = 0, \quad \text{if}~\Vert\bm
    x\Vert_2=4~\text{and}~\phi<0, \\
    & u(\bm x, t) = 0, \quad \text{if}~\Vert \bm x \Vert_2=2,
  \end{aligned}
  \right .
\end{equation}
where $\bm n$ is the outward-pointing normal, $\theta$ and $\phi$ are azimuthal
and elevation angles of points in the sphere. The initial condition is $0$ for
all points. The exact heat conductivity is $\kappa=1.0+\exp(0.05u)$. We solve
this equation using MATLAB's PDE Tool Box~\cite{matlabpde}. Specifically, we use
quadratic element with DOF being 12,854, and set $y(\bm x)=u(\bm x, 10)$, which
is shown in Figure~\ref{fig:heat_truth}. 
\begin{figure}[thbp]
  \centering
    \includegraphics[width=0.25\textwidth, trim=150 100 190 40, clip]{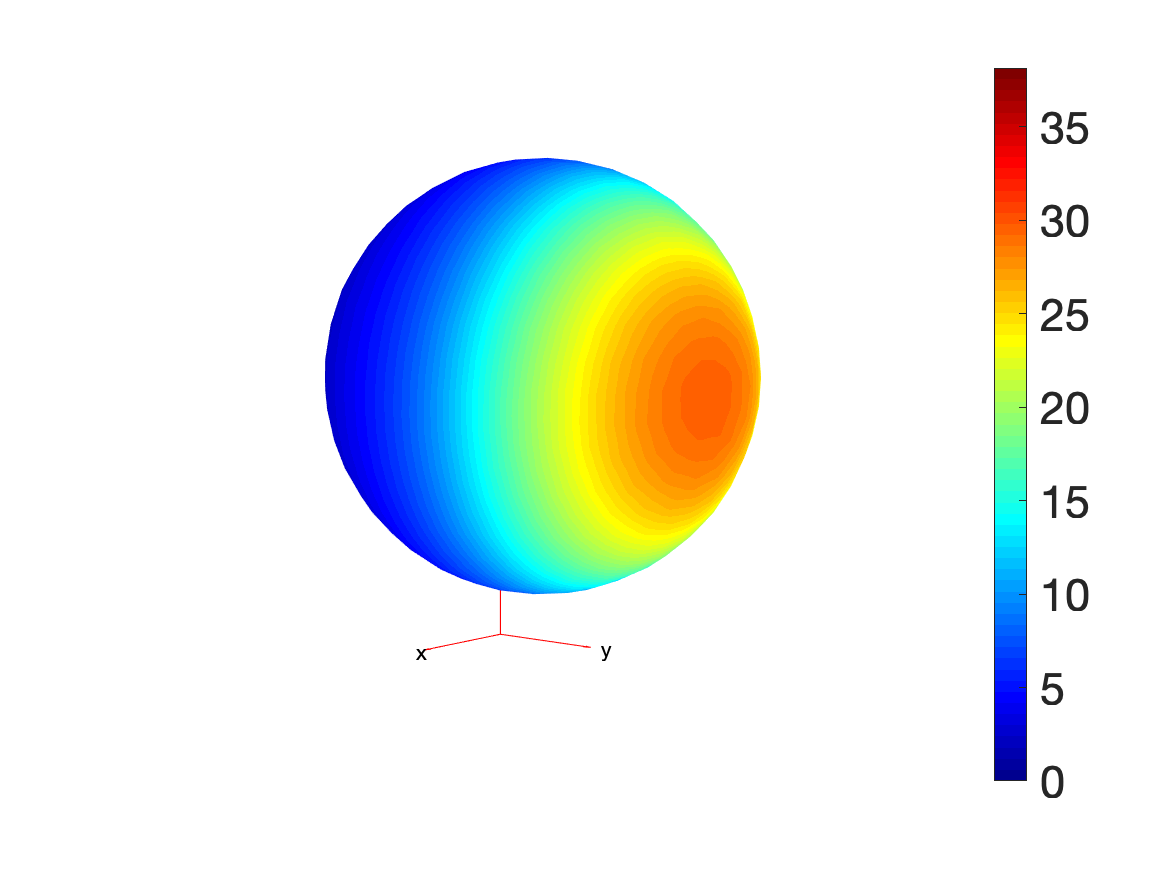}\quad
    \includegraphics[width=0.09\textwidth, trim=410 30 35 30, clip]{./figures/heat_exact_high}
  \caption{Heat transfer: ground truth $\ey$.} 
  \label{fig:heat_truth}
\end{figure}
\subsubsection{Field reconstruction}
We assume that accurate
observation data are available at six locations: 
$\bm X=\{(3,0,0), (-3,0,0), (0, 3, 0), (0, -3, 0), (0,0,3), (0,0,-3)\}$,
which enables us to use Kriging to reconstruct $\ey$.
Figure~\ref{fig:heat_mean}~(a) present the $\hy$ of Kriging, which
significantly deviates from $\ey$. Figure~\ref{fig:heat_std}~(a) shows the $\hs$
of Kriging, which is close to $2$ at many points.
Next, we assume that the exact $\kappa(u)$ is not known, and a stochastic model is
proposed as $\kappa(u;\omega)=2.0+u\xi(\omega)$, where 
$\xi\sim\mathcal{U}[0.2,0.4]$. We generate $M=100$ samples of $\xi$ and solve 
the heat equation $100$ times to obtain corresponding simulation results $\Hy$. 
Figure~\ref{fig:heat_mean}~(b) shows that the empirical mean of $\Hy$ deviate
significantly from $\ey$, the relative error is $17\%$.
Figure~\ref{fig:heat_std}~(b) illustrates the standard deviation
$\sigma_{MC}(\bm x)$ of $\Hy$ at each point. It demonstrates that
$\sigma_{MC}(\bm x)$ is large in the region with higher temperature, which is
consistent with the setup of $\kappa(u;\xi(\omega))$. Then, the PhIK results are
shown in Figure~\ref{fig:heat_mean}~(c) and Figure~\ref{fig:heat_std}~(c).
PhIK's $\hy$ matches the reference solution $\ey$ quite well and its $\hs$ is
rather small.
\begin{figure}[thbp]
  \centering
    \subfigure[Kriging $\hy$]{
    \includegraphics[width=0.25\textwidth, trim=150 100 160 40, clip]{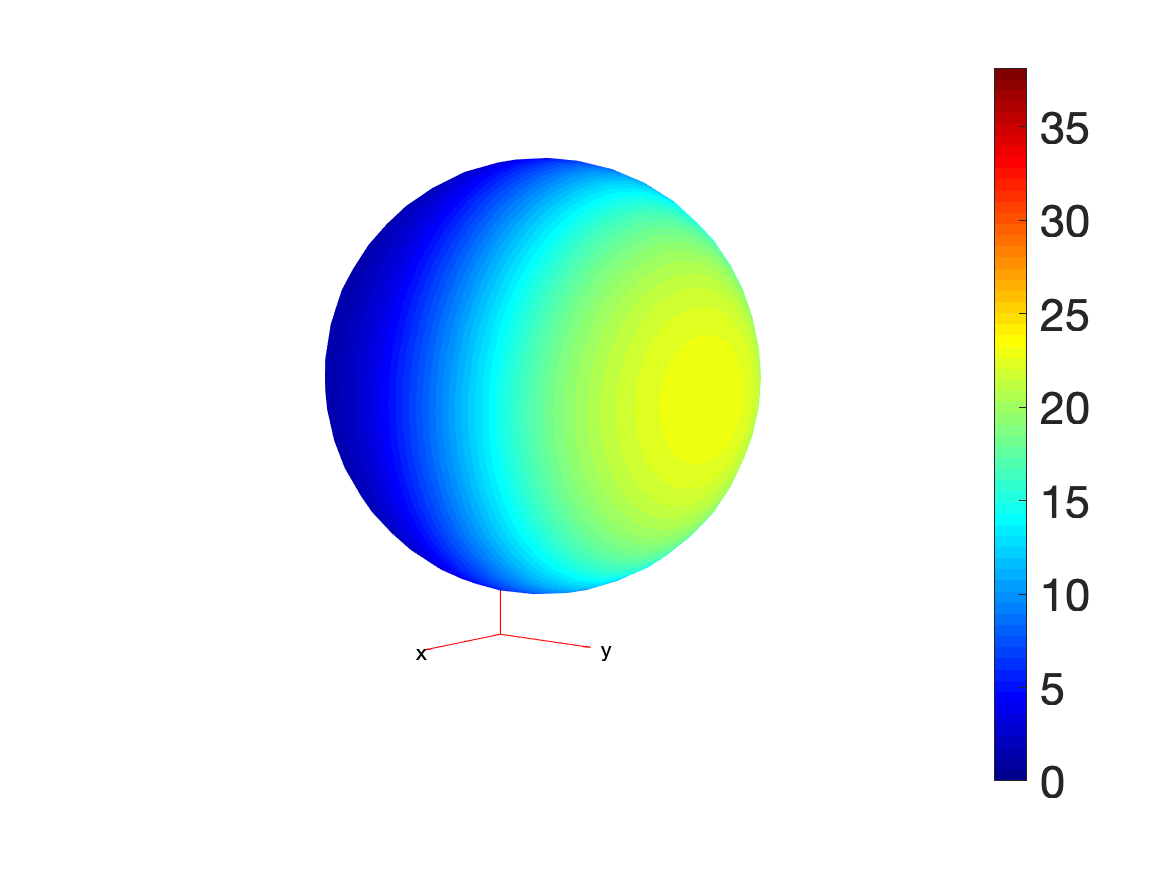}}\quad
  \subfigure[MC empirical mean]{
    \includegraphics[width=0.25\textwidth, trim=150 100 160 40, clip]{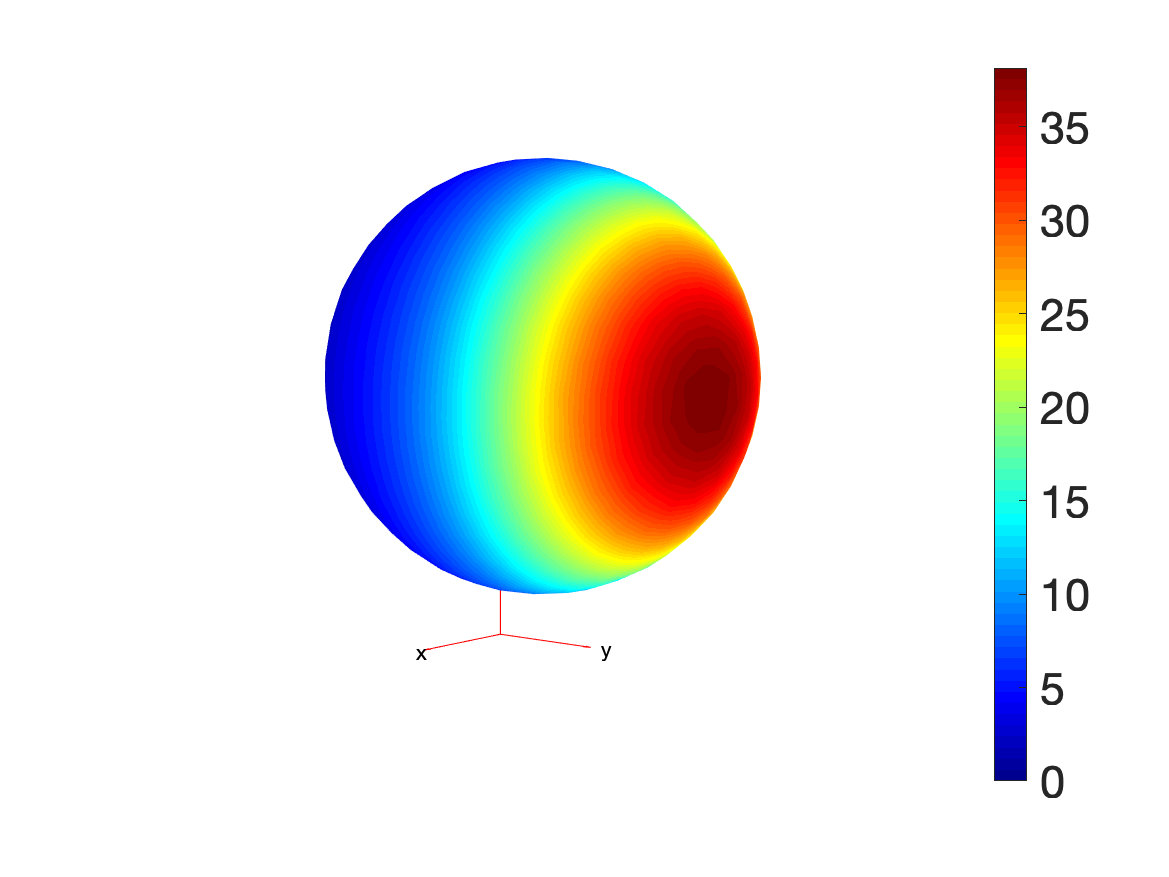}}\quad
  \subfigure[PhIK $\hy$]{
    \includegraphics[width=0.25\textwidth, trim=150 100 160 40, clip]{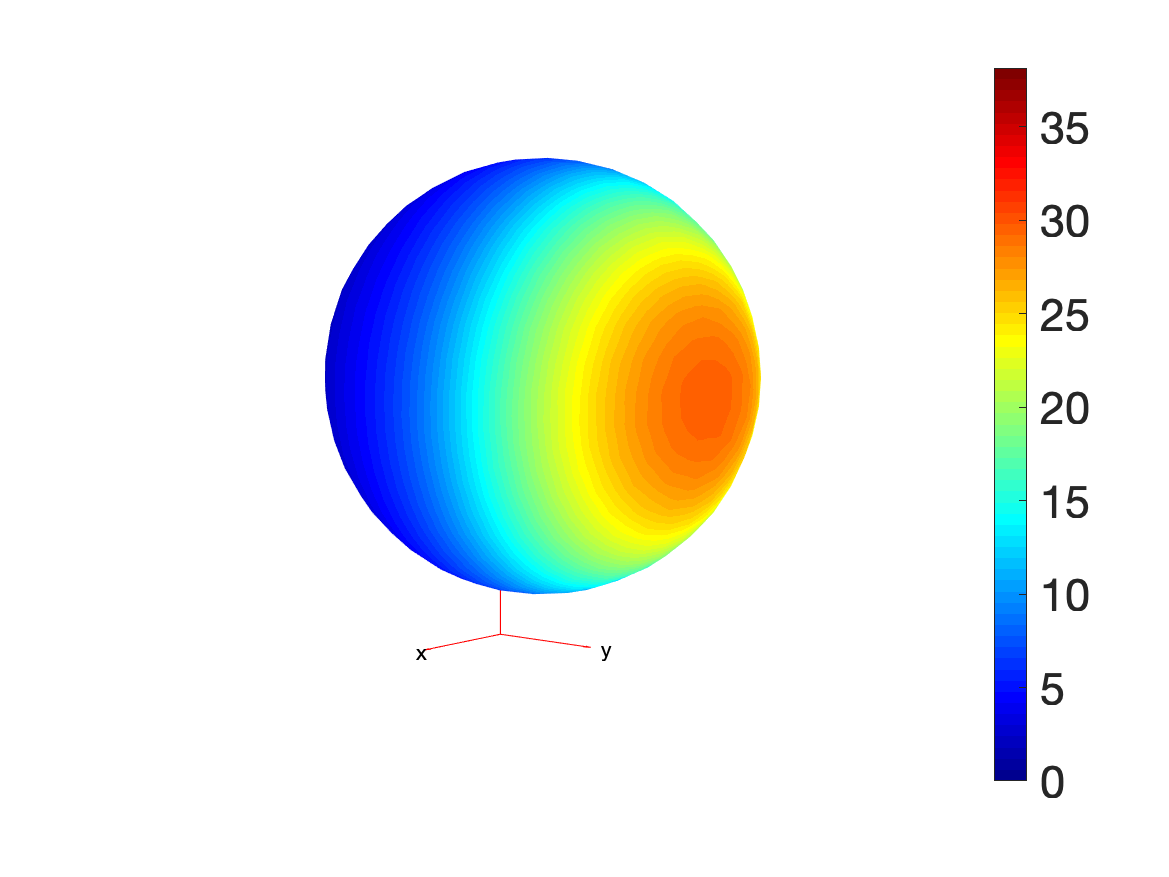}}
    \includegraphics[width=0.08\textwidth, trim=420 30 30 30, clip]{./figures/heat_phik_6obs_high}
  \caption{Heat transfer: mean values of different methods.}
  \label{fig:heat_mean}
\end{figure}

\begin{figure}[thbp]
  \centering
  \subfigure[Kriging]{
    \includegraphics[width=0.3\textwidth, trim=150 70 10 0, clip]{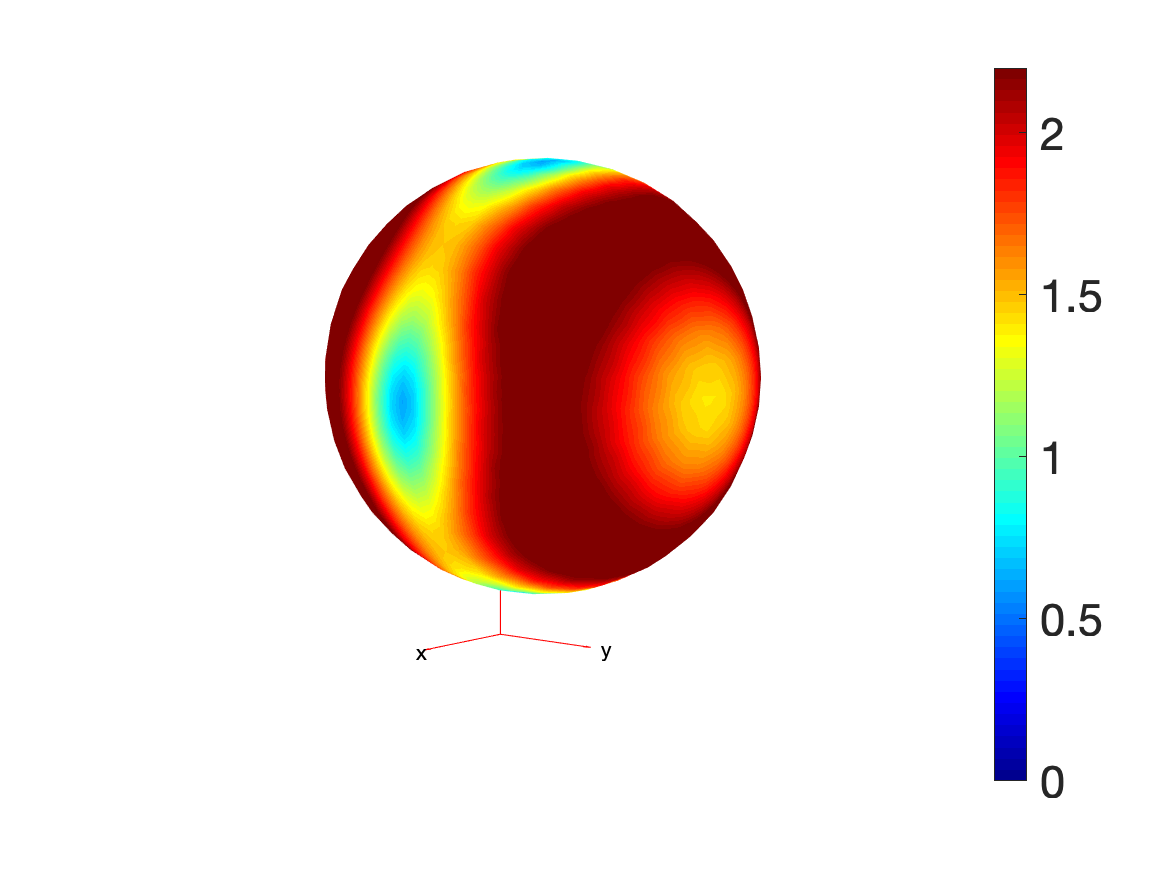}}\quad
  \subfigure[MC]{
    \includegraphics[width=0.3\textwidth, trim=150 70 10 0, clip]{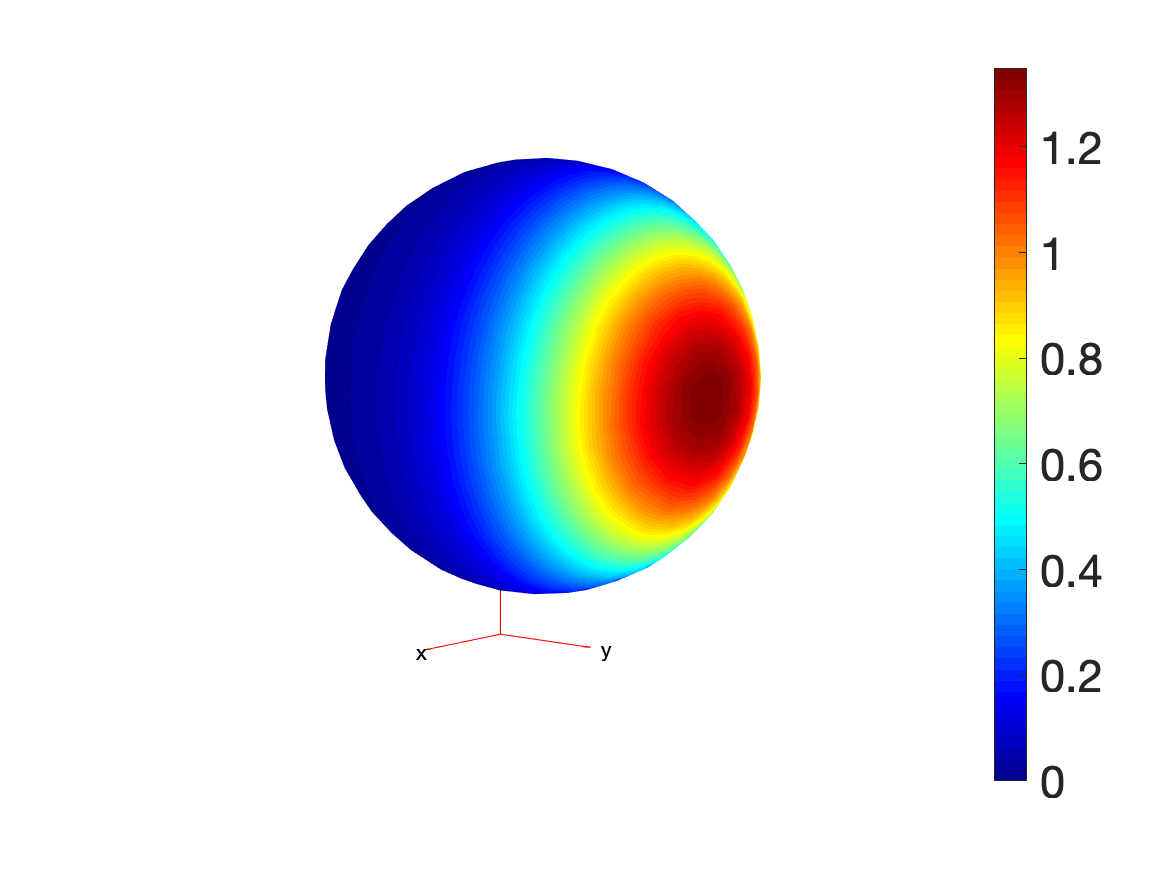}}\quad
  \subfigure[PhIK]{
    \includegraphics[width=0.3\textwidth, trim=150 70 10 0, clip]{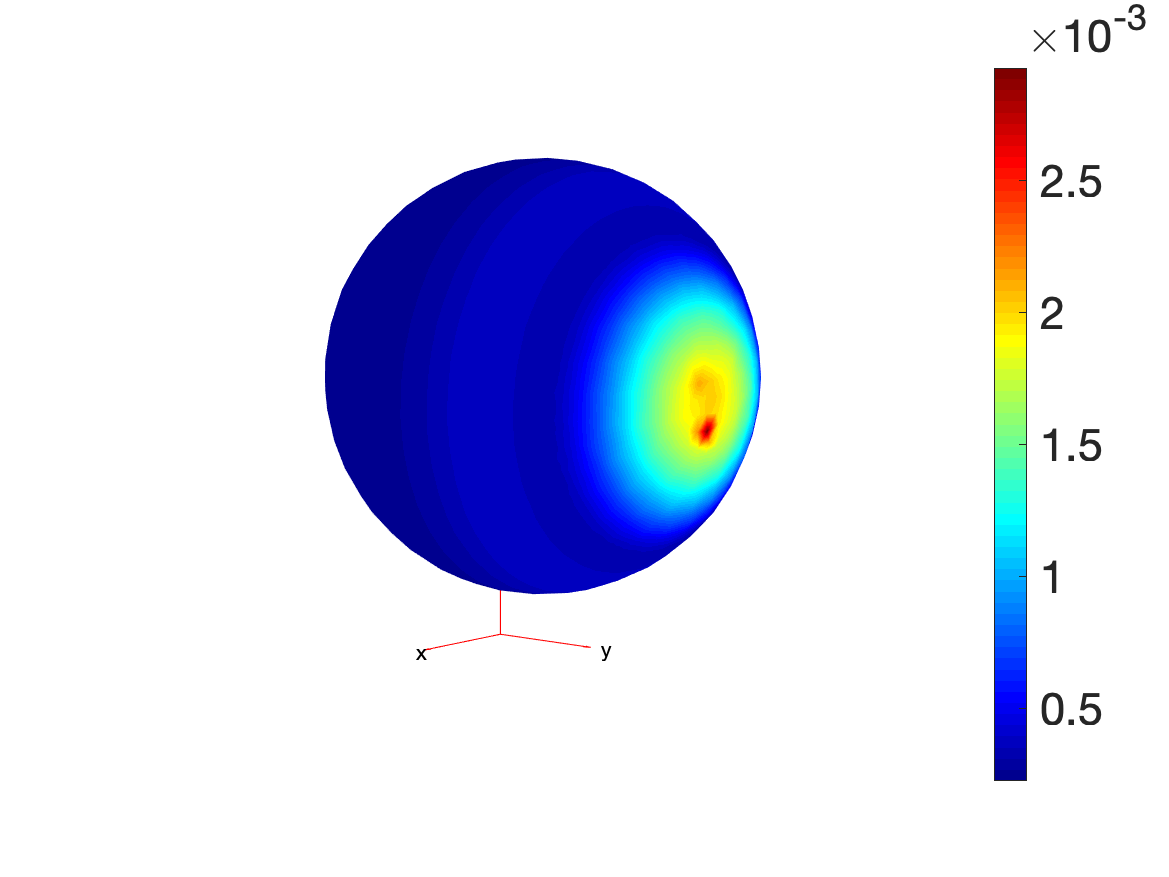}}
  \caption{Heat transfer: standard deviations of different methods.}
  \label{fig:heat_std}
\end{figure}

\subsubsection{Active learning}
Moreover, we use the active learning algorithm to add new observation data one
by one. The location of the six original locations are shown in
Figure~\ref{fig:heat_obs_loc}~(a). The first six new locations identified by
active learning with Kriging is shown in Figure~\ref{fig:heat_obs_loc}~(b)
marked as stars. These stars are actually near the boundary of domain $D$. On
the other hand, the new locations identified by active learning with PhIK,
illustrated in Figure~\ref{fig:heat_obs_loc}~(c) are
located in the region with large $\sigma_{MC}$ as shown in
Figure~\ref{fig:heat_std}~(b).
\begin{figure}[thbp]
  \centering
  \subfigure[Initial]{
    \includegraphics[ width=0.25\textwidth, trim=150 100 150 60, clip]{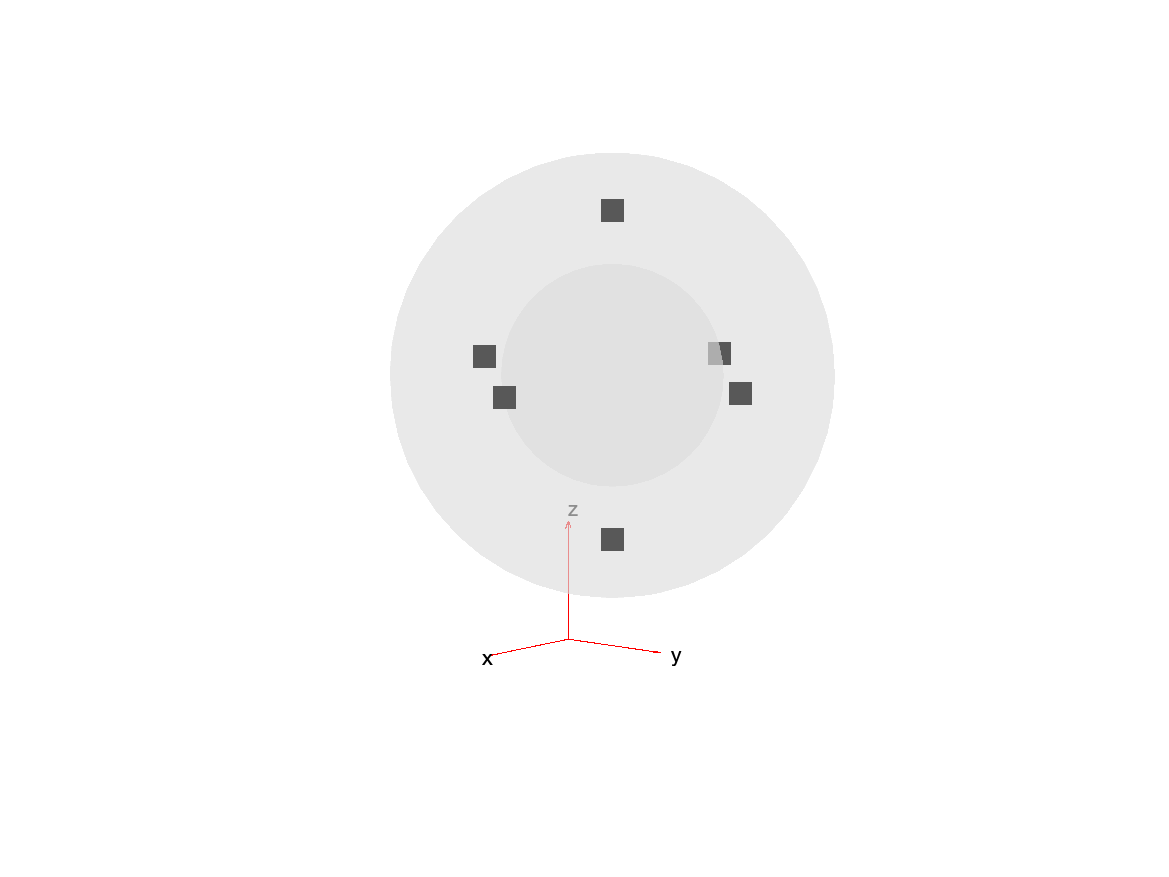}}\quad
  \subfigure[Krig]{
    \includegraphics[ width=0.25\textwidth, trim=150 100 150 60, clip]{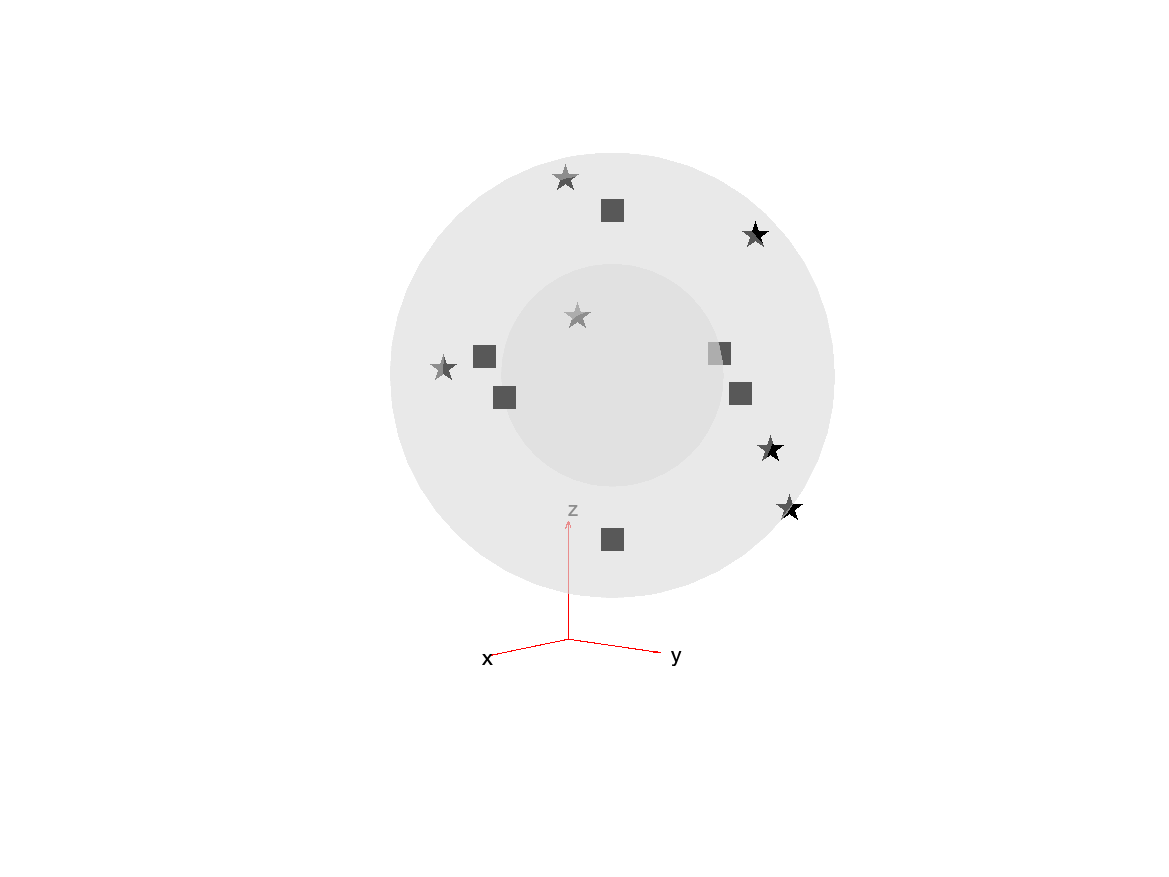}}\quad
  \subfigure[PhIK]{
    \includegraphics[ width=0.25\textwidth, trim=150 100 150 60, clip]{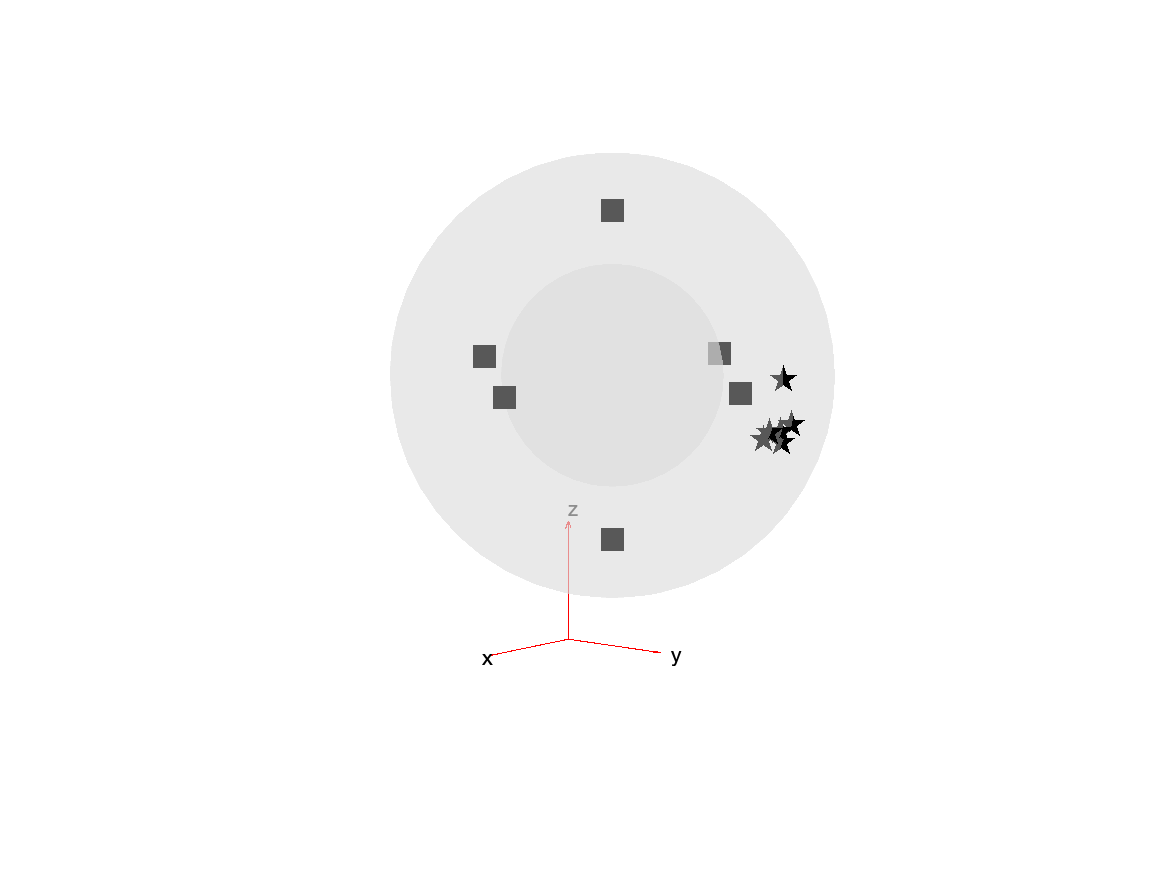}}
  \caption{Heat transfer: additional observation locations identified via
  active learning.}
  \label{fig:heat_obs_loc}
\end{figure}

Finally, we keep adding new observations and present the change of relative
error in Figure~\ref{fig:heat_comp_l2}. For this example, PhIK dramatically
outperforms Kriging, as the accuracy of the former is three orders of magnitude
better than the latter. Also, in both methods, the relative errors reduces as
more observation data are available.
\begin{figure}[thbp]
\centering
\includegraphics[width=0.4\textwidth]{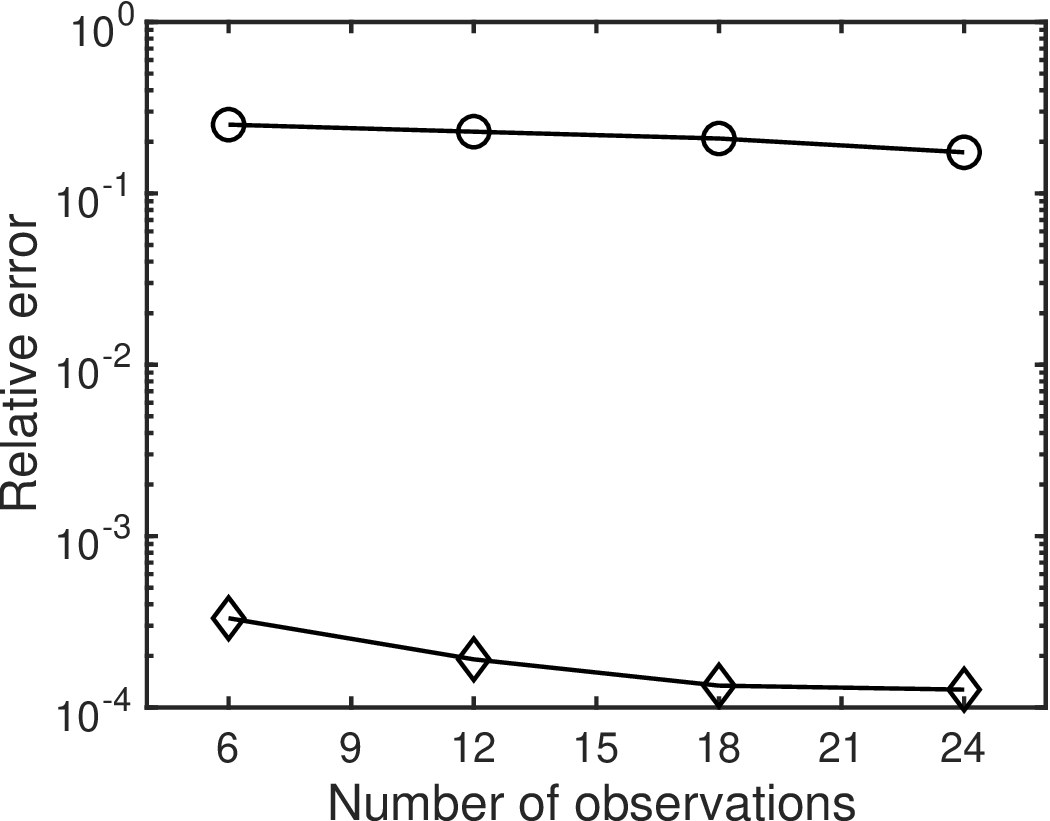}
\caption{Heat transfer: relative error of reconstructed field 
$\Vert\hy-\ey\Vert_2/\Vert\ey\Vert_2$ using Kriging 
  (``$\circ$") and PhIK (``$\diamond$") with different numbers of total observations via 
  active learning.}
\label{fig:heat_comp_l2}
\end{figure}
Moreover, as an example of preserving linear constraint, the Dirichlet boundary
condition is preserved quite well in $\hy$ as the maximum of $\hy$ at grids
located on the inner sphere is less than $0.01$.

\subsection{Solute transport in heterogeneous porous media}

In the third example, we consider steady-state flow, and advection and
dispersion of conservative tracer with concentration $C_e(\bm x,t)$ in a
heterogeneous porous medium with known initial and boundary conditions and the
unknown  hydraulic conductivity $K(\bm x)$. We assume that measurements of
$C_e(\bm x,t)$ are available at several locations at different times. The flow
and transport in porous media can be described by conservation laws, including a
combination of the continuity equation and Darcy law:
\begin{equation} \label{eq:head} \begin{cases} \nabla\cdot[K(\bm x;\omega)\nabla
  h(\bm x;\omega)]=0, \qquad \bm x  \in  D, \\ \dfrac{\partial h(\bm
  x;\omega)}{\partial \bm n} = 0, \qquad x_2=0 \quad\text{or}\quad x_2=L_2, \\
  h(x_1=0,x_2;\omega) = H_1 \quad \text{and} \quad h(x_1=L_1,x_2;\omega) = H_2,
\end{cases} \end{equation}
where $D=[0,L_1]\times [0, L_2]=[0,256]\times [0, 128]$, the unknown
conductivity is modeled as the random log-normally distributed field $K(\bm
x;\omega) = \exp(Z(\bm x; \omega))$ with the known exponential covariance
function $\cov\{Z(\bm x),Z(\bm x')\}= \sigma^2_Z\exp(-|\bm x - \bm x'|/l_z)$
with the variance $ \sigma^2_Z=2$, correlation length $l_z=5$, 
$h(\bm x;\omega)$ is the hydraulic head, and $\omega\in\Omega$. The solute transport is governed by the advection-dispersion equation \cite{emmanuel2005mixing, lin2009efficient}:
\begin{equation}
\label{eq:con}
\begin{cases}
  \begin{aligned}
    \dfrac{\partial C(\bm x,t;\omega)}{\partial t} + & \nabla\cdot(\bm v(\bm x;\omega) C(\bm x,t;\omega))  =  \\
 &\nabla\cdot\left[\left(\dfrac{D_w}{\tau}+ \boldsymbol{\alpha} \Vert\bm v(\bm x;\omega)\Vert_2\right)\nabla C(\bm x,t;\omega)\right], 
    \qquad \bm x~\text{in}~D, \\
\end{aligned} \\
 C(\bm x,t=0;\omega) = \delta(\bm x-\bm x^*), \\ 
 \dfrac{\partial C(\bm x;\omega)}{\partial\bm n} = 0, \qquad \quad x_2=0
 \quad\text{or}\quad x_2=L_2\quad\text{or}\quad x_1=L_1, \\  
  C(x_1 = 0, x_2;\omega) = 0,  
\end{cases}
\end{equation}
where $C(\bm x,t;\omega)$ is the solute concentration defined on 
$D\times [0, T]\times\Omega$, the solute is instantaneously injected at 
$\bm x^*=(50, 64)$, $\bm v(\bm x;\omega)=-K(\bm x;\omega)\nabla h(\bm x;\omega)/\phi$ 
is the average pore velocity, $\phi$ is the porosity, $D_w$ is the diffusion 
coefficient, $\tau$ is the tortuosity, and $\boldsymbol{\alpha}$ is the 
dispersivity tensor with the diagonal components $\alpha_L$ and $\alpha_T$.
In the present work, the transport parameters are set to $\phi=0.317$, 
$\tau=\phi^{1/3}$, $D_w=  2.5 \times 10^{-5}$ m$^2$/s, $\alpha_L=5$ m, and $\alpha_T=0.5$ m. 

We generate $M=1000$ realizations of $Z(\bm x)$ using the SGSIM (sequential 
Gaussian simulation) code \cite{deutsch1992gslib} and solve the governing 
equations for each realization of $K(\bm x)=\exp(Z(\bm x))$ using the finite 
volume code STOMP (subsurface transport over multiple phases) \cite{white2006stomp}
with the mesh size $1$m $\times$ $1$m, i.e., $256\times 128$
uniform grids. 
Here, we study the results at $t=8$ days. The reference solution is generated 
as one of the $1000$ solutions of the governing equations and is shown in 
Figure~\ref{fig:adv_truth} with observation locations. We assume that six 
uniformly spaced observations are available near the domain boundary, and nine 
randomly placed observations are given within the domain $D$. Because Kriging is
known to be less accurate for extrapolation (also illustrated in the first 
numerical example), it is common to collect data near the boundary of the domain
of interest in practice, e.g., \cite{dai2017geostatistics}. 
\begin{figure}[h]
\centering
\includegraphics[width=0.53\textwidth]{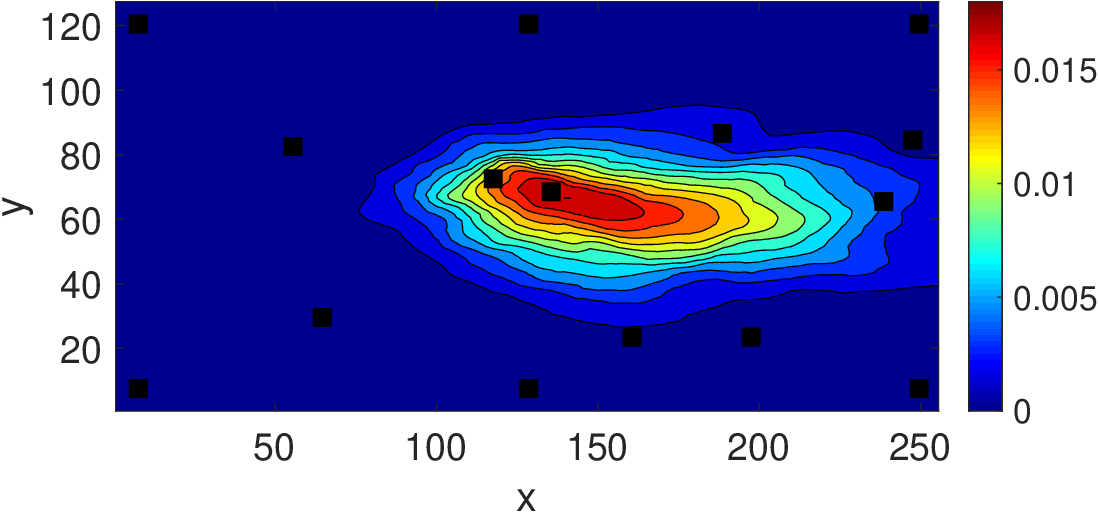}
\caption{Solute transport: reference solution of the solute concentration when $t=8$ days and 
  observation locations (black squares).}
\label{fig:adv_truth}
\end{figure}

\subsubsection{Field reconstruction}

We first use Kriging to reconstruct $\ey$ using 15 observations.
Figures~\ref{fig:adv_pred}~(a) and (b) present $\hy$ and
the error $\hy-\ey$. 
We can see that Kriging performs poorly, and the relative error 
is larger than $40\%$.
Next, we assume that only $10$ simulations (i.e., $M_H=10$) on
$256\times 128$ uniform grids are available and use them in the MC-based PhIK
to reconstruct $\ey$. Specifically, the mean and covariance matrix are computed from 
Eqs.~\eqref{eq:mc_mean} and \eqref{eq:mc_kernel} using ensembles
$\Hyh$. 
Figure~\ref{fig:adv_pred}(c) and (d) present corresponding $\hy$ and $\hy-\ey$,
respectively. These results are better than the Kriging as the relative
error is around $26\%$. Finally, we include $M_L=500$ coarse-resolution 
simulations with $64\times 32$ uniform grids to use MLMC 
Eqs.~\eqref{eq:mlmc_mean} and \eqref{eq:mlmc_cov}, to approximate the mean
and covariance. The corresponding $\hy$ and $\hy-\ey$ are presented in Figure~\ref{fig:adv_pred}(e) and (f), respectively.
Adding the coarse simulations significantly improves prediction as the
relative error reduces to approximately $14\%$. Of note, if we only use the MLMC
mean to approximate $\ey$, the relative error is $18\%$.
\begin{figure}[h]
\centering
\subfigure[$\hy$ by Kriging]{
\includegraphics[width=0.45\textwidth]{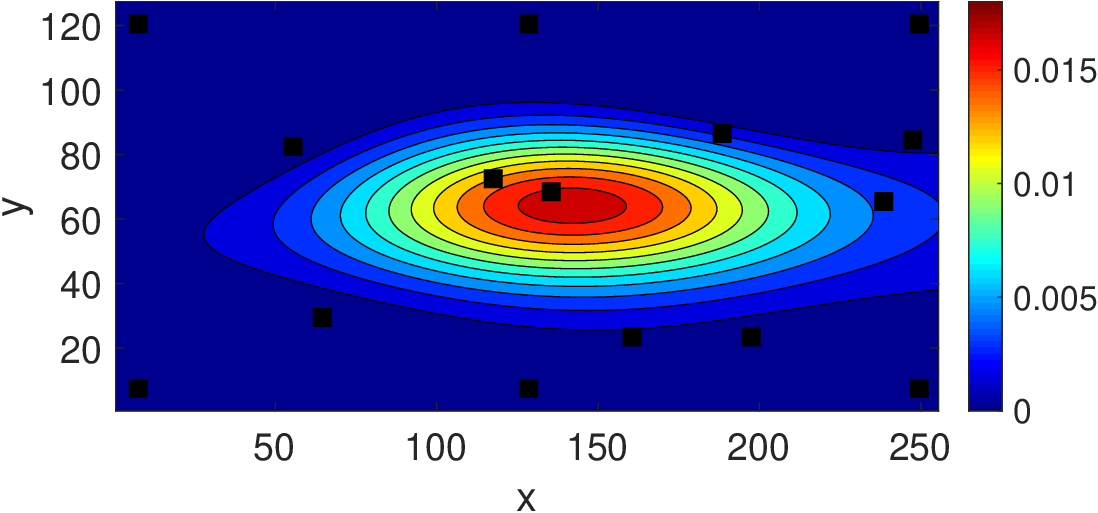}}\quad
\subfigure[$\hy-\ey$ by Kriging]{
\includegraphics[width=0.45\textwidth]{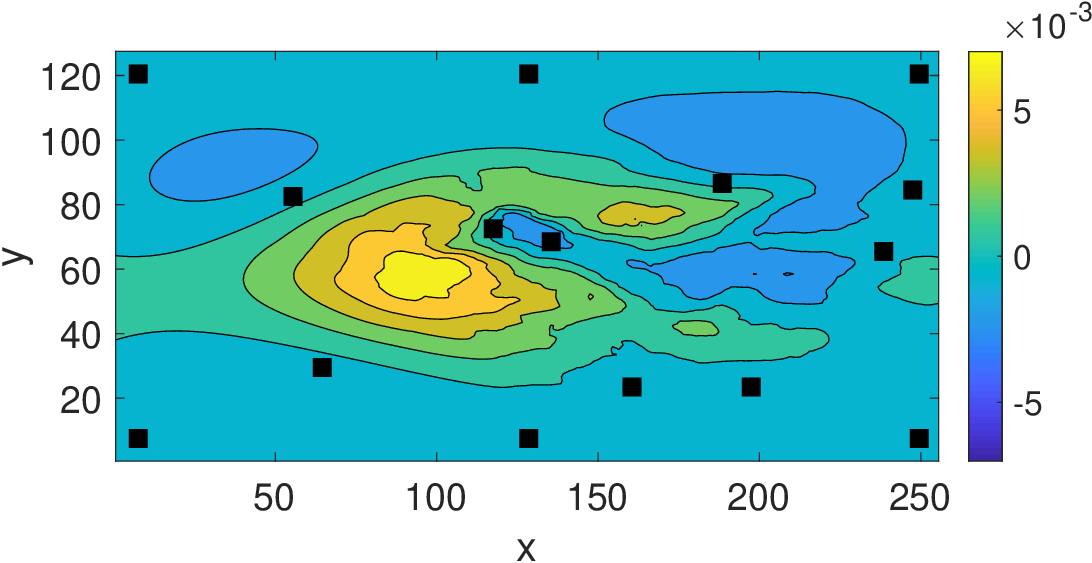}} \\
\subfigure[$\hy$ by MC-PhIK]{
\includegraphics[width=0.45\textwidth]{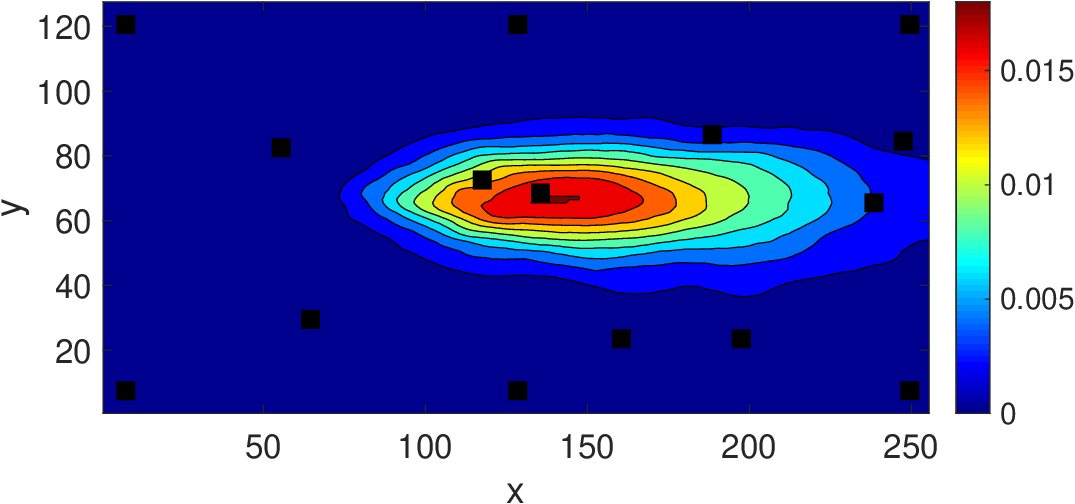}}\quad
\subfigure[$\hy-\ey$ by MC-PhIK]{
\includegraphics[width=0.45\textwidth]{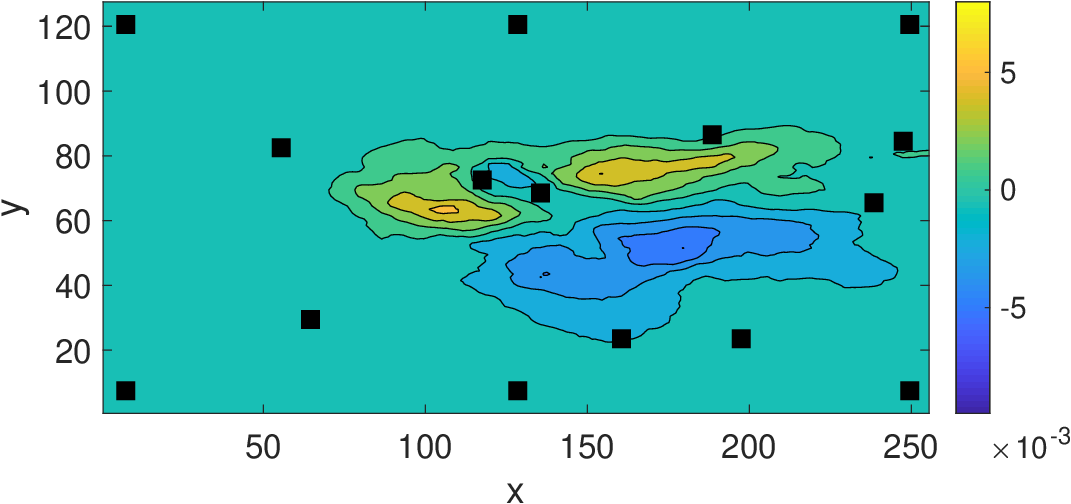}} \\
\subfigure[$\hy$ by MLMC-PhIK]{
\includegraphics[width=0.45\textwidth]{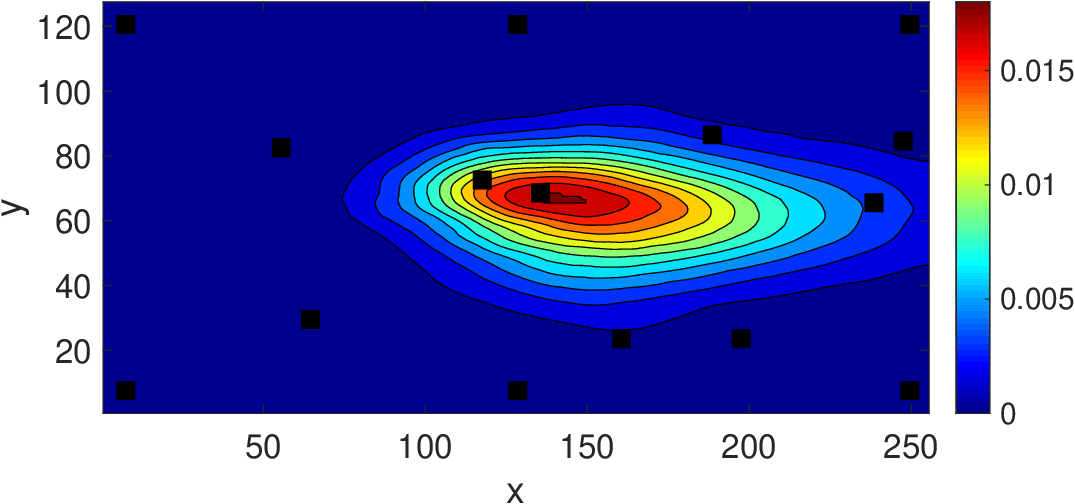}}\quad
\subfigure[$\hy-\ey$ by MLMC-PhIK]{
\includegraphics[width=0.45\textwidth]{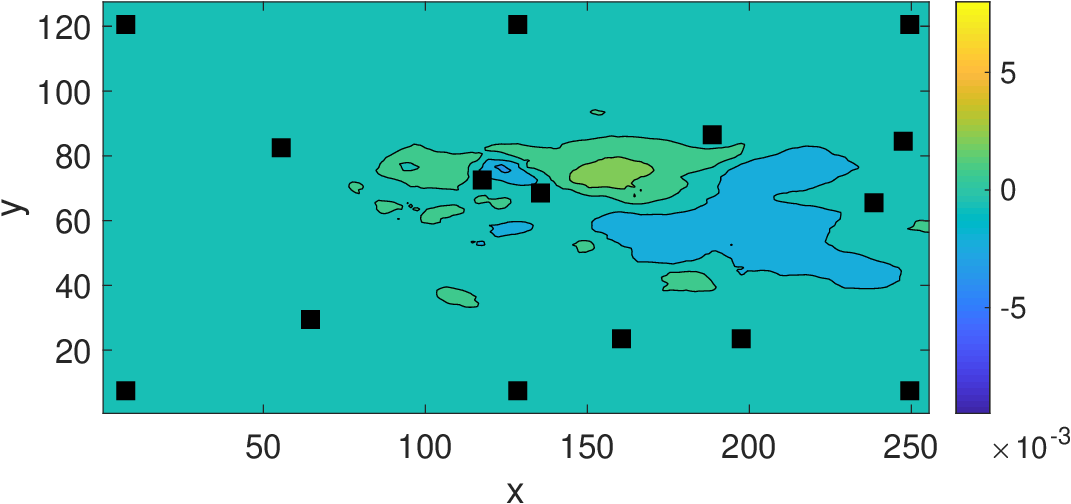}} 
\caption{Solute transport: reconstructed solute concentration field $\hy$ by Kriging, MC-based
  PhIK with $10$ high-resolution simulations, MLMC-based PhIK with $10$
  high-resolution (mesh size $1$m$\times1$m) simulations and $500$ low-resolution
  (grid size $4$m$\times 4$m) simulations, and their difference from the exact field 
  $\hy-\ey$. Black squares are the observation locations.}
\label{fig:adv_pred}
\end{figure}

Next, we study how the MLMC-based PhIK's accuracy depends on the number of 
high-resolution simulations $M_H$ for the fixed number of low-resolution 
simulations $M_L=500$. Figure~\ref{fig:adv_mlmc_err}(a) shows how the MLMC-based 
PhIK error $\Vert\hy-\ey\Vert_2/\Vert \ey\Vert_2$ decreases with 
increasing $M_H$. For comparison, we also compute error in the MC-based PhIK for
the same number of $M_H$. It is clear that MC-based PhIK is less accurate than 
MLMC-based PhIK, especially for small $M_H$.  
Also, the smaller error in MLMC-based PhIK is achieved with a smaller 
computational cost than that of MC-based PhIK. In this example, the number of 
degrees of freedom in the low-resolution simulation is $1/16$ of that in the 
high-resolution simulation. As we use an implicit scheme for the dispersion
operator and an explicit scheme for the advection operator, according to the CFL 
condition, the time step in a low-resolution simulation is approximately four 
times larger than the time step in a high-resolution simulation.  
Therefore, the computational cost of a high-resolution simulation is at least 
$64$ times that of a low-resolution simulation and the cost of $500$ 
low-resolution simulations is less than eight high-resolution ones. Thus, for
the considered problem, the MLMC-based PhIK using $10$ high-resolution and $500$ 
low-resolution simulations is less costly than MC-based PhIK with $18$ high-resolution
simulations, while its accuracy is better than the latter with $90$ 
high-resolution simulations (as shown in Figure~\ref{fig:adv_mlmc_err}(a)).
Here, the accuracy of MLMC-based PhIK changes slowly as $M_H$
increase from $20$ to $90$. 
This can be partially explained by examining the mean and covariance
computed from MLMC. We use the mean and covariance of $5000$ high-resolution
simulations as reference and compare them with results by MLMC, respectively,
in Figure~\ref{fig:adv_mlmc_err}(b). The differences are normalized by the
$\Vert y(\bm x)\Vert_2$ to investigate their influence on the MLMC-based PhIK's
accuracy. Here, we abuse the notation to set $\Vert\bm c(\bm
x)\Vert_F=\left(\sum_{\bm x\in D}
\Vert \bm c(\bm x)\Vert_2^2 \right)^{1/2}$, where $\bm c(\bm x)$ is the
vector used in GPR prediction.
It is clear that the influence from the difference in estimating the
$\bm c(\bm x)$ (triangles) is very small, and difference between the 
covariance matrix (not shown) is even smaller than the difference in 
$\bm c(\bm x)$ because the $k$th row of $\tensor C$ is $\bm c(\bm x^{(k})$.
The difference in estimating the mean (squares) changes slowly when the number 
of high-resolution simulations is larger than $30$ which is consistent with the
trend of the diamonds in Figure~\ref{fig:adv_mlmc_err}(a).

\begin{figure}[!h]
\centering
  \subfigure[Field reconstruction error]{
  \includegraphics[width=0.4\textwidth]{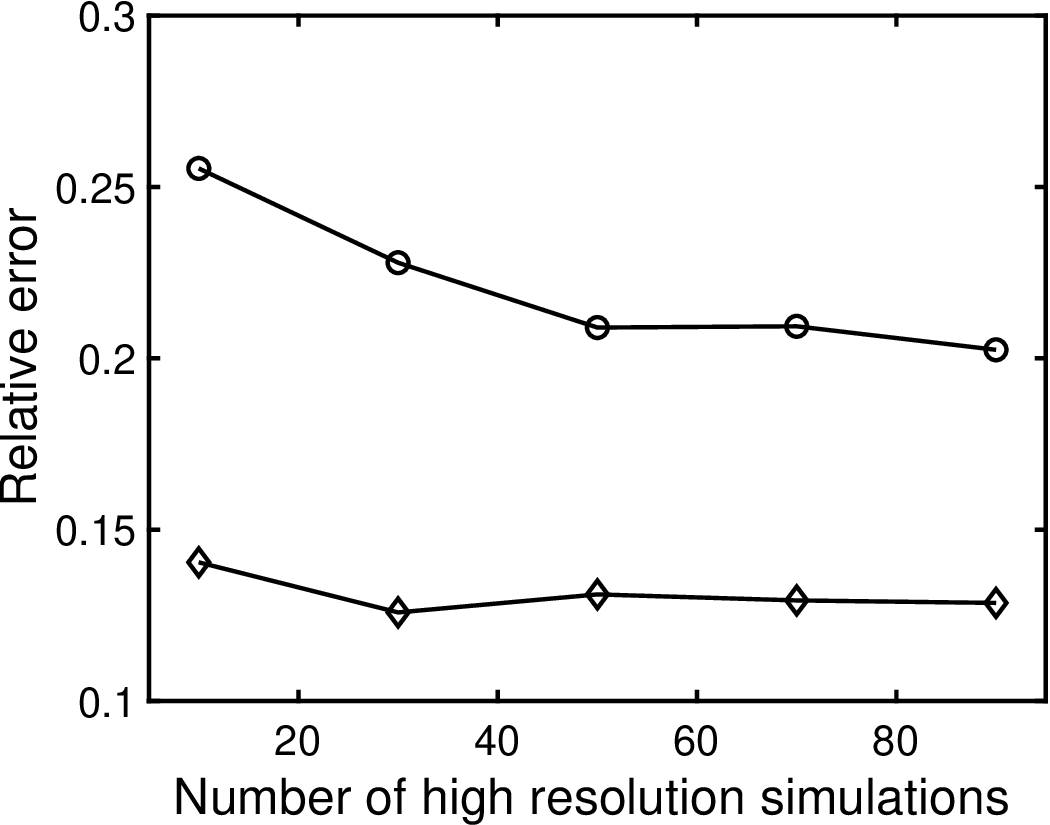}}\quad
\subfigure[MLMC approximation error]{
  \includegraphics[width=0.4\textwidth]{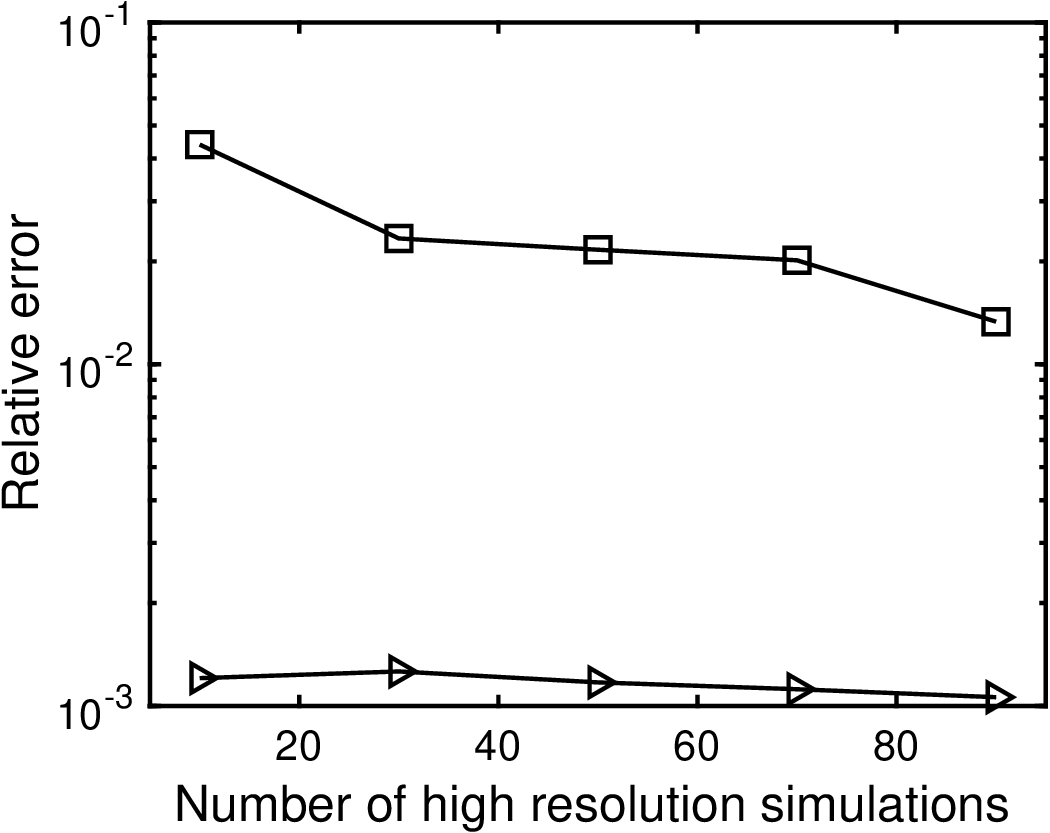}}
\caption{Solute transport: (a) Relative error of solute concentration 
  $\Vert\hy-\ey\Vert_2/\Vert \ey\Vert_2$ by PhIK using different numbers of
  high-resolution simulations (mesh size $1$m$\times 1$m) only
  (``$\circ$") and $500$ low-resolution simulations (mesh size 
  $4$m$\times 4$m) in addition to different numbers of high-resolution simulations
  (``$\diamond$"). (b) $\Vert \mu_{_{MC}}(\bm x) - \mu_{_{MLMC}}
  (\bm x) \Vert_2/\Vert y(\bm x)\Vert_2$ (``$\square$'') and $\Vert
  c_{_{MC}}(\bm x) - c_{_{MLMC}} (\bm x) \Vert_F/\Vert y(\bm x)\Vert_2$
  (``$\triangleright$''). }
\label{fig:adv_mlmc_err}
\end{figure}

Moreover, we denote the cost of each single low-fidelty simulation as $C_L$ and
the cost of each high-fidelity simulation as $C_H$. The total cost of
simulations for MC-based PhIK is $C_HM$, where $M\approx M_L$, while the total
cost of simulations for MLMC-based PhIK is $C_HM_H+C_LM_L$. The cost ratio (MLMC
cost over MC cost) is approximately $M_H/M_L+C_L/C_H$. In this specific case, 
$M_H/M_L=10/500=0.02$, and $C_L/C_H=1/64=0.015625$. So the ratio 
of the cost is $0.035$, i.e., $18$ high-resolution simulations against $500$
ones.

\subsubsection{Active learning}

We now compare the performance of the active learning algorithm based on
Kriging and MLMC-based PhIK with ensembles $\Hyh$ and $\Hyl$ with $M_H=10, M_L=500$.
Because we demonstrated that MLMC-based PhIK is more accurate and less costly than 
MC-based PhIK, we do not use the latter in this comparison. 
Figures~\ref{fig:adv_std} shows $\hs$ for Kriging, and MLMC-PhIK,
both using the initial 15 observations (locations are denoted by squares). 
Note that $\hs$ in MLMC-PhIK is much smaller than that in Kriging and the
locations of local maxima differ. 
\begin{figure}[!h]
\centering
\subfigure[Kriging $\hat s$]{
\includegraphics[width=0.45\textwidth, trim=10 70 10 80, clip ]{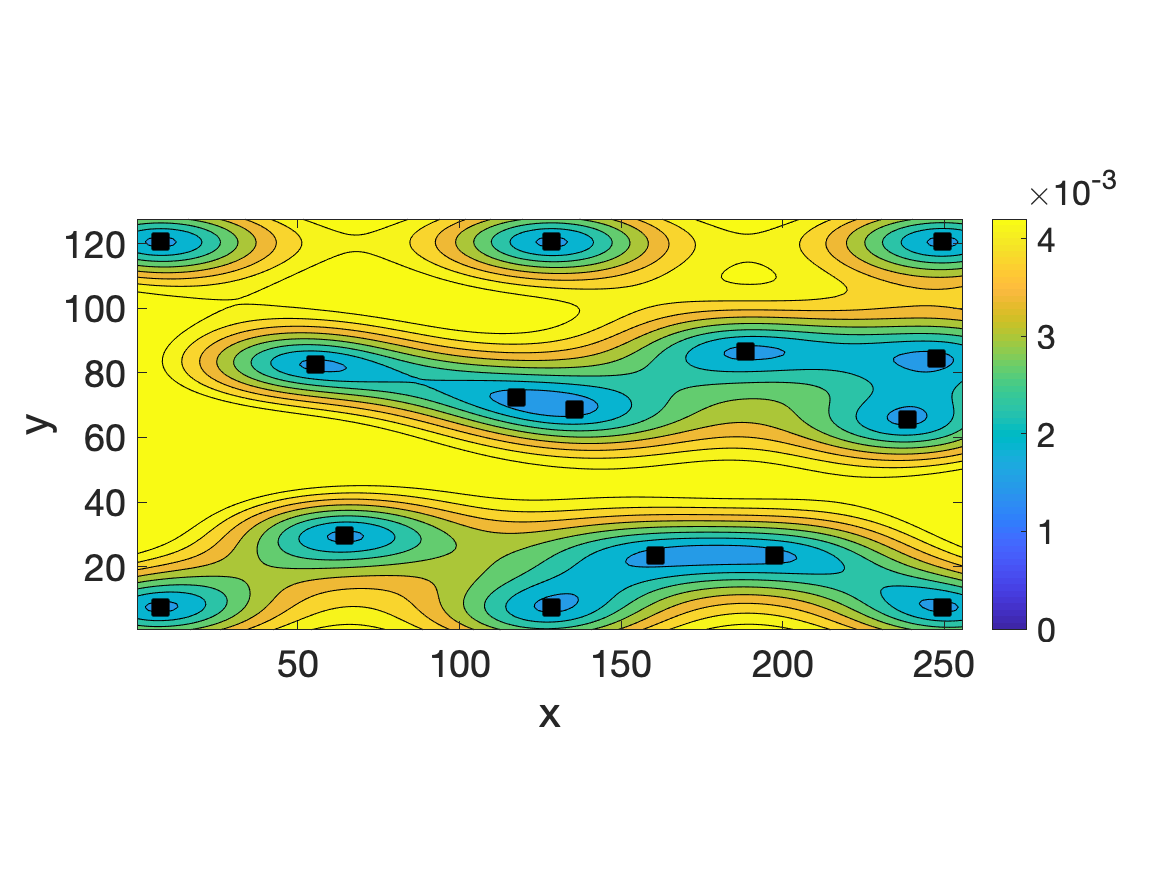}}
\subfigure[MLMC-based PhIK $\hat s$]{
\includegraphics[width=0.44\textwidth]{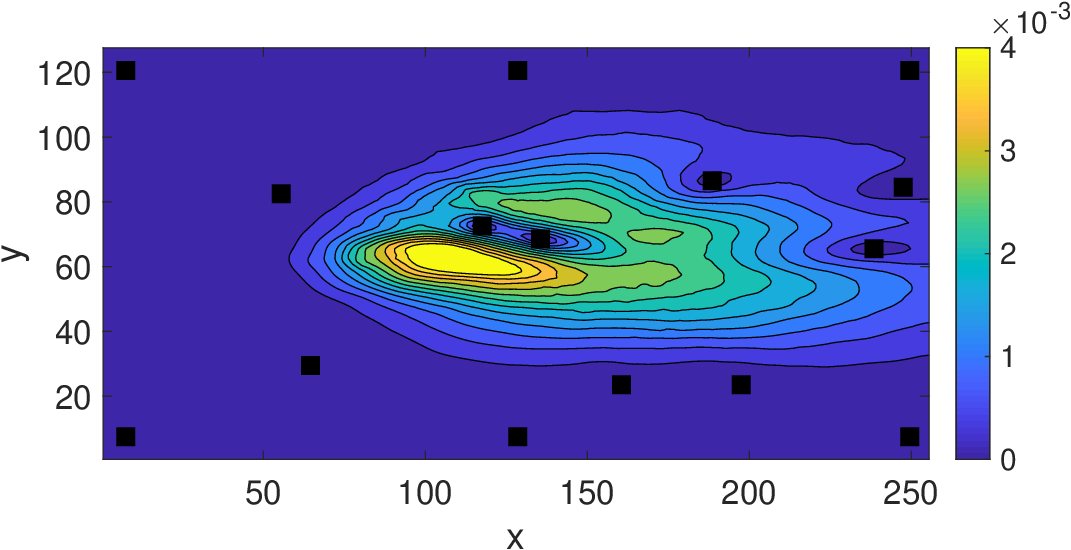}}
\caption{Solute concentration: (a) $\hs$
 of Kriging using $15$ observations; 
 (b) $\hs$ by MLMC-based PhIK using $15$ observations.}
\label{fig:adv_std}
\end{figure}

Next, we use Algorithm~\ref{algo:act} in combination with Kriging 
MLMC-based PhIK to add new observations one by one. In these figures, the 
initial $15$ observation locations are marked as squares and new locations are 
marked as stars. 
Figure~\ref{fig:adv_act} compares these results when $15$ new observations are
added. It shows that
MLMC-based PhIK consistently outperforms Kriging as 
quantitatively confirmed by the comparison in Figure~\ref{fig:adv_comp_l2}. For both 
methods, the error and uncertainty decrease with an increasing number of 
observations. 
Also, there are significant differences in the results. In Kriging, most 
new points are added near the boundary, while in MLMC-based PhIK, new measurements are
added inside the domain close
to the plume center. This is because the error in Kriging is dominated by the 
extrapolation error at the boundary. In MLMC-based PhIK, the boundary conditions in
the physical model provide sufficient information near the boundaries.
Consequently, the active learning algorithm explores more information around the
plume. Moreover, compared with the first numerical example, the Gaussian kernel used 
in the Kriging method is less suitable for approximating the inhomogeneous field
$\bm F$ in this example. As a result, PhIK
achieves higher accuracy than Kriging with a smaller number of observations.
\begin{figure}[!h]
\centering
\subfigure[Kriging $\hy$]{
\includegraphics[width=0.45\textwidth]{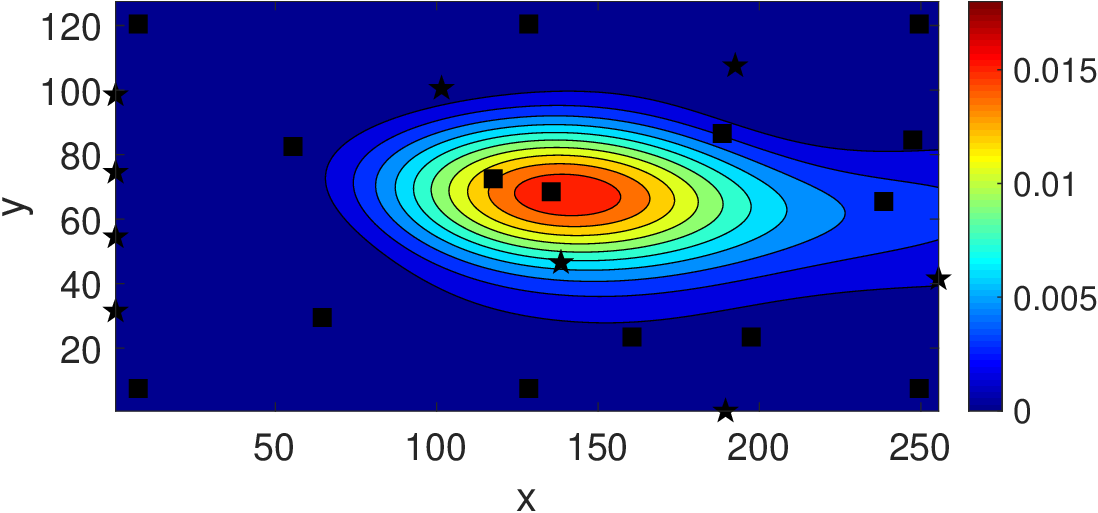}} \quad
\subfigure[PhIK $\hy$]{
\includegraphics[width=0.45\textwidth]{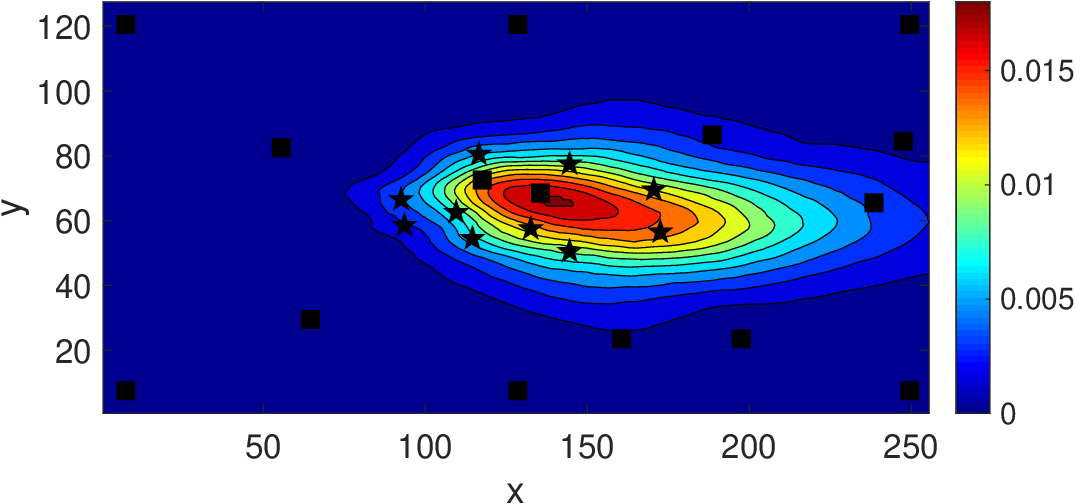}} \\
\subfigure[Kriging $\hs$]{
\includegraphics[width=0.45\textwidth]{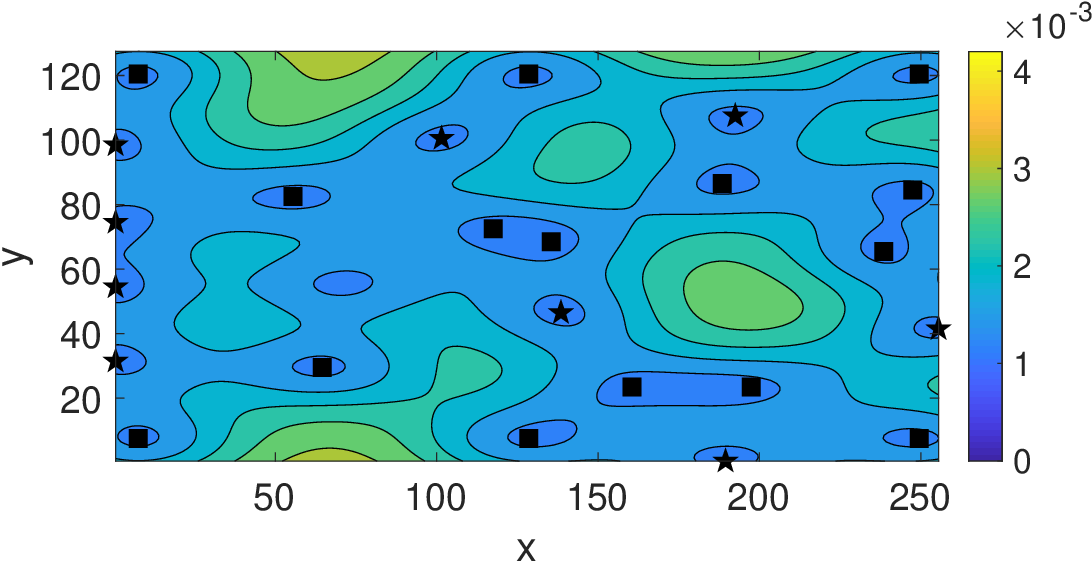}} \quad
\subfigure[PhIK $\hs$]{
\includegraphics[width=0.45\textwidth]{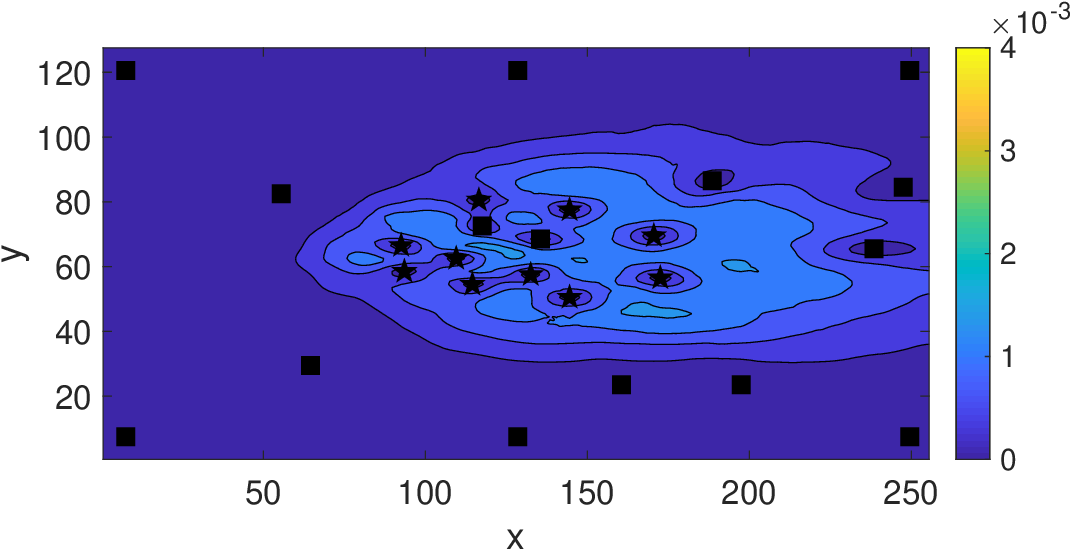}}  \\
\subfigure[Kriging $\hy-\ey$ ]{
\includegraphics[width=0.45\textwidth]{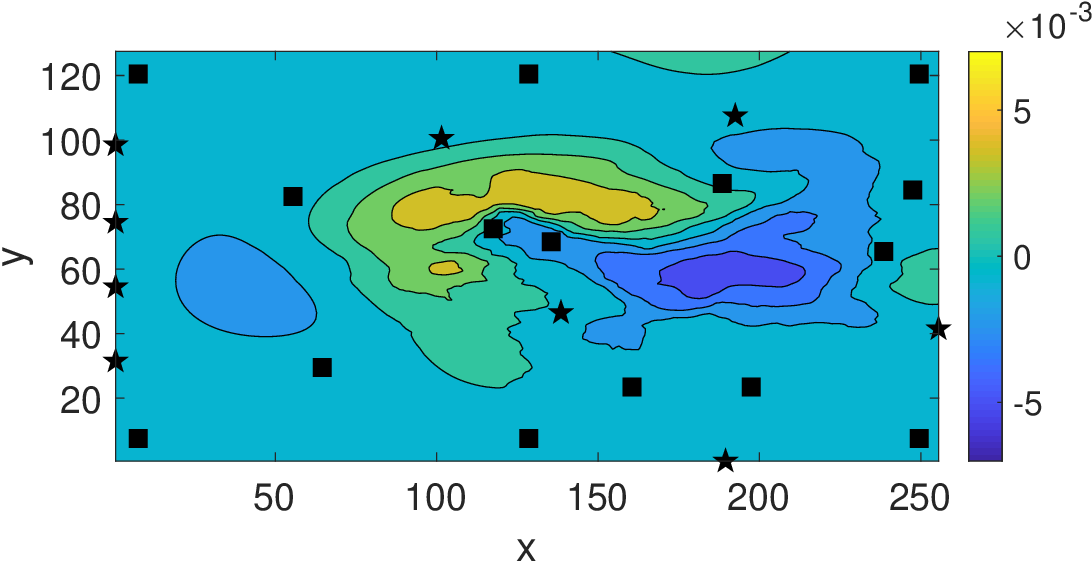}} \quad
\subfigure[PhIK $\hy-\ey$]{
\includegraphics[width=0.45\textwidth]{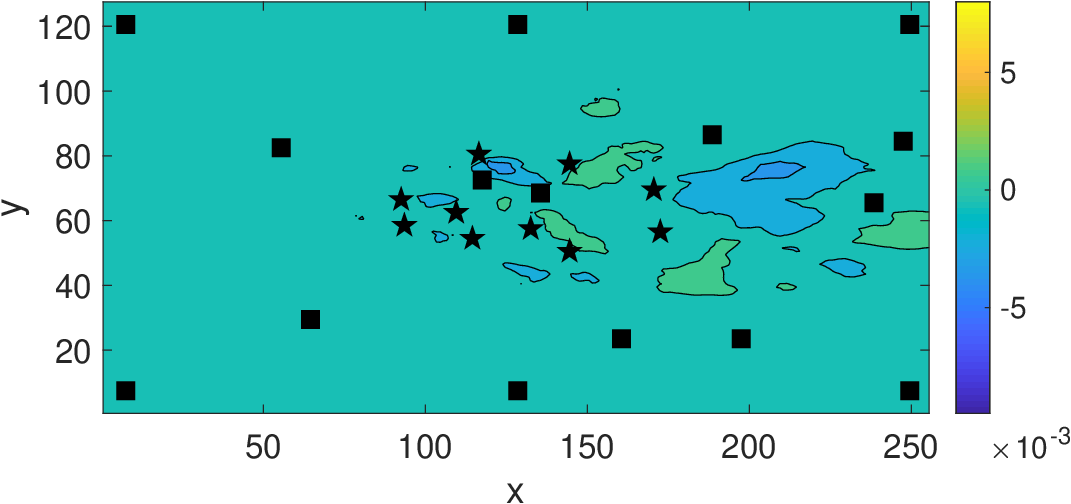}} 
\caption{Solute transport: reconstruction of the solute concentration by Kriging and MLMC-based
  PhIK via active learning.
  Black squares mark the locations of the original eight observations, and stars 
  are newly added observations. Left colum: Kriging; right column: PhIK. First 
  row: $\hy$; second row: $\hs$; third row: $\hy-\ey$.}
\label{fig:adv_act}
\end{figure}
\begin{figure}[!h]
\centering
\includegraphics[width=0.4\textwidth]{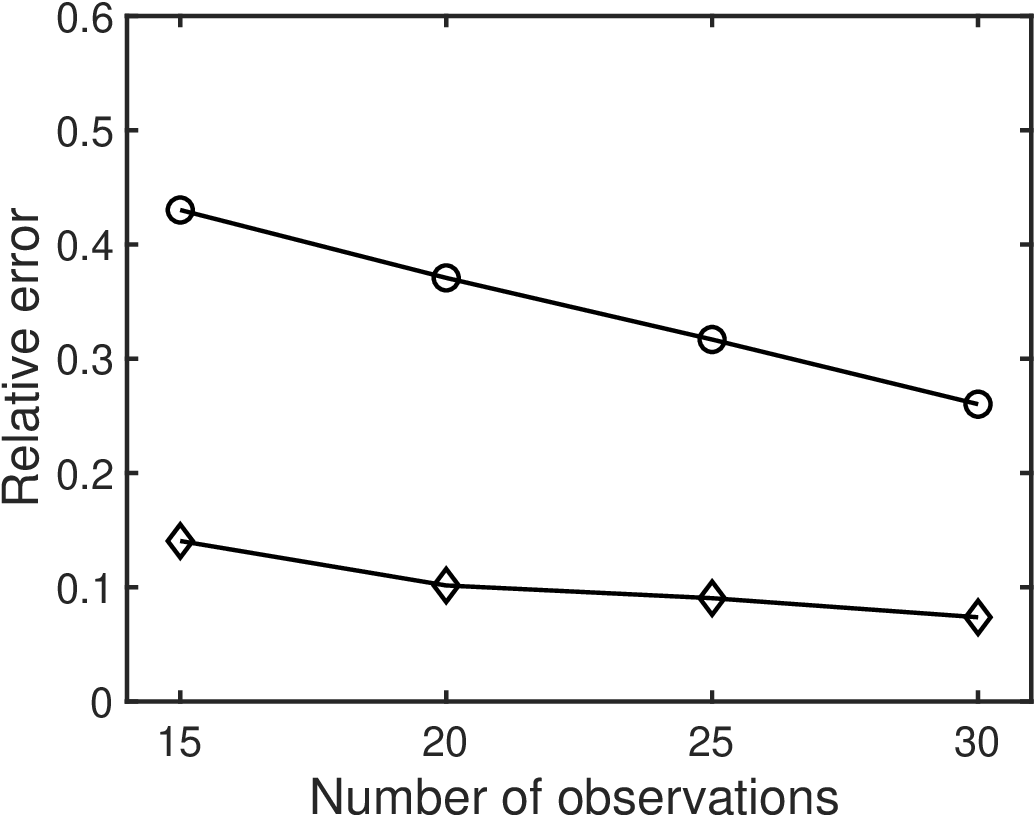}
\caption{Solute transport: relative error of 
  $\Vert\hy-\ey\Vert_2/\Vert\hy \Vert_2$ of Kriging (``$\circ$") and 
  MLMC-based PhIK (``$\diamond$") using different numbers of total observations
  via active learning.}
\label{fig:adv_comp_l2}
\end{figure}
Finally, as an example of preserving linear constraints, the
Dirichlet boundary on the left side is kept very well with maximum number
smaller than 1e-6.


\section{Conclusion}
\label{sec:concl}
In this work, we propose the PhIK method, where the mean and covariance 
function in the GP model are computed from a partially known physical model of
the states. We also propose a novel MLMC estimate of the covariance function 
that, in combination with the standard MLMC estimate of the mean, leads to 
significant cost reduction in estimating statistics compared to the standard
MC method. The resulting statistics in PhIK is non-stationary as 
can be expected for states of many physical systems due to their
intrinsic inhomogeneity.
This is different from the standard ``data-driven'' Kriging, where the mean and
kernel are estimated from data only and usually requires an assumption of
stationarity. In addition, PhIK avoids the need for estimating hyperparameters 
in the covariance function, which can be a challenging optimization problem. 

We prove that PhIK preserves the physical knowledge if it is in the form of a
deterministic linear operator. 
We also provide an upper error bound in the PhIK prediction in the presence of 
numerical errors. These theoretical results indicate that the accuracy of PhIK 
prediction depends on the physical model's accuracy
($\Vert\bm y-\bm\mu\Vert_2$), numerical error ($\epsilon$)
the physical model's stochastic properties, and the selection of observation 
locations ($(\Vert\tensor C+\alpha\tensor I)^{-1}\Vert_2$). We demonstrate that
an active learning
algorithm in combination with PhIK suggests very different locations for new
observations than the data-driven Kriging and has the potential to result in
significantly more accurate predictions with reduced uncertainty using fewer
measurements. Other Kriging methods, e.g., university Kriging, may perform 
better than ordinary Kriging. However, such methods require non-stationary mean
or kernel with larger numbers of hyperparameters, which adds to the difficulty
of the optimization problem in identifying these hyperparameters.

Our method allows model and data convergence without solving complex
optimization problems. Moreover, this method is nonintrusive as 
it can utilize existing domain codes to compute mean and covariance functions for GPR. This differs
from other ``physics-informed'' GPR methods, e.g.,  
 \cite{hennig2015probabilistic, schober2014probabilistic, raissi2017machine, raissi2018numerical}, 
where physical laws are used to derive equations for the covariance function, which,
in general, must be solved numerically. 
PhIK is especially suitable for problems with very costly observations and
partially known physics models, including climate, oceanography, and transport
in porous media. Such applications are governed by conservation laws, but the
parameters, source/sink terms, and stresses in these conservation laws are often
unknown and could be modeled as random processes.

Finally, it is worth repeating that the accuracy of PhIK prediction depends on
the accuracy of the stochastic physical model. In other words, the PhIK accuracy depends on the distance between the
exact solution and the linear space spanned by the simulation ensemble. The
accuracy may be improved by adding correction terms, e.g,~\cite{yang2019physics}
and the cost of simulations may be further reduced by using other multi-fidelity
approaches~\cite{yang2020bifidelity}.

\section*{Acknowledgments}

This work was supported by the U.S. Department of Energy (DOE), Office of 
Science, Office of Advanced Scientific Computing Research (ASCR) as part of the
Multifaceted Mathematics for Complex Systems and Uncertainty Quantification in 
Advection-Diffusion-Reaction Systems projects. A portion of the research 
described in this paper was conducted under the Laboratory Directed Research
and Development Program at Pacific Northwest National Laboratory (PNNL). PNNL is
operated by Battelle for the DOE under Contract DE-AC05-76RL01830.


\appendix
\section{Proof of Theorems~\ref{thm:err_bound}}
\label{sec:app_proof1}
\begin{proof}
The Kriging prediction Eq.~\eqref{eq:krig_pred0} can be rewritten as the following function form:
  \begin{equation}
    \label{eq:krig_form0}
  \hat y(\bm x) = \mu(\bm x) + \sum_{i=1}^N a_i k(\bm x, \bm x^{(i)}), 
  \end{equation}
  where $\bm x\in D$, $a_i$ is the $i$-th entry of $(\tensor C+\alpha\tensor I)^{-1}(\bm y-\bm\mu)$. 
Similarly, the PhIK prediction can be written as
  \begin{equation}
    \label{eq:krig_form1}
    \hat y(\bm x) = \mu_{_{MC}}(\bm x) + \sum_{i=1}^N \tilde a_i k_{_{MC}}(\bm x, \bm x^{(i)}), 
  \end{equation}
  where $\tilde a_i$ is the $i$-th entry of 
  $(\tensor C_{_{MC}}+\alpha\tensor I)^{-1}(\bm y-\bm\mu_{_{MC}})$. 
We have
  \[\begin{aligned} \Vert\mathcal{L}\mu_{_{MC}}(\bm x)-\overline{g(\bm x)}\Vert 
      & = \Big\Vert\dfrac{1}{M}\sum_{m=1}^M\mathcal{L} Y^m(\bm x)-\dfrac{1}{M}\sum_{m=1}^Mg(\bm x;\omega^m)\Big\Vert \\
   & \leq \dfrac{1}{M}\sum_{m=1}^M\left\Vert\mathcal{L} Y^m(\bm x)-g(\bm
  x;\omega^m)\right\Vert \leq \epsilon. \end{aligned}\]
Also,
\begin{equation}
  \label{eq:comp_err_bound}
  \begin{aligned}
    & \Vert\mathcal{L}k_{_{MC}}(\bm x,\bm x^{(i)})\Vert \\
     = &\left\Vert\dfrac{1}{M-1} \sum_{m=1}^M\left(Y^m(\bm
    x^{(i)})-\mu_{_{MC}}(\bm x^{(i)})\right) \mathcal{L}\Big(Y^m(\bm
    x)-\mu_{_{MC}}(\bm x)\Big)\right\Vert \\
     \leq & \dfrac{1}{M-1} \sum_{m=1}^M\left\vert Y^m(\bm x^{(i)})-\mu_{_{MC}}(\bm
    x^{(i)})\right\vert\Big\Vert \mathcal{L}\Big(Y^m(\bm x)-\mu_{_{MC}}(\bm
    x)\Big)\Big\Vert \\
     \leq & \dfrac{1}{M-1} \sum_{m=1}^M\left\vert Y^m(\bm
    x^{(i)})-\mu_{_{MC}}(\bm x^{(i)})\right\vert\cdot \\
    &\left\{\Big\Vert \mathcal{L}Y^m(\bm x)-g(\bm x;\omega^m) - \Big(\mathcal{L}\mu_{_{MC}}(\bm x)-\overline{g(\bm
  x)}\Big)\Big\Vert + \Vert g(\bm x;\omega^m)-\overline{g(\bm x)}\Vert\right\} \\
     \leq &\dfrac{2\epsilon}{M-1} \bigg(M\sum_{m=1}^M \left\vert Y^m(\bm
x^{(i)})-\mu_{_{MC}}(\bm x^{(i)})\right\vert^2\bigg)^{\half} \\ 
    & + \dfrac{1}{M-1} \bigg(\sum_{m=1}^M \left\vert Y^m(\bm x^{(i)})-\mu_{_{MC}}(\bm
  x^{(i)})\right\vert^2\bigg)^{\half}\bigg(\sum_{m=1}^M \Vert g(\bm
  x;\omega^m)-\overline{g(\bm x)}\Vert^2 \bigg)^{\half} \\ 
  = &2\epsilon \sqrt{\dfrac{M}{M-1}}\bigg(\dfrac{1}{M-1}\sum_{m=1}^M
\left\vert Y^m(\bm x^{(i)})-\mu_{_{MC}}(\bm x^{(i)})\right\vert^2\bigg)^{\half} \\ 
    &  + \bigg(\dfrac{1}{M-1}\sum_{m=1}^M \left\vert Y^m(\bm
x^{(i)})-\mu_{_{MC}}(\bm x^{(i)})\right\vert^2\bigg)^{\half}
\bigg(\dfrac{1}{M-1}\sum_{m=1}^M \Vert g(\bm x;\omega^m)-\overline{g(\bm
x)}\Vert^2 \bigg)^{\half}\\ 
 = & \left(2\epsilon \sqrt{\dfrac{M}{M-1}}+ \sigma(g(\bm x;\omega^m))\right) \sigma(Y^m(\bm x^{(i)})).
  \end{aligned}
\end{equation}
Thus, according to Eq.~\eqref{eq:krig_form1}:
\[\begin{aligned}
    \Vert\mathcal{L}\hat y(\bm x)-\overline{g(\bm x)}\Vert & \leq
    \epsilon+\left[ 2\epsilon\sqrt{\dfrac{M}{M-1}} + \sigma(g(\bm x;\omega^m))\right]\sum_{i=1}^N |\tilde a_i| \sigma(Y^m(\bm x^{i})) \\
    & \leq\epsilon+ \left[ 2\epsilon\sqrt{\dfrac{M}{M-1}} + \sigma(g(\bm
    x;\omega^m))\right]\max_{1\leq i\leq N}|\tilde a_i| \sum_{i=1}^N \sigma(Y^m(\bm x^{i})) 
\end{aligned}
\]
  Because $\max_i\vert\tilde a_i\vert= \left\Vert(\tensor
  C_{_{MC}}+\alpha\tensor I)^{-1}(\bm y-\bm\mu_{_{MC}})\right\Vert_{\infty}$, the
conclusion holds.
\end{proof}

\section{Proof of Corollary~\ref{cor:err_bound2}}
\label{sec:app_proof2}
\begin{proof}
  \[\begin{aligned}
      \Vert\mathcal{L}\overline Y^m(\bm x)\Vert & = \Vert\mathcal{L} Y_H^m(\bm
      x)-\mathcal{L} Y_L^m(\bm x)\Vert \\ & = \Vert\mathcal{L} Y_H^m(\bm
  x)-g(\bm x;\omega^m) -(\mathcal{L} Y_L^m(\bm x)-g(\bm x;\omega^m))\Vert  \\
  & \leq \epsilon_{_H} + \epsilon_{_L}.\end{aligned} \]
  We denote $\mu_{_L}(\bm x)=\dfrac{1}{M_L}\sum_{m=1}^{M_L}Y_L^m(\bm x)$, and 
  $\overline\mu(\bm x)=\dfrac{1}{M_H}\sum_{m=1}^{M_H}\overline Y^m(\bm x)$.
  According to Eq.~\eqref{eq:mlmc_mean}, $\mu_{_{MLMC}}(\bm x)=\mu_{_L}(\bm
  x)+\overline\mu(\bm x)$. 
  By construction, 
  $\Vert\mathcal{L}\mu_{_L}(\bm x)-\overline{g(\bm x)}\Vert\leq \epsilon_{_L}$ 
  and $\Vert\mathcal{L}\overline\mu(\bm x)\Vert\leq \epsilon_{_L}+\epsilon_{_H}$. Thus,
  \[
    \Vert\mathcal{L}\mu_{_{MLMC}}(\bm x)-\overline{g(\bm x)}\Vert
    =\left\Vert\mathcal{L}\mu_{_L}(\bm x)-\overline{g(\bm x)} +
    \mathcal{L}\overline\mu(\bm x)\right\Vert\leq 2\epsilon_{_L}+\epsilon_{_H}.
  \]
  Following the same procedure in Eq.~\eqref{eq:comp_err_bound}, we have
  \begin{multline*}\left\Vert\dfrac{1}{M_L-1} \sum_{m=1}^{M_L}\left(Y_L^m(\bm
    x^{(i)})-\mu_{_{L}}(\bm x^{(i)})\right) \mathcal{L}\left(Y_L^m(\bm x)-\mu_{_{L}}(\bm x)\right)\right\Vert  \\
  \leq\bigg(2\epsilon_{_L} \sqrt{\dfrac{M_L}{M_L-1}}+\sigma(g(\bm
  x;\omega^m)\bigg)\sigma(Y_L^m(\bm x^{(i)})),\end{multline*}
  and
  \begin{multline*}\left\Vert\dfrac{1}{M_H-1} \sum_{m=1}^{M_H}\left(\overline Y^m(\bm
    x^{(i)})-\overline\mu(\bm x^{(i)})\right) \mathcal{L}\left[\overline Y^m(\bm
    x)-\overline\mu(\bm x)\right]\right\Vert \\
  \leq 2(\epsilon_{_H}+\epsilon_{_L}) \sqrt{\dfrac{M_H}{M_H-1}}\sigma(\overline
Y^m(\bm x^{(i)})). \end{multline*}
  As such,
  \[\begin{aligned}
      & \Vert\mathcal{L}\hat y(\bm x)-\overline{g(\bm x)}\Vert \\ \leq &\epsilon_{_H}+2\epsilon_{_L} +
      \left(2\epsilon_{_L}\sqrt{\dfrac{M_L}{M_L-1}} + \sigma(g(\bm
      x;\omega^m))\right)\sum_{i=1}^N\tilde a_i\sigma(Y_L^m(\bm x^{(i)})) \\
      & + 2(\epsilon_{_H}+\epsilon_{_L})\sum_{i=1}^N\tilde a_i\sqrt{\dfrac{M_H}{M_H-1}}\sigma(\overline Y^m(\bm x^{(i)}))\\
     = &\epsilon_{_H}\left(1+2\sum_{i=1}^N\tilde a_i\sqrt{\dfrac{M_H}{M_H-1}}\sigma(\overline Y^m(\bm x^{(i)}))\right)  \\
    & + \epsilon_{_L}\left[2+2\sum_{i=1}^N\tilde
    a_i\left(\sqrt{\dfrac{M_L}{M_L-1}}\sigma(Y_L^m(\bm
  x^{(i)}))+\sqrt{\dfrac{M_H}{M_H-1}}\sigma(\overline Y^m(\bm x^{(i)}))\right)\right] \\
  & + \sigma(g(\bm x;\omega^m))\sum_{i=1}^N\tilde a_i\sigma(Y_L^m(\bm x^{(i)})).
  \end{aligned}
  \]
\end{proof}

\bibliographystyle{siamplain}
\bibliography{ref}

\end{document}